\documentclass[twoside]{article}

\usepackage[preprint]{aistats2027}
\usepackage[round]{natbib}

\usepackage[normalem]{ulem}
\usepackage{xcolor}
\definecolor{SMIDeleted}{RGB}{170,35,35}
\definecolor{SMIAdded}{RGB}{20,70,175}

\usepackage{paper_style}
\definecolor{deepgreen}{HTML}{19382B}
\definecolor{green}{HTML}{28734F}
\definecolor{lightgreen}{HTML}{8EB49D}
\definecolor{lightlightgreen}{HTML}{F3F8F4}
\definecolor{orange}{HTML}{E7836B}
\usepackage{hyperref}
\hypersetup{
    colorlinks=true,
    linkcolor=green, 
    citecolor=green,     
    urlcolor=green       
}

\usepackage{etoc}

\begin{document}
\etocdepthtag.toc{main}

%
\runningtitle{Spread Mutual Information}

%
\runningauthor{Yu, OuYang, Hern\'{a}ndez-Lobato, Liu}

\twocolumn[

\aistatstitle{Controlling Dependence in Implicit Generative Models\\via Spread Mutual Information}

\begin{center}

\textbf{Jiahao Yu}\textsuperscript{1,*}
\qquad
\textbf{Song Liu}\textsuperscript{2}
\qquad
\textbf{Jos\'{e} Miguel Hern\'{a}ndez-Lobato}\textsuperscript{3}
\qquad
\textbf{RuiKang OuYang}\textsuperscript{3,\(\dagger\)}

\vspace{0.4em}

{\small
\textsuperscript{1}University of Oxford
\qquad
\textsuperscript{2}University of Bristol
\qquad
\textsuperscript{3}University of Cambridge
}

\vspace{0.3em}

{\small
\textsuperscript{*}Work completed at the University of Cambridge.
\qquad
\textsuperscript{\(\dagger\)}Correspondence to \texttt{ro352@cam.ac.uk}.
}

\end{center}

\vspace{1em}
]




\begin{abstract}
Mutual information (MI) provides an objective for suppressing or encouraging statistical dependence in implicit generative models.
However, direct MI evaluation is challenging in implicit models due to typically intractable densities.
A remedy is estimating the generator gradient from the difference between conditional and marginal scores. 
This score difference can, in turn, be estimated by differentiating a log density ratio learned through classification.
This construction nevertheless faces two difficulties: (i) singular distributions need not admit the required score functions, and (ii) poor overlap can hinder density-ratio estimation.
We therefore introduce \textit{\textbf{S}pread \textbf{M}utual \textbf{I}nformation} (SMI), a weighted integral of MI across noise levels obtained by applying a common spreading kernel to the generated variable.
Gaussian spreading yields smooth, strictly positive conditional and marginal densities, extending the gradient construction to distributions that may originally be singular.
Across a variaty of experiments, SMI consistently achieves effective dependence control among MI-based methods and remains competitive with established task-specific approaches.
\end{abstract}


\section{Introduction}
\label{sec:introduction}

Controlling how much information generated outputs retain about an auxiliary variable is important in conditional generation and representation learning. 
Reducing dependence supports information bottlenecks, invariance, and independence regularisation \citep{tishby_information_2000,alemi_deep_2017,zhang_causaldiff_2024,gaudi_coind_2025,yu_missing_2025}, while increasing dependence can preserve latent factors and conditioning signals \citep{chen_infogan_2016,yu_regularized_2024,wang_information_2025}.
Mutual information (MI) quantifies statistical dependence and provides a common objective for both goals. We study dependence minimisation and maximisation in implicit generative models, which allow sampling without tractable density evaluations \citep{mohamed_learning_2016}.

Optimising MI in these models presents two distinct difficulties. First, its generator gradient must be estimated from samples. Methods such as MINE \citep{belghazi_mutual_2018}, InfoNCE \citep{oord_representation_2018}, and CLUB \citep{cheng_club_2020} provide tractable objectives through variational bounds or conditional-density estimation, but accurate objective values do not by themselves guarantee accurate gradients. Second, MI itself can be poorly suited to smooth dependence control. Singular conditional distributions may lead to infinite MI or undefined scores, while even finite MI can be insensitive to support geometry. 
For example, if $Y$ is uniform on $\{-1,+1\}$ and $X_a=(Z,aY)$ with independent Gaussian $Z$, then $I(X_a;Y)=\log 2$ for every $a>0$, regardless of how close the two parallel lines supporting $X_a\mid Y=+1$ and $X_a\mid Y=-1$ become. Independence occurs only at $a=0$, but the MI gradient with respect to $a$ is zero for every $a>0$, so gradient descent cannot reduce $a$ towards zero. We revisit this example experimentally in \Cref{fig:synthetic-summary}a.

We introduce \textit{\textbf{S}pread \textbf{M}utual \textbf{I}nformation} (SMI) for dependence control, building on Spread Divergence and KL objectives integrated over noise levels \citep{zhang_spread_2020,he_training_2025}. SMI measures dependence after applying a fixed noise channel to the generated variable, identically under the joint and product-of-marginals laws; we combine noise levels through a weighted integral. Gaussian spreading produces smooth, positive conditional and marginal densities even when the original distributions are singular. 
In the example above, it makes the information available about $Y$ sensitive to the distance $2a$ between these lines, providing a gradient for reducing $a$.

To optimise SMI, we target its generator gradient directly, following the motivation of MIGE \citep{wen_mutual_2020}. Under suitable regularity conditions, this gradient depends on the difference between conditional and marginal scores. Following DiffRatio \citep{chen_diffratio_2025}, we estimate this difference by differentiating a learned log density ratio. A single classifier, conditioned on the noise level, learns the ratio by distinguishing joint samples from independently paired samples \citep{gutmann_noisecontrastive_2010}. This construction avoids fitting two separate score models and supports both minimisation and maximisation alongside a task-specific objective.

The noise channel and its weighting determine the scales at which SMI measures dependence.
Under suitable conditions on the spreading kernel and the weighting across noise levels,
SMI vanishes if and only if the variables are \textit{independent}.\footnote{Small SMI does not in general imply small MI.}

Our contributions are: (i) we introduce SMI to mitigate singularity and poor overlap in MI optimisation, and develop a classifier-based pathwise gradient estimator for both minimisation and maximisation, (ii) we characterise independence and information dissipation under spreading, and bound gradient sampling and score-difference estimation errors, and (iii) we show that SMI performs strongly on controlled synthetic problems and remains competitive on real-world tasks.
\section{Background}
\label{sec:background}

\textbf{Mutual information (MI)} quantifies the statistical dependence between
two random variables $X$ and $Y$, defined as the
Kullback--Leibler (KL) divergence between their joint distribution
and the product of their marginals. When the relevant densities
exist, this gives
\begin{align}
    I(X;Y)
    &=\mathbb{D}_{\mathrm{KL}}\!\left(
        p(x\mid y)p_Y(y)
        \,\middle\|\,
        p(x)p_Y(y)
    \right) \notag\\
    &=\mathbb{E}_{p(x,y)}\!\left[
        \log\frac{p(x\mid y)}{p(x)}
    \right].
    \label{eq:mutual-info-def-cancel}
\end{align}
The density expression above is a special case of the KL definition, which applies to probability laws even without densities and allows MI to take the value $+\infty$.
\paragraph{Problem setup.}
We consider conditional generation of an output $X\in\mathbb{R}^d$
given an input $Y\sim p_Y$, whose distribution does not depend on
the generator parameters $\theta$.
When densities exist, a parameterized conditional distribution
$p_\theta(x\mid y)$ induces the joint density
$p_\theta(x,y)=p_\theta(x\mid y)p_Y(y)$ and the marginal density
\(
    p_\theta(x)
    =\mathbb{E}_{Y\sim p_Y}\!\left[p_\theta(x\mid Y)\right].
\)

Our objective is to maximise or minimise MI under this induced
joint distribution:
\begin{align}
    I_\theta(X;Y)
    =\mathbb{D}_{\mathrm{KL}}\!\left(
        p_\theta(x\mid y)p_Y(y)
        \,\middle\|\,
        p_\theta(x)p_Y(y)
    \right).
    \label{eq:neural-mutual-info}
\end{align}

\paragraph{Implicit generative model.}
We parameterize the conditional distribution using an implicit
generator
\begin{equation}
    X=g_\theta(Y,Z),
    \quad Z\sim q_Z,
    \quad Z\perp Y,
\end{equation}
where $g_\theta$ is a neural network and $q_Z$ is a fixed noise
distribution independent of $\theta$.
For each input $y$, the conditional law of $X$ is the
pushforward of $q_Z$ through $g_\theta(y,\cdot)$.
When this law admits a density, it can be written formally as
\begin{equation}
    p_\theta(x\mid y)
    =\mathbb{E}_{Z\sim q_Z}\!\left[
        \delta\!\left(x-g_\theta(y,Z)\right)
    \right],
\end{equation}
where $\delta$ denotes the Dirac delta. This construction enables sampling from the conditional law
without evaluating its density.
Even when densities exist, evaluating either
$p_\theta(x\mid y)$ or $p_\theta(x)$ is \textit{generally intractable}.

\paragraph{MI optimisation via gradient estimation.}
The intractability of these densities generally prevents direct
evaluation of \Cref{eq:neural-mutual-info}.
Nevertheless, MI optimisation can proceed by estimating its gradient
\textit{without evaluating MI itself}.
Earlier work derived MI differentials in terms of score differences \citep{babaie-zadeh_differential_2004}, while
\citet{wen_mutual_2020} developed score-based estimators of the MI
gradient for representation learning.

Suppose the generator is differentiable in $\theta$ and the
induced densities are differentiable in $(\theta,x)$ and
positive on a common, $\theta$-independent support.
If MI and the relevant expectations are finite and differentiation
can be interchanged with integration, then
\begin{multline}
    \nabla_\theta I_\theta(X;Y)
    =\mathbb{E}_{Y\sim p_Y,\,Z\sim q_Z}\!\biggl[
        \left(
            \frac{\partial g_\theta(Y,Z)}{\partial\theta}
        \right)^{\!\top}
        \\
        \left.
        \left(
            \nabla_{x}\log p_\theta(x\mid Y)
            -\nabla_{x}\log p_\theta(x)
        \right)
        \right|_{x=g_\theta(Y,Z)}
    \biggr].
    \label{eq:mi-gradient}
\end{multline}
This identity expresses the MI gradient in terms of the generator
Jacobian and the difference between the conditional and marginal
scores. Equivalently, this score difference is the input
gradient of the log density ratio
$\log[p_\theta(x\mid y)/p_\theta(x)]$.
If the scores are known exactly, independent samples from
$p_Yq_Z$ yield an unbiased Monte Carlo estimator of
\Cref{eq:mi-gradient}.
In implicit models, however, the score difference must generally be estimated, introducing approximation error into the generator gradient.
Moreover, the output laws may be singular, in which case the required densities and scores need not exist.
These issues motivate the spread mutual information formulation introduced next.


\section{Spread Mutual Information}
\label{sec:spread-mutual-information}


To address these limitations, we measure dependence after passing the generated variable through a stochastic channel.
At any positive Gaussian noise level, spreading turns even singular laws into smooth, strictly positive conditional and marginal densities with common support.
This makes the required scores well defined and eliminates exact support separation, which can facilitate density-ratio estimation.


This idea underlies spread KL divergence \citep[SKL;][]{zhang_spread_2020}.
Given probability measures $p$ and $q$ and a spreading kernel $k$, SKL compares the corresponding spread distributions:
\begin{equation}
    \mathbb D_{\mathrm{KL}}(p*k \,\|\, q*k),
    \ \ 
    (p*k)(\tilde x)
    =
    \int k(\tilde x\mid x)\,p(\mathrm dx).
\end{equation}
A common choice is the Gaussian spearding kernel $k(\tilde x|x)=\mathcal{N}(\tilde x;\alpha x,\sigma^2I)$ with $\alpha,\sigma>0$.
This kernel produces smooth, strictly positive densities, even when
the original distributions are singular or have disjoint supports.
Moreover, since both scaling by $\alpha>0$ and Gaussian convolution
are injective, this choice preserves identifiability:
$\mathrm{SKL}=0\Leftrightarrow p=q$.

Mutual information is the KL divergence between the joint distribution
of $(X,Y)$ and the product of its marginals. We define
\textit{spread mutual information} by applying the same channel to $X$
in both distributions, leaving $Y$ unchanged.
\begin{definition}[Spread Mutual Information]
\itshape
Let $(X,Y)$ be jointly distributed, with $X\in\mathbb{R}^d$
and $Y\sim q_Y$ taking values in a standard Borel space.
For a probability transition kernel $k$ on $\mathbb{R}^d$,
let $\tilde X\mid(X,Y)\sim k(\cdot\mid X)$.
The \textit{spread mutual information} w.r.t. $k$ is
\begin{equation}
    I_k(X;Y)\equiv I(\tilde X;Y).
\end{equation}
\end{definition}
In particular, this definition allows discrete or continuous $Y$ and requires
no density for $X$.
When $k$ admits a transition density, denoted by
$k(\tilde x\mid x)$, the spread conditional and marginal densities are
\begin{align}
    \tilde p(\tilde x\mid y)
    &=\mathbb{E}[k(\tilde x\mid X)\mid Y=y],\\
    \tilde p(\tilde x)
    &=\mathbb{E}[k(\tilde x\mid X)]
      =\mathbb{E}_Y[\tilde p(\tilde x\mid Y)].
\end{align}
Consequently,
\begin{equation}
    I_k(X;Y)
    =\mathbb{E}_Y\!\left[
        \mathbb{D}_{\mathrm{KL}}\!\left(
            \tilde p(\cdot\mid Y)\,\|\,\tilde p
        \right)
    \right].
\end{equation}
SMI is nonnegative and bounded above by MI. If the spreading
operation is injective, it also detects independence exactly.

\begin{proposition}[Data processing and identifiability]\label{prop:smi-properties}
Spread mutual information satisfies
\begin{equation}
    0\le I_k(X;Y)\le I(X;Y).
\end{equation}
If the spreading operation is injective, i.e.
$r*k=s*k$ implies $r=s$ for any input distributions
$r$ and $s$, then
\begin{equation}
    I_k(X;Y)=0
    \quad\Longleftrightarrow\quad
    X\perp Y.
\end{equation}
\end{proposition}

In practice, we aggregate SMI across multiple noise levels, with the
noise distribution and weighting determining the scales at which
dependence is measured.

\subsection{Gradients of SMI}
For an implicit generator $g_\theta(y,z)$ and a fixed transition
kernel $k$, we optimise $I_\theta(\tilde X;Y)$ through its gradient
with respect to the generator parameters $\theta$.
\begin{proposition}[SMI gradient, informal]
\label{prop:smi-gradient}
Let $\tilde X=h(g_\theta(Y,Z),\epsilon)$ be the output of a
generator followed by a fixed, reparameterizable spreading
channel. Under suitable regularity conditions,
\begin{multline}
    \nabla_\theta I_\theta(\tilde X;Y)
    =\mathbb{E}\biggl[
        \biggl(\frac{\partial\tilde X}{\partial\theta}\biggr)^{\!\top}\\
        \bigl(
            \nabla_{\tilde x}
            \log \tilde p_\theta(\tilde x\mid y)
            -
            \nabla_{\tilde x}
            \log \tilde p_\theta(\tilde x)
        \bigr)        
        \bigg|_{\tilde x=\tilde X}
    \biggr],
\label{eq:smi-gradient}
\end{multline}
The formal statement appears in \Cref{prop:smi-gradient-formal}.
\end{proposition}

\begin{figure}[t]
    \centering
    \includegraphics[width=0.645\linewidth]
        {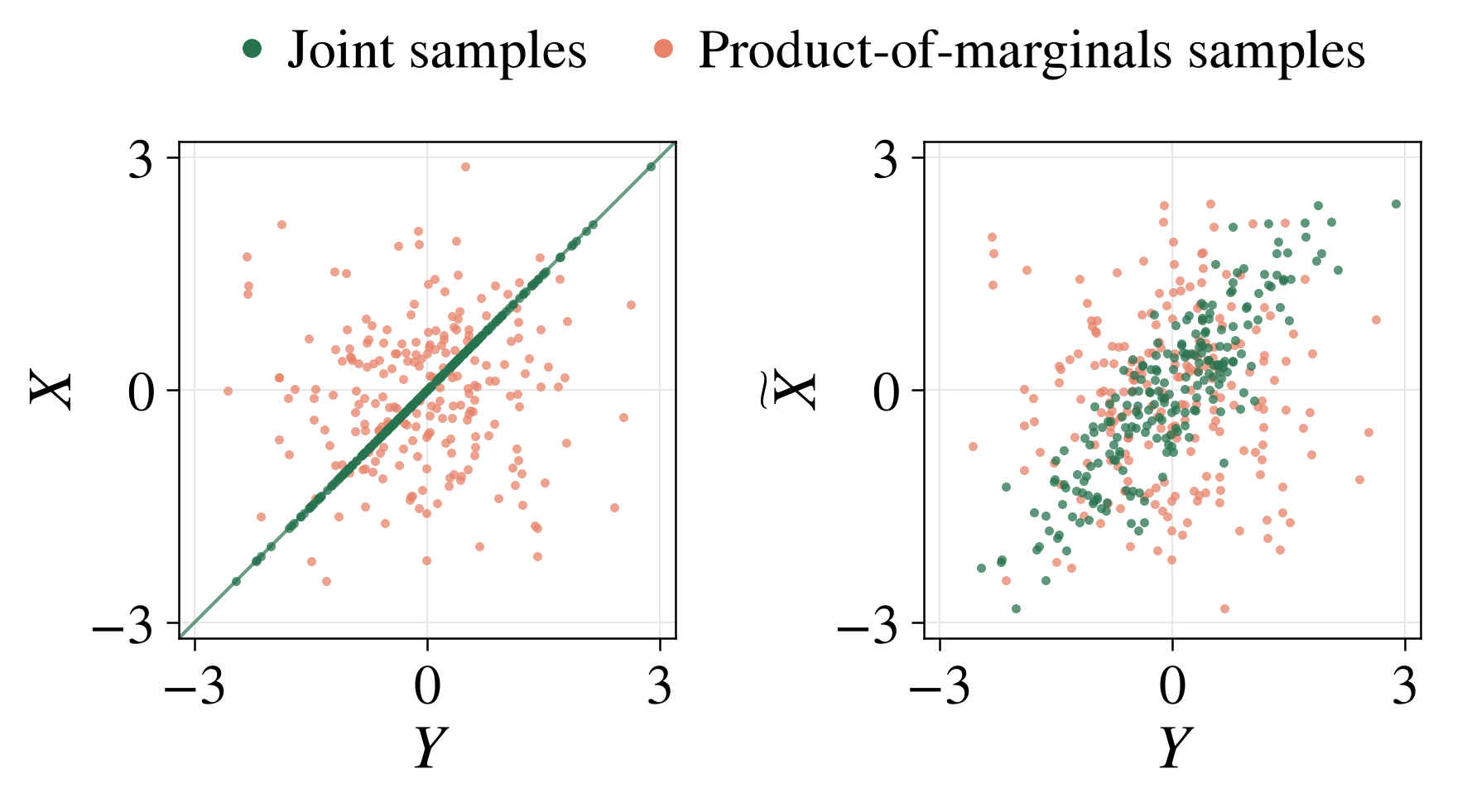}
    \hfill
    \includegraphics[width=0.335\linewidth]
        {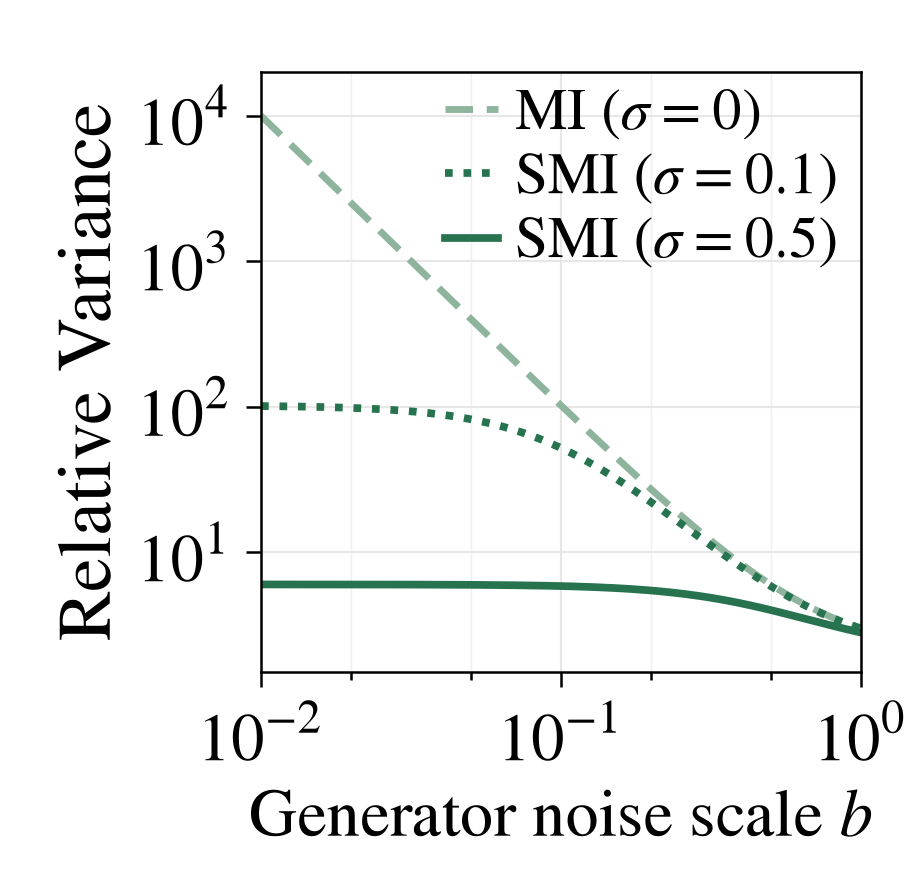}
    \caption{
    \textbf{Gaussian spreading near a singular limit.}
    \textbf{Left pair:} {\color{green}Joint} and
    {\color{orange}product-of-marginals} samples at
    $b=0$, before (leftmost) and after (middle) spreading with $\sigma=0.5$.
    \textbf{Right:} Relative variance of the pathwise estimator
    of $\partial_a I(\widetilde{X};Y)$ using exact scores.
    }
    \label{fig:smi-singularity}
\end{figure}

The model in \Cref{fig:smi-singularity} illustrates how
Gaussian spreading regularises score-based gradient estimation
near singular generator outputs. In particular, we consider $X=aY+bZ$ and $\widetilde{X}=X+\sigma\epsilon$,
    with $a=1$ and independent standard Gaussian $Y,Z,\epsilon$.
At $b=0$, $X\mid Y$ is a point mass, so its score is undefined
and $I(X;Y)=\infty$.
Any $\sigma>0$ yields smooth, positive conditional and marginal
densities, making the score difference in \Cref{eq:smi-gradient}
well-defined.
More generally, for $a>0$ and $b^2+\sigma^2>0$, the single-sample
pathwise estimator $G_\sigma$ of
$\partial_a I(\widetilde{X};Y)$ in \Cref{eq:smi-gradient},
evaluated using exact scores, satisfies
\begin{equation}
    \frac{\operatorname{Var}(G_\sigma)}
         {(\mathbb{E}G_\sigma)^2}
    =2+\frac{a^2}{b^2+\sigma^2}.\label{eq:toy-relative-variance}
\end{equation}
Without spreading ($\sigma=0$), the relative variance \Cref{eq:toy-relative-variance} diverges as $b\to0$, whereas for any fixed $\sigma>0$ it remains bounded by $2+a^2/\sigma^2$.
This reveals an advantage of SMI in highly correlated settings: Gaussian spreading preserves well-defined scores and bounded relative gradient variance near the deterministic limit, where the unspread law becomes singular and MI diverges.


\subsection{Estimating SMI gradients through density ratios}
\label{sec:smi-class-ratio}

The main challenge in estimating
$\nabla_\theta I_\theta(\tilde X;Y)$ using
\Cref{eq:smi-gradient} is estimating the score difference:
{
\fontsize{9.5pt}{11.4pt}\selectfont
\begin{equation}
    \nabla_{\tilde x}\log\tilde p_\theta(\tilde x\mid y)
      -\nabla_{\tilde x}\log\tilde p_\theta(\tilde x)
    =\nabla_{\tilde x}\log
        \frac{\tilde p_\theta(\tilde x\mid y)}
             {\tilde p_\theta(\tilde x)}.
\label{eq:smi-score-diff}
\end{equation}}

Following \cite{chen_diffratio_2025}, we estimate this field by
differentiating a learned log density ratio. Let
$c_\phi(\tilde x,y)\in(0,1)$ predict whether a pair is drawn from the
\emph{joint} distribution (label $1$) or the product of its marginals
(label $0$). With $\theta$ fixed, minimise the balanced cross-entropy
objective
\begin{multline}
    \mathcal{L}_{\mathrm{clf}}(\phi)
    =-\mathbb{E}[\log c_\phi(\tilde X,Y)]
       \\-\mathbb{E}[\log(1-c_\phi(\tilde X,\bar Y))],
\end{multline}
where $\bar Y\sim p_Y$ is independent of $(\tilde X,Y)$.
When the spread densities are positive, as for Gaussian spreading,
the population objective over all measurable classifiers is minimised by
\begin{equation}
    c_\phi^*(\tilde x,y)
    =\frac{\tilde p_\theta(\tilde x\mid y)}
           {\tilde p_\theta(\tilde x\mid y)+\tilde p_\theta(\tilde x)},
\end{equation}
uniquely up to null sets \citep{gutmann_noisecontrastive_2010}. The score difference can be estimated through
\begin{align}
    \nabla_{\tilde x}\log
        \frac{\tilde p_\theta(\tilde x\mid y)}
             {\tilde p_\theta(\tilde x)}=\nabla_{\tilde x}\log\frac{c_\phi^*(\tilde x, y)}{1-c_\phi^*(\tilde x, y)}.\label{eq:classifier-to-score-diff}
\end{align}

\subsection{SMI with diffusion kernels}
\label{sec:smi-with-diffusion-kernels}
We use Gaussian noising kernels of the form employed in diffusion
models \citep{song_scorebased_2021,ho_denoising_2020}. Write $X_0:=g_\theta(Y,Z)$ and
let $X_t$ denote the corresponding spread variable at noise level $t$:
\begin{equation}
    k_t(x_t\mid x_0)
    =\mathcal{N}(x_t;\alpha_t x_0,\sigma_t^2I_d).
\end{equation}

The corresponding reparameterization is
\begin{equation}
    X_t=h_t(X_0,\epsilon)
       =\alpha_tX_0+\sigma_t\epsilon,
    \ \ \epsilon\sim\mathcal{N}(0,I_d).
    \label{eq:diffusion-noising-x0}
\end{equation}

The schedule is fixed with respect to $\theta$. We assume
$\mathbb{E}\|X_0\|^2<\infty$ and use $\alpha_t,\sigma_t>0$
at the sampled training levels. Positive noise gives smooth,
positive conditional and marginal densities, denoted by
$p_\theta(x_t\mid y)$ and $p_\theta(x_t)$, respectively.
The parameter-gradient identity additionally requires the
differentiability and integrability assumptions of
\Cref{prop:smi-gradient}.

Since $I_\theta(X_t;Y)=I_\theta(X_0+\sqrt{\tau_t}\epsilon;Y)$ (see \Cref{app:effective-noise-variance-proof}),
the effective noise variance is $\tau_t=\sigma_t^2/\alpha_t^2$.
Increasing $\tau_t$ reduces the information about $Y$, whereas
$\tau_t\to0$ recovers the unspread MI, as established in
\Cref{thm:smi-information-dissipation}. 
Small effective noise variances retain fine-scale dependence, while larger variances measure dependence that persists under stronger perturbations. To control dependence across these scales, we average over $t\sim p_t$, with the sampling distribution $p_t$ and nonnegative weights $w_{\mathrm{SMI}}(t)$ determining the contribution of each scale:
\begin{equation}
    \mathcal{L}_{\mathrm{SMI}}(\theta)
    =\mathbb{E}_{t}\!\left[
        w_{\mathrm{SMI}}(t)I_\theta(X_t;Y)
    \right].
    \label{eq:smi-objective}
\end{equation}
The time-dependent classifier is trained with
\begin{align}
    \mathcal{L}_{\mathrm{clf}}(\phi;t)
    &=-\mathbb{E}[\log c_\phi(X_t,Y,t)]\notag\\
    &\qquad-\mathbb{E}[\log(1-c_\phi(X_t,\bar Y,t))],\\
    \mathcal{L}_{\mathrm{clf}}(\phi)
    &=\mathbb{E}_{t}\!\left[
        w_{\mathrm{clf}}(t)\mathcal{L}_{\mathrm{clf}}(\phi;t)
    \right],
    \label{eq:smi-classifier-objective}
\end{align}
where $\bar Y\sim p_Y$ is independent of $(Y,Z,\epsilon,t)$.
The distribution $p_t$ and positive weights
$w_{\mathrm{SMI}},w_{\mathrm{clf}}$ are fixed with respect to $\theta$
and chosen so that the SMI objective is finite. A sufficient choice is
to sample a compact time interval on which the scales and weights
are bounded and $\sigma_t$ is bounded away from zero. We assume that
the classifier loss is finite at the fitted parameters and that
differentiation of SMI can be interchanged with the time expectation. Let
\begin{equation}
    v_t(x_t,y)\equiv\nabla_{x_t}\log
        \frac{p_\theta(x_t\mid y)}{p_\theta(x_t)},
    \label{eq:smi-diffusion-field}
\end{equation}
then \Cref{prop:smi-gradient} gives
\begin{align}
    &\nabla_\theta\mathcal{L}_{\mathrm{SMI}}(\theta) \nonumber\\
    =\,&\mathbb{E}\!\left[
        w_{\mathrm{SMI}}(t)\alpha_t
        \left(\frac{\partial g_\theta(y,z)}{\partial\theta}\right)^\top
        v_t(x_t,y)
    \right].
    \label{eq:smi-weighted-gradient}
\end{align}
where $v_t(x,y)$ is estimated through the classifier $c_\phi$ as \Cref{eq:classifier-to-score-diff}. The effect of estimation error in $v_t$ is analysed in \Cref{prop:smi-classifier-error}.

\subsection{Optimising SMI in practice}
SMI controls dependence but does not determine the output distribution; for example, a constant generator achieves zero SMI. We therefore use it as a regulariser alongside a task-specific loss, such as a distribution-matching, reconstruction, or prediction loss.
We therefore use it as a regulariser alongside a task-specific objective.
It could alternatively be imposed as an explicit dependence constraint.
Specifically, we optimise
\begin{align}
    \min_\theta\;
    \mathcal{L}_{\mathrm{task}}(\theta)
    + \beta\mathcal{L}_{\mathrm{SMI}}(\theta).
    \label{eq:optimise-smi-with-task-loss}
\end{align}
Here, $\beta>0$ encourages dependence minimisation, $\beta<0$ encourages maximisation, and $|\beta|$ controls its strength.
\Cref{alg:smi} summarises the training procedure.

\begin{algorithm}[t]
\caption{Dependence control with SMI}
\label{alg:smi}
\small
\KwInput{
    Generator $g_\theta$; classifier $c_\phi$;
    batch size $B$; classifier updates $N_{\mathrm{clf}}$;
    SMI strength $\beta$;
    learning rates $\eta_\theta,\eta_\phi$.
}
\While{the stopping criterion is not met}{
    \tcp{Classifier updates:\hspace{0.6em}hold $\theta$ fixed}
    \For{$j=1,\ldots,N_{\mathrm{clf}}$}{
        $y_i,\bar y_i\sim p_Y$, $z_i\sim q_Z$,
        $t_i\sim p_t$, and
        $\epsilon_i\sim\mathcal{N}(0,I_d)$\;

        $x_i\leftarrow$
        \Cref{eq:diffusion-noising-x0}\;
        $\widehat{\mathcal{L}}_{\mathrm{clf}}\gets$\Cref{eq:smi-classifier-objective}\;

        $\phi\gets\phi-\eta_\phi\nabla_\phi\widehat{\mathcal{L}}_{\mathrm{clf}}$\;
    }

    \tcp{Generator update:\hspace{0.6em}hold $\phi$ fixed}
    $y_i\sim p_Y$, $z_i\sim q_Z$,
        $t_i\sim p_t$, and
        $\epsilon_i\sim\mathcal{N}(0,I_d)$\;
        
    $x_i\gets$\Cref{eq:diffusion-noising-x0}\;

    $v_i\gets$ \Cref{eq:classifier-to-score-diff}\;

    $\nabla_\theta\widehat{\mathcal{L}}_\mathrm{SMI}\gets$ \Cref{eq:smi-weighted-gradient}\;

    Compute a minibatch task loss
 $\widehat{\mathcal{L}}_{\mathrm{task}}$\;

    $\displaystyle
        \theta\gets\theta-\eta_\theta
        \left(
            \nabla_\theta\widehat{\mathcal{L}}_{\mathrm{task}}
            +\beta\,
            {\nabla_\theta\widehat{\mathcal{L}}_{\mathrm{SMI}}}
        \right)
    $\;
}
\end{algorithm}
\section{Theoretical Analysis}
\label{sec:smi-theoretical-analysis}
We analyse SMI under the Gaussian spreading model introduced in
\Cref{sec:smi-with-diffusion-kernels}, focusing on regularity and
bounds on gradient sampling and estimation errors.
All proofs and supplementary results are provided in
\Cref{sec:smi-theory-proofs}.
The appendix additionally compares the relative variance of
exact-score SMI and standard-MI gradient estimators in a binary
Gaussian model (\Cref{sec:smi-theory-comparison}) and characterizes
information dissipation, limiting behavior, and scale sensitivity
across noise levels (\Cref{sec:smi-theory-dissipation}).


\subsection{Regularity and gradient moments}
\label{sec:smi-theory-regularity}
\label{sec:smi-theory-gradient}
Gaussian spreading gives the conditional and marginal distributions
smooth, positive densities with common support $\mathbb{R}^d$ and
finite SMI, even when $X_0$ has no density
(\Cref{thm:smi-smooth-finite}). For $\alpha_t\ne0$, it also preserves
the equivalence between zero SMI and independence. Parameter
differentiability still requires the assumptions of
\Cref{prop:smi-gradient}. Additional bounds on spatial derivatives
and a sufficient condition for strong log-concavity under bounded
generator outputs are given in \Cref{thm:smi-field-regularity}.

By \Cref{prop:smi-gradient}, a single-sample estimator of the
fixed-level SMI gradient is
\begin{equation}
    G_t=\alpha_t
        \left(\frac{\partial g_\theta(Y,Z)}{\partial\theta}\right)^{\!\top}
        v_t(X_t,Y).
    \label{eq:theory-single-gradient}
\end{equation}
Here $G_t$ is a random vector whose dependence on $(Y,Z,\epsilon)$
is suppressed for brevity.
Singular unspread laws may lack ambient-space densities, so the
scores appearing in the MI-gradient identity need not exist
\citep{song_generative_2019}.
Controlling the second moment of stochastic gradient estimates is
standard in convergence analyses of stochastic gradient methods
\citep{bottou_optimization_2018}.
The following result shows that, at every fixed positive noise level,
Gaussian spreading yields an exact-score SMI gradient estimator with
finite second moment under a bounded generator Jacobian, without
requiring bounded generator outputs.
Consequently, averaging $B$ independent samples gives an
$O(B^{-1})$ mean-squared sampling-error bound.


\begin{proposition}[Second-moment bounds for SMI gradients]
\label{thm:smi-gradient-moments}
Assume the regularity conditions of \Cref{prop:smi-gradient} and,
for some $M<\infty$,
\begin{equation}
    \left\|\frac{\partial g_\theta(Y,Z)}{\partial\theta}
    \right\|_{\mathrm{op}}\le M
    \qquad\text{almost surely}.
    \label{eq:theory-jacobian-bound}
\end{equation}
Then $\mathbb{E}[G_t]=\nabla_\theta I_\theta(X_t;Y)$ and %
\vspace{-1em}
\par
\begingroup
\fontsize{9.5}{11.5}\selectfont
\noindent
\begin{minipage}{0.48\linewidth}
\begin{equation}
    \mathbb{E}\|v_t(X_t,Y)\|^2 \le \frac{d}{\sigma_t^2},
    \label{eq:theory-score-moment}
\end{equation}
\end{minipage}\hfill
\begin{minipage}{0.48\linewidth}
\begin{equation}
    \mathbb{E}\|G_t\|^2 \le \frac{M^2d\alpha_t^2}{\sigma_t^2}.
    \label{eq:theory-gradient-moment}
\end{equation}
\end{minipage}
\par
\endgroup
For $B$ independent draws of $(T,Y,Z,\epsilon)$, the weighted
estimator satisfies
\begin{multline}
    \mathbb{E}\left\|
        \frac1B\sum_{i=1}^B w_{\mathrm{SMI}}(t_i)G_{t_i,i}
        -\nabla_\theta\mathcal{L}_{\mathrm{SMI}}
    \right\|^2\\
    \le\frac{M^2d}{B}\,
       \mathbb{E}_t\!\left[
           w_{\mathrm{SMI}}(t)^2\frac{\alpha_t ^2}{\sigma_t^2}
       \right],
    \label{eq:theory-weighted-gradient-mse}
\end{multline}
provided the RHS is finite and differentiation can be
interchanged with the time expectation in \Cref{eq:smi-objective}.
\end{proposition}
\Cref{thm:smi-gradient-moments} demonstrates that the gradient second-moment bound
scales linearly with the signal-to-noise ratio
$\alpha_t^2/\sigma_t^2$, yielding quantitative control of minibatch
sampling error. For estimation across noise levels, the integrability
criterion guides the choice of time distribution and weights,
particularly near zero noise.
The proof is given in \Cref{app:smi-gradient-moments-proof}.

\subsection{Effect of classifier estimation error}
\label{sec:smi-theory-classifier-error}
In practice, differentiating the classifier's log-odds gives an
approximation $\widehat v_t$ to the score difference $v_t$.
The following result quantifies how its error
$\delta_t$ propagates to gradient bias and minibatch sampling error.

\begin{proposition}[Effect of classifier estimation error]
\label{prop:smi-classifier-error}
Under the assumptions of \Cref{thm:smi-gradient-moments}, fix a
differentiable classifier $c_\phi$ taking values in $(0,1)$ and define
\begin{align}
    \widehat v_t(x_t,y)
    &=\nabla_{x_t}\log
        \frac{c_\phi(x_t,y,t)}{1-c_\phi(x_t,y,t)},\\
    \delta_t^2
    &=\mathbb{E}\|\widehat v_t(X_t,Y)-v_t(X_t,Y)\|^2<\infty.
    \label{eq:theory-classifier-error}
\end{align}
Let $\widehat G_t=\alpha_t
        \left(\frac{\partial g_\theta(Y,Z)}{\partial\theta}\right)^{\!\top}
        \widehat v_t(X_t,Y)$, then
\begin{equation}
    \|\mathbb{E}\widehat G_t-\nabla_\theta I_\theta(X_t;Y)\|
    \le M\alpha_t\delta_t.
    \label{eq:theory-gradient-bias}
\end{equation}
For $B$ independent samples,
\begin{multline}
    \mathbb{E}\left\|
        \frac1B\sum_{i=1}^B\widehat G_{t,i}
        -\nabla_\theta I_\theta(X_t;Y)
    \right\|^2\\
    \le\frac{2M^2\alpha_t^2}{B}
       \left(\frac d{\sigma_t^2}+\delta_t^2\right)
       +M^2\alpha_t^2\delta_t^2.
    \label{eq:theory-estimated-gradient-mse}
\end{multline}
For a randomly fitted classifier and generator, the statements hold
conditionally on their fitted state, using fresh evaluation samples.
\end{proposition}

The bound separates $O(B^{-1})$ sampling error from a squared-bias
term controlled by $\delta_t^2$. Gaussian spreading increases overlap
between the joint and product-of-marginals distributions, which can
ease density-ratio estimation by mitigating near-separation
\citep{gutmann_noisecontrastive_2010,rhodes_telescoping_2020}.
Together with derivative control
(\Cref{prop:smi-risk-counterexample}), such as Sobolev regularization
\citep{xie_sobolev_2026}, this provides a rationale for more reliable
score-difference estimation.
The proof and weighted extension are in
\Cref{app:smi-classifier-error-proof,cor:smi-weighted-classifier-error}.

\definecolor{smirow}{HTML}{ECF5EF}

\section{Experiments}
\label{sec:experiments}

We evaluate SMI on synthetic and real-world tasks spanning dependence minimisation and maximisation.
We compare against MI-based methods and, where appropriate, task-specific baselines.
Within each task, MI-based methods share the same generator or representation architecture.
See \cref{app:experiments} for implementation details and additional results.

\subsection{Synthetic Experiments}
\label{sec:exp-synthetic}

\begin{figure*}[!t]
    \centering
    \includegraphics[width=\textwidth]{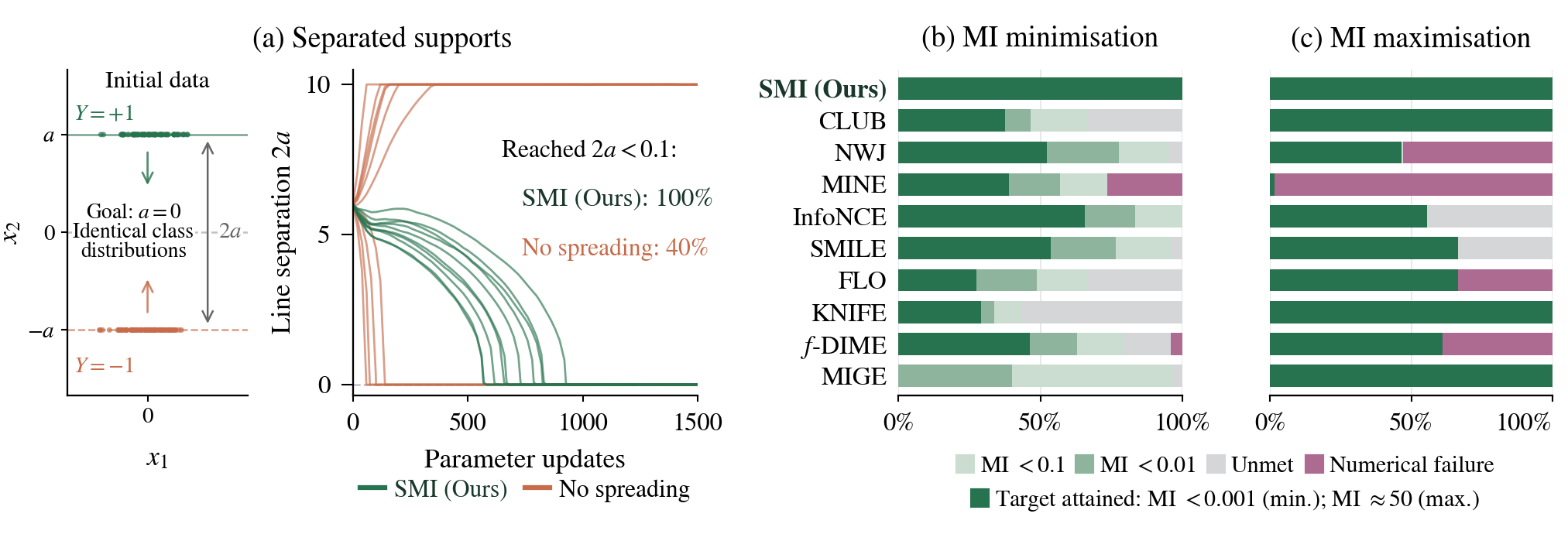}
    \caption{\textbf{Synthetic experiments.}
    (a) Conditional supports $x_2=\pm a$ and optimisation of their separation $2a$.
    (b,c) MI minimisation and maximisation results, pooled over Gaussian, cubic, asinh, signed-power and spiral data; $d\in\{20,100,200\}$; a shared correlation coefficient or one per latent coordinate pair; and ten seeds.}
    \label{fig:synthetic-summary}
\end{figure*}

\paragraph{Support geometry.}
Let $X_a=(Z,aY)$, with independent $Z\sim\mathcal N(0,1)$ and $Y\sim\operatorname{Unif}\{-1,+1\}$.
The conditional supports are parallel lines separated by $2a$ (\cref{fig:synthetic-summary}a).
Independence occurs at $a=0$, but $I(X_a;Y)=\log 2$ for every $a>0$, so its gradient cannot guide their contraction.
Gaussian spreading makes the objective sensitive to this separation.
SMI reaches $2a<0.1$ in all ten runs, versus four of ten without spreading, although those four reach the target sooner.

\paragraph{Gradient accuracy.}
Following \citet{chen_diffratio_2025}, we compare SMI with subtracting conditional and marginal scores fitted separately by denoising score matching (DSM) \citep{vincent_connection_2011}, using fixed Gaussian-smoothed discrete distributions.
With errors pooled over noise levels, SMI has lower mean relative errors for both score differences and parameter gradients (\cref{tab:synthetic-gradient}).

\paragraph{MI optimisation.}
We optimise correlations between $d$-dimensional Gaussian vectors under invertible transformations, using up to $5{,}000$ generator updates (\cref{fig:synthetic-summary}b,c).
SMI samples $\sigma\sim\operatorname{LogUniform}(0.01,1)$; evaluation uses analytic MI before spreading.
Starting from $10$ nats, SMI keeps MI below $10^{-3}$ throughout the final $1{,}000$ updates in all minimisation runs.
At $d=100$ and $200$, the best baseline attains this criterion in only $50\%$ and $48\%$ of runs, respectively.
Starting from $0.5$ nats, SMI, CLUB \citep{cheng_club_2020}, KNIFE \citep{pichler_differential_2022} and MIGE \citep{wen_mutual_2020} reach the $50$-nat maximisation target in every run without numerical failure.

\begin{table}[!t]
    \centering\small
    \setlength{\tabcolsep}{4pt}
    \caption{\textbf{Gradient accuracy.} Relative errors ($\downarrow$), mean $\pm$ SD over ten seeds.}
    \label{tab:synthetic-gradient}
    \begin{tabularx}{\columnwidth}{l*{2}{>{\centering\arraybackslash}X}}
        \toprule
        Estimator & Score difference & Parameter gradient \\
        \midrule
        \rowcolor{smirow}
        \textbf{SMI} & $\mathbf{0.408\pm0.030}$ & $\mathbf{0.275\pm0.031}$ \\
        DSM scores & $2.608\pm0.376$ & $6.225\pm1.834$ \\
        \bottomrule
    \end{tabularx}
\end{table}

\paragraph{Computational costs.}
SMI uses a single density-ratio network and has time and memory costs that scale linearly with minibatch size for a fixed architecture.
Empirically, its per-iteration cost is comparable to other classifier-based MI-control methods, while using substantially less memory than MIGE.
See \cref{app:synthetic-cost} for detailed results and complexity analysis.

\subsection{Missing Data Imputation}
\label{sec:real-world}

\begin{table*}[t]
\centering\small
\caption{\textbf{Missing data imputation.} Ranks across all methods ($\downarrow$), averaged over ten datasets at each missingness rate; mean $\pm$ SD over three seeds. Failed runs rank last.}
\label{tab:tabular-imputation}
\begin{tabularx}{\textwidth}{l*{6}{>{\centering\arraybackslash}X}}
\toprule
 & \multicolumn{2}{c}{(a) 30\% MCAR} & \multicolumn{2}{c}{(b) 60\% MCAR} & \multicolumn{2}{c}{(c) 80\% MCAR} \\
\cmidrule(lr){2-3}\cmidrule(lr){4-5}\cmidrule(lr){6-7}
Method & MMD$^2$ rank & RMSE rank & MMD$^2$ rank & RMSE rank & MMD$^2$ rank & RMSE rank \\
\midrule
\multicolumn{7}{l}{\textit{MI-based methods}} \\
\rowcolor{smirow}
\textbf{SMI (Ours)} & $\mathbf{5.43\pm0.31}$ & $\mathbf{7.63\pm0.15}$ & $\mathbf{4.10\pm0.10}$ & $7.63\pm0.06$ & $\mathbf{2.87\pm0.21}$ & $\underline{7.33\pm0.35}$ \\
CLUB & $12.30\pm0.66$ & $9.93\pm0.21$ & $11.97\pm0.21$ & $\mathbf{7.37\pm0.12}$ & $11.00\pm0.20$ & $\mathbf{5.33\pm0.12}$ \\
NWJ & $11.90\pm0.20$ & $13.63\pm0.23$ & $8.43\pm0.40$ & $11.80\pm0.20$ & $6.47\pm0.83$ & $10.53\pm0.95$ \\
MINE & $10.33\pm0.83$ & $12.90\pm0.40$ & $7.43\pm0.32$ & $11.27\pm0.15$ & $7.40\pm0.66$ & $10.40\pm0.53$ \\
InfoNCE & $11.67\pm0.12$ & $13.50\pm0.17$ & $9.20\pm0.20$ & $11.20\pm0.17$ & $7.53\pm0.29$ & $9.20\pm0.30$ \\
SMILE & $10.23\pm0.40$ & $12.03\pm0.06$ & $7.80\pm0.36$ & $10.50\pm0.17$ & $6.83\pm0.21$ & $8.13\pm0.12$ \\
FLO & $11.73\pm0.25$ & $12.97\pm0.25$ & $8.93\pm0.55$ & $13.30\pm0.40$ & $7.33\pm0.21$ & $12.50\pm0.17$ \\
KNIFE & $\underline{6.10\pm0.20}$ & $\underline{8.17\pm0.06}$ & $\underline{5.47\pm1.07}$ & $\underline{7.60\pm0.30}$ & $\underline{4.40\pm0.87}$ & $7.63\pm0.35$ \\
$f$-DIME & $9.13\pm0.68$ & $11.40\pm0.46$ & $10.43\pm1.14$ & $12.93\pm0.71$ & $11.67\pm0.97$ & $13.30\pm1.05$ \\
MIGE & $7.23\pm0.64$ & $9.87\pm0.06$ & $10.80\pm0.36$ & $15.10\pm0.10$ & $9.67\pm0.06$ & $14.33\pm0.21$ \\
\midrule
\multicolumn{7}{l}{\textit{Established imputers}} \\
GAIN & $12.83\pm1.65$ & $4.40\pm0.44$ & $15.40\pm0.10$ & $3.83\pm0.25$ & $15.27\pm0.21$ & $7.97\pm1.88$ \\
MIWAE & $6.17\pm0.51$ & $4.53\pm0.12$ & $9.63\pm0.45$ & $4.20\pm0.17$ & $10.73\pm0.21$ & $3.27\pm0.57$ \\
MIRI & $\underline{3.30\pm0.30}$ & $6.90\pm0.52$ & $\mathbf{3.20\pm0.61}$ & $9.70\pm0.95$ & $\mathbf{6.27\pm0.40}$ & $11.77\pm0.51$ \\
WGF & $\mathbf{2.33\pm0.49}$ & $5.03\pm0.06$ & $\underline{3.90\pm0.10}$ & $6.07\pm0.57$ & $\underline{7.77\pm1.21}$ & $8.73\pm1.22$ \\
HyperImpute & $3.77\pm0.06$ & $\mathbf{1.17\pm0.06}$ & $7.43\pm0.64$ & $\underline{1.87\pm0.06}$ & $9.13\pm0.87$ & $\underline{2.97\pm0.55}$ \\
MissForest & $11.53\pm0.59$ & $\underline{1.93\pm0.06}$ & $11.87\pm0.06$ & $\mathbf{1.63\pm0.23}$ & $11.73\pm0.21$ & $\mathbf{2.67\pm0.59}$ \\
\bottomrule
\end{tabularx}
\end{table*}

Following \citet{liu_minimizing_2024,yu_missing_2025}, we reduce dependence between the completed data and the missingness mask while keeping observed entries fixed:
\begin{equation}
\label{eq:imputation-mi}
\begin{aligned}
\min_\theta\; & I(X_\theta;M),\qquad\text{where}\\
X_\theta={}&M\odot X+(1-M)\odot g_\theta(M\odot X,M,Z).
\end{aligned}
\end{equation}
Here, $M_j=1$ denotes an observed entry and $Z$ provides imputation randomness.
SMI instead minimises $\mathbb E_\sigma I(X_\theta+\sigma\epsilon;M)$, with $\sigma\sim\operatorname{LogUniform}(0.01,1)$.

We impute numerical features from ten TabArena datasets \citep{erickson_tabarena_2025}, with entries missing completely at random (MCAR) \citep{rubin_inference_1976} at rates of $30\%$, $60\%$, and $80\%$.
We repeat masking and training over three seeds, sharing masks across methods and fitting preprocessing only to observed entries.
MI-control methods share a residual MLP imputer without reconstruction penalties.
Established imputers, including the MI-based WGF \citep{liu_minimizing_2024} and MIRI \citep{yu_missing_2025}, retain their task-specific procedures.

\Cref{tab:tabular-imputation} ranks methods by masked-entry RMSE and squared maximum mean discrepancy (MMD$^2$) between completed and reference rows \citep{gretton_kernel_2012}.
SMI achieves the lowest mean MMD$^2$ rank among matched MI methods at every missingness rate and among all methods at $80\%$, with no numerical failures.
HyperImpute \citep{jarrett_hyperimpute_2022} and MissForest
\citep{stekhoven_missforest_2012} achieve better mean RMSE ranks,
consistent with their iterative column-wise prediction objectives, which are more directly aligned with pointwise reconstruction.

\subsection{Nuisance Suppression on Waterbirds}
\label{sec:waterbirds}

\begin{figure}[!t]
\centering
\includegraphics[width=\columnwidth]{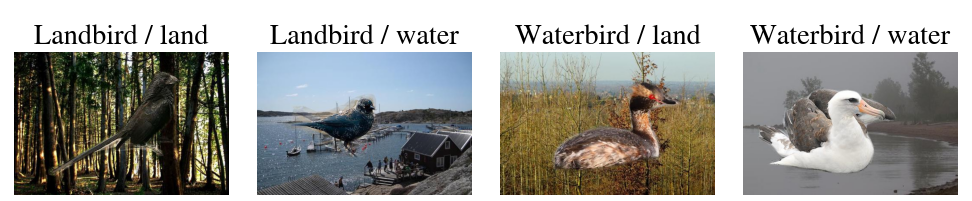}
\caption{\textbf{Waterbirds.} Training examples from the four bird--background groups.}
\vspace{-0.5em}
\label{fig:waterbirds-demo}
\end{figure}

\begin{table}[t]
\centering\small
\setlength{\tabcolsep}{1pt}
\caption{\textbf{Waterbirds.} Test results; mean $\pm$ SD over five training seeds.}
\label{tab:waterbirds-main}
\begin{tabularx}{\columnwidth}{lc*{2}{>{\centering\arraybackslash}X}}
\toprule
 & Worst-group & \multicolumn{2}{c}{Leakage (points, $\downarrow$)} \\
\cmidrule(lr){3-4}
Method & accuracy (\%, $\uparrow$) & Linear & MLP \\
\midrule
\multicolumn{4}{l}{\textit{MI-based methods}} \\
\rowcolor{smirow}
\textbf{SMI (Ours)} & $\mathbf{71.68\pm1.33}$ & $\mathbf{18.19\pm5.35}$ & $\mathbf{9.20\pm5.49}$ \\
MINE & $70.87\pm1.16$ & $26.21\pm2.42$ & $22.23\pm4.00$ \\
SMILE & $71.43\pm1.83$ & $25.34\pm3.15$ & $17.69\pm1.76$ \\
InfoNCE & $71.43\pm1.40$ & $25.59\pm3.14$ & $18.12\pm4.09$ \\
CLUB & $70.62\pm1.32$ & $24.82\pm1.78$ & $20.95\pm2.62$ \\
NWJ & $70.34\pm1.09$ & $27.18\pm4.59$ & $18.82\pm2.32$ \\
FLO & $\underline{71.59\pm0.95}$ & $\underline{20.94\pm3.08}$ & $\underline{14.39\pm2.36}$ \\
KNIFE & $65.79\pm1.91$ & $25.70\pm5.60$ & $21.67\pm4.78$ \\
$f$-DIME & $71.37\pm1.01$ & $33.69\pm4.24$ & $34.50\pm3.65$ \\
MIGE & $67.16\pm1.50$ & $30.55\pm2.21$ & $30.34\pm2.55$ \\
\midrule
\multicolumn{4}{l}{\textit{Other methods}} \\
Adversarial & $\underline{72.31\pm0.80}$ & $\underline{20.50\pm6.49}$ & $\mathbf{12.22\pm3.65}$ \\
HSIC & $\mathbf{73.05\pm0.87}$ & $\mathbf{16.12\pm1.77}$ & $15.41\pm2.15$ \\
ERM & $68.57\pm1.41$ & $36.40\pm8.88$ & $37.40\pm7.63$ \\
GroupDRO & $65.95\pm1.55$ & $21.15\pm6.29$ & $\underline{13.12\pm8.63}$ \\
DFR & $68.27\pm4.21$ & $32.91\pm6.04$ & $30.23\pm3.37$ \\
\bottomrule
\end{tabularx}
\vspace{-0.5em}
\end{table}

On Waterbirds \citep{sagawa_distributionally_2020}, $95\%$ of training examples pair landbirds with land or waterbirds with water (\cref{fig:waterbirds-demo}).
We learn a 64-dimensional representation $R_\theta$ from frozen ResNet-50 features \citep{he_deep_2016}, combining bird-class cross-entropy with class-conditional background suppression (\Cref{app:nuisance}).

Worst-group accuracy is the minimum test accuracy across four bird--background groups.
Background leakage is measured using separately fitted linear and MLP probes, as the class-averaged absolute deviation of ROC AUC from $0.5$, in percentage points (\cref{tab:waterbirds-main}).
Probe scores measure predictability rather than MI.
SMI attains the lowest mean leakage among MI-based methods under both probes, the lowest mean MLP leakage overall, and mean worst-group accuracy comparable to FLO \citep{guo_tight_2022}.
Among alternative dependence-control methods, HSIC
\citep{gretton_measuring_2005} attains higher mean worst-group accuracy and lower linear leakage, but higher MLP leakage than SMI; conditional adversarial training \citep{edwards_censoring_2016,louppe_learning_2017} attains higher worst-group accuracy but higher leakage under both probes.

\subsection{Edges-to-Photos Translation}
\label{sec:edges-control}

\begin{figure}[!t]
\centering
{\scriptsize
\makebox[0.25\columnwidth]{Input edge}%
\makebox[0.25\columnwidth]{Paired photo}%
\makebox[0.25\columnwidth]{Output 1}%
\makebox[0.25\columnwidth]{Output 2}\par}
\includegraphics[width=\columnwidth,viewport=0 0 1024 256,clip]
{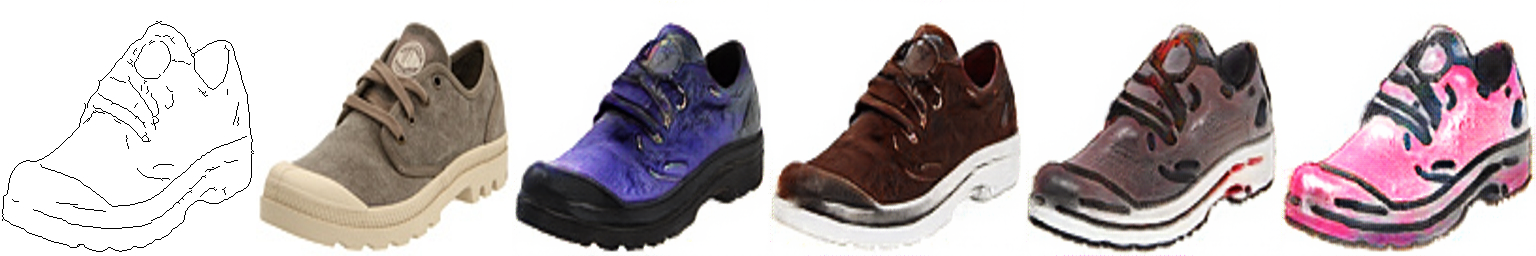}
\caption{\textbf{Edges-to-photos translation.} Input, paired photo and two SMI outputs (different latents).}
\vspace{-1em}
\label{fig:edges-demo}
\end{figure}

\begin{table}[t]
\centering\small
\setlength{\tabcolsep}{1pt}
\medmuskip=2mu
\caption{\textbf{Edges2Shoes.} Latent-code prediction $R^2$ ($\uparrow$); mean $\pm$ SD over three training seeds.}
\label{tab:edges-main}
\begin{tabularx}{\columnwidth}{l*{3}{>{\centering\arraybackslash}X}}
\toprule
 & \multicolumn{3}{c}{Latent-code prediction $R^2$ ($\uparrow$)} \\
\cmidrule(lr){2-4}
Method & Unperturbed & JPEG (75) & Resize \\
\midrule
\multicolumn{4}{l}{\textit{MI-based methods}} \\
\rowcolor{smirow}
\textbf{SMI (Ours)} & $\mathbf{0.798\pm0.008}$ & $\mathbf{0.623\pm0.045}$ & $\mathbf{0.726\pm0.032}$ \\
InfoNCE & $0.757\pm0.004$ & $0.588\pm0.032$ & $0.623\pm0.071$ \\
NWJ & $\underline{0.788\pm0.008}$ & $\underline{0.596\pm0.047}$ & $\underline{0.682\pm0.033}$ \\
FLO & $0.725\pm0.001$ & $0.503\pm0.028$ & $0.619\pm0.025$ \\
KNIFE & $0.765\pm0.007$ & $0.589\pm0.065$ & $0.678\pm0.022$ \\
MIGE & $0.742\pm0.001$ & $0.526\pm0.035$ & $0.632\pm0.031$ \\
\midrule
\multicolumn{4}{l}{\textit{Other methods and reference}} \\
Task only & $0.725\pm0.008$ & $0.495\pm0.035$ & $0.626\pm0.015$ \\
Latent $\ell_1$ & $\mathbf{0.798\pm0.001}$ & $\underline{0.584\pm0.044}$ & $\mathbf{0.686\pm0.043}$ \\
DSGAN & $\underline{0.735\pm0.005}$ & $\mathbf{0.594\pm0.026}$ & $\underline{0.682\pm0.014}$ \\
Pretrained & $0.732$ & $0.431$ & $0.655$ \\
\bottomrule
\end{tabularx}
\vspace{-0.5em}
\end{table}

We encourage a conditional generator to retain information about its latent code:
\begin{equation}
\label{eq:edges-objective}
\min_\theta\;\mathcal L_{\mathrm{task}}(\theta)+\beta I(H_\theta;Z),
\  H_\theta=\Phi(G_\theta(C,Z)),
\end{equation}
where $\beta<0$, $C$ is an edge map, $Z\sim\mathcal N(0,I_8)$ is independent of $C$, and $\Phi$ extracts frozen AlexNet features \citep{krizhevsky_imagenet_2012}, followed by PCA whitening to 256 dimensions.
The MI is computed marginally over $C$.
SMI replaces it with $\mathbb E_\sigma I(H_\theta+\sigma\epsilon;Z)$, where $\epsilon\sim\mathcal N(0,I)$ and $\sigma\sim\operatorname{LogUniform}(0.01,1)$.

On Edges2Shoes \citep{yu_finegrained_2014,isola_imagetoimage_2017}, all trained methods fine-tune the same pretrained BicycleGAN generator and encoder \citep{zhu_multimodal_2017} for $20{,}000$ generator updates, sharing the adversarial, image-reconstruction, and KL losses.
Additional baselines augment this shared objective with either latent $\ell_1$ reconstruction or DSGAN's diversity-sensitive regularisation \citep{yang_diversitysensitive_2019}.

We fit MLP regressors to predict $Z$ from frozen DINOv2 features
\citep{oquab_dinov2_2024} of generated images; \cref{tab:edges-main} reports $R^2$ averaged over latent coordinates.
Regressors trained on unperturbed outputs are reused after JPEG compression (quality 75) and resizing to half resolution and back.
SMI achieves the highest mean $R^2$ among MI-based methods in all settings and among all methods under both perturbations.
Its higher kernel Inception distance (KID) \citep{binkowski_demystifying_2018} indicates a trade-off with photo-distribution matching. See \cref{app:edges-additional} for additional results.
\section{Conclusion}
\label{sec:discussion}

We introduced SMI to mitigate singularity and poor overlap in MI optimisation.
Our analysis establishes regularity and gradient-error bounds, while experiments show strong synthetic optimisation and competitive real-world performance.
Under injective spreading, SMI vanishes exactly at independence, although small SMI need not imply small MI.
Moreover, accurate classification need not yield accurate gradients, and improved dependence control need not translate into better task performance.
Future work includes principled noise selection, derivative-aware classifier training, and task-specific trade-offs.

\clearpage

\subsection*{AI use statement}

In this work, we used generative AI tools for assist in the writing of proofs, formulate mathematical claims, provide critical ingredients for proving mathematical claims, design or provide feedback on research methodology or experiments, implement methods, and interpret results.
We have not used generative AI tools for help develop theoretical models or conceptual frameworks, propose or refine hypotheses, and clean and reformat dataset, and generate synthetic data sets, assist with translation, and support qualitative and thematic data analysis are not applicable to this work.
Additionally, we used generative AI tools for sourcing/searching for information.
We have reviewed all AI-assisted work. We checked all LLM-generated proofs, verified and tested LLM-generated code for correctness.
We take responsibility for the final content of this work,
including text, claims or artifacts produced with the aid of generative AI.

\subsubsection*{Acknowledgements}
JMHL acknowledges funding from the AI Hub in Generative Models under grant EP/Y028805/1.
RKOY acknowledges funding from the UK Engineering and Physical Sciences Research Council (EPSRC) under grant EP/L016516/1 for the University of Cambridge Centre for Doctoral Training, the Cambridge Centre for Analysis.


\bibliography{refs}

\clearpage
\appendix
\thispagestyle{empty}

\onecolumn
\etocdepthtag.toc{appendix}
\aistatstitle{Controlling Dependence in Implicit Generative Models via\\Spread Mutual Information: Supplementary Materials}

\begingroup
  \etocsettagdepth{main}{none}
  \etocsettagdepth{appendix}{subsection}
  \etocsetnexttocdepth{subsection}
  \tableofcontents
\endgroup

\section{Additional Theoretical Analysis}
\label{app:additional-theory}

We use $X_t=\alpha_tX_0+\sigma_t\epsilon$, where
$\epsilon\sim\mathcal N(0,I_d)$ is independent of $(X_0,Y)$ and
$\alpha_t,\sigma_t>0$ are fixed with respect to the generator parameters.
Write $\tau_t=\sigma_t^2/\alpha_t^2$ for the effective noise variance and
$v_t(x,y)=\nabla_x\log[p_\theta(x\mid y)/p_\theta(x)]$
for the exact score difference of the spread distributions.

\subsection{A comparison with standard MI}
\label{sec:smi-theory-comparison}
A direct comparison is possible in a binary Gaussian model, using
the posterior mean of a binary input to a Gaussian channel
\citep{guo_mutual_2005}.

\begin{proposition}[Relative gradient error in a binary Gaussian model]
\label{prop:smi-binary-comparison}
Let $Y$ be uniform on $\{-1,+1\}$ and let $Z,\epsilon$ be
independent standard Gaussian variables.
Consider $X_0=aY+bZ$, with parameter $\theta=a>0$ and fixed $b>0$.
Define $\tau_t=\sigma_t^2/\alpha_t^2$ and let
\begin{equation}
m_t=\mathbb{E}\bigl[(Y-\mathbb{E}[Y\mid X_t])^2\bigr]
=\mathbb{E}\bigl[\operatorname{Var}(Y\mid X_t)\bigr]
\end{equation}
measures the uncertainty about $Y$ remaining after
observing $X_t$.
The exact estimator in \Cref{eq:theory-single-gradient} satisfies
\begin{equation}
    \mathbb{E}G_t=\frac{a}{b^2+\tau_t}m_t,
    \qquad
    \frac{\operatorname{Var}(G_t)}{(\mathbb{E}G_t)^2}
    =\frac1{m_t}-1.
    \label{eq:theory-binary-relative-variance}
\end{equation}
Its relative variance is nonincreasing in $\tau_t$ and is no larger
than that of the standard-MI estimator, obtained at $\tau_t=0$.
The relative mean-square error of the average of $B$ independent
copies is $(1/m_t-1)/B$.
\end{proposition}
As the effective noise $\tau_t$ increases, predicting $Y$ from $X_t$
becomes harder, so the minimum mean-square prediction error $m_t$
increases. Proposition~4.2 shows that this increase reduces the
relative gradient variance $1/m_t-1$ and, for a minibatch of $B$
independent samples, the relative mean-square error $(1/m_t-1)/B$.
Thus, at a fixed minibatch size, spreading improves the relative
precision of gradient estimates in this model.
This improvement accompanies a change from standard MI to the
smoothed objective $I(X_t;Y)$, whose gradient is estimated without
bias when using exact scores. The explicit estimator and proof are
in \Cref{app:smi-binary-comparison-proof}.

\subsection{Information across noise levels}
\label{sec:smi-theory-dissipation}
\Cref{prop:smi-properties} establishes that injective spreading
preserves identifiability: zero SMI is equivalent to independence.
For Gaussian spreading, we can further characterize how information
changes with the noise level. Assume $\mathbb{E}\|X_0\|^2<\infty$,
fix the generator, and let $\alpha_t,\sigma_t>0$ be continuously
differentiable. The Gaussian-channel I--MMSE identity
\citep{guo_mutual_2005} yields
\begin{equation}
    \frac{\mathrm{d}}{\mathrm{d}t}I_\theta(X_t;Y)
    =-\frac{\alpha_t^2}{2}
        \frac{\mathrm{d}}{\mathrm{d}t}
        \left(\frac{\sigma_t^2}{\alpha_t^2}\right)
        \mathbb{E}\|v_t(X_t,Y)\|^2.
    \label{eq:theory-information-dissipation}
\end{equation}
Thus, SMI is nonincreasing as the effective noise variance
$\tau_t=\sigma_t^2/\alpha_t^2$ increases, and the same score
difference used in the generator update determines the rate of
information dissipation. As $\tau_t\to0$, SMI approaches
$I_\theta(X_0;Y)$, including when this limit is infinite;
as $\tau_t\to\infty$, SMI approaches zero.
The complete statement, proof, and integrated identity are in
\Cref{thm:smi-information-dissipation}.

This scale sensitivity provides a useful optimisation signal.
For $Y$ uniform on $\{-1,+1\}$ and $X_0=aY$, standard MI is
constant for $a>0$, whereas SMI decreases smoothly as $a\to0$
at any fixed positive noise level. SMI therefore provides a
gradient for reducing the separation between conditional outputs,
with the noise schedule determining the scale at which they
are distinguished.

See \cref{app:information-across-noise} for empirical results.


\section{Proofs}
\label{sec:smi-theory-proofs}
We first prove the method's two propositions and justify the classifier
construction. We then establish Gaussian regularity, prove the three
main theoretical results, and give additional results on classifier
error and information dissipation. Unless stated otherwise, the
Gaussian results use the standing assumptions of
\Cref{sec:smi-theoretical-analysis}: $\mathbb{E}\|X_0\|^2<\infty$,
$\alpha_t,\sigma_t>0$, and independent standard Gaussian noise.
Results involving parameter gradients additionally require the
regularity assumptions of \Cref{prop:smi-gradient}.
For brevity in the proofs, write
$J_\theta=\partial g_\theta(Y,Z)/\partial\theta$.

\subsection{Proof of the effective noise variance}\label{app:effective-noise-variance-proof}
\begin{proof}
Fix $t$ with $\alpha_t>0$ and $\sigma_t\geq 0$, and define
$\tau_t=\sigma_t^2/\alpha_t^2$.
Assume $Z\sim\mathcal{N}(0,I)$ is independent of $(X_0,Y)$.
Setting $\epsilon=Z$, we have
\[
U_t:=\frac{X_t}{\alpha_t}
=X_0+\frac{\sigma_t}{\alpha_t}Z
=X_0+\sqrt{\tau_t}\epsilon.
\]
Since $U_t$ is a deterministic function of $X_t$ and, conversely,
$X_t=\alpha_t U_t$, the data-processing inequality gives
\[
I_\theta(U_t;Y)
\leq I_\theta(X_t;Y)
\leq I_\theta(U_t;Y).
\]
Therefore,
\[
I_\theta(X_t;Y)
=I_\theta(X_0+\sqrt{\tau_t}\epsilon;Y).
\]
\end{proof}

\subsection{Proof of the data processing and identifiability}
\label{app:smi-properties-proof}
\begin{proof}[Proof of \Cref{prop:smi-properties}]
Nonnegativity follows from that of mutual information. Since the
channel acts only on $X$, the variables form the Markov chain
$Y\to X\to\tilde X$. The data-processing inequality gives
\begin{equation}
    0\le I_k(X;Y)=I(\tilde X;Y)\le I(X;Y),
\end{equation}
with the upper bound understood in the extended real sense.
If $X\perp Y$, this inequality implies $I_k(X;Y)=0$.

Conversely, suppose $I_k(X;Y)=0$. Then $\tilde X\perp Y$, so
\begin{equation}
    P_{X\mid Y=y}*k=P_X*k
    \qquad\text{for }p_Y\text{-almost every }y.
\end{equation}
Here the conditional distributions exist because the variables take
values in standard Borel spaces. Injectivity of the spreading
operation yields $P_{X\mid Y=y}=P_X$ for $p_Y$-almost every $y$,
which is equivalent to $X\perp Y$.
\end{proof}

\subsection{Proof of the pathwise SMI gradient}
\begin{proposition}[Gradient of Spread Mutual Information]
\label{prop:smi-gradient-formal}
Let $\theta\in\Theta\subseteq\mathbb{R}^m$, where $\Theta$
is open, and let $Y\sim p_Y$, $Z\sim q_Z$, and
$\epsilon\sim q_\epsilon$ be mutually independent with
$\theta$-independent distributions. Define
\begin{align}
    X=g_\theta(Y,Z),
    \qquad
    \tilde X=h(X,\epsilon),
\end{align}
where $h$ is independent of $\theta$ and
$h(x,\epsilon)\sim k(\cdot\mid x)$ for a fixed transition
kernel $k$.

Assume that:
\begin{enumerate}
    \item $g_\theta(y,z)$ is continuously differentiable in
    $\theta$ and $h(x,\epsilon)$ is continuously
    differentiable in $x$, for almost every realization
    of the base variables.

    \item The induced conditional and marginal densities
    $\tilde p_\theta(\tilde x\mid y)$ and
    $\tilde p_\theta(\tilde x)$ are positive and continuously
    differentiable in $(\theta,\tilde x)$ on a common open,
    $\theta$-independent support, for $p_Y$-almost every $y$.

    \item Writing
    \begin{align}
        \ell_\theta(\tilde x,y)
        =
        \log\frac{\tilde p_\theta(\tilde x\mid y)}
                  {\tilde p_\theta(\tilde x)},
    \end{align}
    the quantities
    \begin{align}
        \ell_\theta(\tilde X,Y),\qquad
        (\nabla_\theta\ell_\theta)(\tilde X,Y),\qquad
        \left(\frac{\partial\tilde X}{\partial\theta}\right)^{\!\top}
        \nabla_{\tilde x}\ell_\theta(\tilde X,Y)
    \end{align}
    are integrable, where $\nabla_\theta\ell_\theta$ holds
    $(\tilde x,y)$ fixed. Differentiation may be interchanged
    with the expectation defining SMI and with the integrals
    normalizing the conditional and marginal densities.
\end{enumerate}
Then
\begin{equation}
    \nabla_\theta I_\theta(\tilde X;Y)
    =
    \mathbb{E}\!\left[
        \left(\frac{\partial\tilde X}{\partial\theta}\right)^{\!\top}
        \left.
        \nabla_{\tilde x}
        \log\frac{\tilde p_\theta(\tilde x\mid Y)}
                  {\tilde p_\theta(\tilde x)}
        \right|_{\tilde x=\tilde X}
    \right],
\end{equation}
where the expectation is over
$(Y,Z,\epsilon)\sim p_Yq_Zq_\epsilon$.
\end{proposition}
\begin{proof}[Proof of \Cref{prop:smi-gradient}]\label{app:smi-gradient-proof}
Write
\begin{equation}
    \ell_\theta(u,y)
    =\log\frac{\tilde p_\theta(u\mid y)}{\tilde p_\theta(u)},
    \qquad
    \tilde X=h(g_\theta(Y,Z),\epsilon).
\end{equation}
The distribution of $(Y,Z,\epsilon)$ does not depend on $\theta$.
The assumed interchange of differentiation and expectation therefore
gives the chain rule
\begin{equation}
\begin{split}
    \nabla_\theta I_\theta(\tilde X;Y)
    =\mathbb{E}\!\left[
        \left(\frac{\partial\tilde X}{\partial\theta}\right)^{\!\top}
        \nabla_u\ell_\theta(\tilde X,Y)
    \right]
    +\mathbb{E}\!\left[
        \left.\partial_\theta\ell_\theta(u,Y)\right|_{u=\tilde X}
    \right].
    \label{eq:app-smi-gradient-chain-rule}
\end{split}
\end{equation}
The partial derivative in the second expectation holds the density
argument $u$ fixed. Its two terms vanish separately. Indeed, by the
common, parameter-independent support and normalization,
\begin{align}
    \mathbb{E}\!\left[
        \left.\partial_\theta\log\tilde p_\theta(u\mid Y)
        \right|_{u=\tilde X}\right]
    =\mathbb{E}_Y\!\left[
        \int\partial_\theta\tilde p_\theta(u\mid Y)\,\mathrm{d}u
    \right]=0,
\end{align}
and, using the marginal law of $\tilde X$,
\begin{equation}
    \mathbb{E}\!\left[
        \left.\partial_\theta\log\tilde p_\theta(u)
        \right|_{u=\tilde X}\right]
    =\int\partial_\theta\tilde p_\theta(u)\,\mathrm{d}u=0.
\end{equation}
Substituting these identities into
\Cref{eq:app-smi-gradient-chain-rule} proves
\Cref{eq:smi-gradient}. For Gaussian spreading,
$\partial X_t/\partial\theta=\alpha_tJ_\theta$; interchanging the
time expectation and differentiation gives
\Cref{eq:smi-weighted-gradient} under the method's stated assumptions.
\end{proof}

\subsection{Classifier optimum and product-pair sampling}
\label{app:smi-classifier-optimum}
Fix $\theta$ and a spreading kernel for which the conditional and
marginal densities are positive. Relative to the common measure
$\mathrm{d}u\,p_Y(\mathrm{d}y)$, the joint and product densities are
$a(u,y)=\tilde p_\theta(u\mid y)$ and $b(u,y)=\tilde p_\theta(u)$.
The balanced population cross-entropy is the integral of
\begin{equation}
    -a(u,y)\log c(u,y)-b(u,y)\log(1-c(u,y)).
\end{equation}
For each $(u,y)$ with $a,b>0$, this expression is strictly convex
in $c\in(0,1)$. Setting its derivative to zero gives
\begin{equation}
    -\frac ac+\frac b{1-c}=0,
    \qquad
    c^*_\theta(u,y)=\frac a{a+b}.
\end{equation}
Thus the optimal classifier is unique up to null sets and
$\log[c^*_\theta/(1-c^*_\theta)]=\log(a/b)$. This proves the
population identity used in \Cref{eq:smi-classifier-objective}.
Its smooth representative has the required input gradient.
The argument also applies at each time $t$; multiplying both class
losses by the same positive time weight preserves the unrestricted
pointwise optimum. It does not assert attainability within a finite
neural-network class.

If $\bar Y\sim p_Y$ is independent of $(\tilde X,Y)$, the pair
$(\tilde X,\bar Y)$ has the product law
$P_{\tilde X}\otimes P_Y$, whereas $(\tilde X,Y)$ has the joint
law. Reusing the same $\tilde X$ in these two pairs does not change
either marginal expectation in the classifier loss. At sampled
noise levels, the same argument holds conditionally on $T=t$.
For generator updates, the fitted classifier is held fixed and
its input gradient replaces the exact field in
\Cref{eq:smi-weighted-gradient}.

\subsection{Smooth densities and finite SMI}
\label{app:smi-smooth-finite}
\begin{theorem}[Smooth densities and finite SMI]
\label{thm:smi-smooth-finite}
For every $t$ with $\sigma_t>0$, the density $p_\theta(x_t)$ is
strictly positive and infinitely differentiable in $x_t$. The same
holds for $p_\theta(x_t\mid y)$ for $p_Y$-almost every $y$.
Moreover,
\begin{equation}
    0\le I_\theta(X_t;Y)
    \le\frac d2\log\!\left(
        1+\frac{\alpha_t^2
        \mathbb{E}\|X_0-\mathbb{E}X_0\|^2}{d\sigma_t^2}
    \right)<\infty.
    \label{eq:theory-finite-smi}
\end{equation}
If $\alpha_t\ne0$, Gaussian spreading is injective and
$I_\theta(X_t;Y)=0$ if and only if $X_0\perp Y$.
These conclusions do not require a density for $X_0$.
At $\alpha_t=0$, $X_t$ is pure noise and SMI is zero for every generator.
\end{theorem}

\begin{proof}
The marginal density is the mixture
\begin{equation}
    p_\theta(u)=\mathbb{E}\!\left[
        (2\pi\sigma_t^2)^{-d/2}
        \exp\!\left(-\frac{\|u-\alpha_tX_0\|^2}{2\sigma_t^2}\right)
    \right].
\end{equation}
It integrates to one by Tonelli's theorem and is strictly positive
for every $u$. Each spatial derivative of the Gaussian kernel is
a polynomial times that kernel and is bounded as a function of its
argument. Differentiation of every order can therefore be passed
through the expectation by dominated convergence. The same argument
applies to the regular conditional law of $X_0$ given $Y=y$.
This proves the density and smoothness claims.

Let $\mu=\mathbb{E}X_0$, $\Sigma=\operatorname{Cov}(X_0)$, and
$Q_t=\mathcal{N}(\alpha_t\mu,\sigma_t^2I_d+\alpha_t^2\Sigma)$.
The expected conditional Gaussian KL is finite and equals
\begin{equation}
    \mathbb{E}\!\left[
       D_{\mathrm{KL}}\bigl(
           \mathcal{N}(\alpha_tX_0,\sigma_t^2I_d)\,\|\,Q_t
       \bigr)\right]
    =\frac12\log\det\!\left(I_d+\frac{\alpha_t^2}{\sigma_t^2}\Sigma\right).
\end{equation}
The KL chain rule identifies the left-hand side with
$I(X_0;X_t)+D_{\mathrm{KL}}(P_{X_t}\|Q_t)$.
Nonnegativity, data processing for $Y\to X_0\to X_t$, and
concavity of the logarithm therefore give
\begin{align}
    I(X_t;Y)
    &\le I(X_t;X_0)\nonumber\\
    &\le\frac12\log\det\!\left(I_d+
                    \frac{\alpha_t^2}{\sigma_t^2}\Sigma\right)\nonumber\\
    &\le\frac d2\log\!\left(1+
                    \frac{\alpha_t^2\operatorname{tr}\Sigma}
                         {d\sigma_t^2}\right).
\end{align}
This proves \Cref{eq:theory-finite-smi} without requiring differential
entropy or a density for the input.

For any input law, its characteristic function after spreading is
\begin{equation}
    \varphi_{X_t}(u)=\varphi_{X_0}(\alpha_tu)
        \exp(-\sigma_t^2\|u\|^2/2).
\end{equation}
The Gaussian factor never vanishes. If $\alpha_t\ne0$, equality of
two output characteristic functions therefore implies equality of
the input characteristic functions, and hence of the input laws.
The independence equivalence follows from \Cref{prop:smi-properties}.
The pure-noise case is immediate from the channel definition.
\end{proof}

\subsection{Posterior identities and spatial regularity}
\label{app:smi-field-regularity}
The next identities provide a common basis for the spatial bounds,
gradient moment bounds, and information-dissipation formula.
They follow from differentiation of Gaussian mixtures; related
posterior-covariance identities appear in
\cite{wibisono2018convexitymutualinformationheat}.

\begin{lemma}[Posterior representations of the score difference]
\label{lem:smi-posterior-identities}
Define
\begin{align}
    \mu_t(x,y)&=\mathbb{E}[X_0\mid X_t=x,Y=y],
    &\bar\mu_t(x)&=\mathbb{E}[X_0\mid X_t=x],\\
    C_t(x,y)&=\operatorname{Cov}(X_0\mid X_t=x,Y=y),
    &\bar C_t(x)&=\operatorname{Cov}(X_0\mid X_t=x).
\end{align}
For every $x$ and $p_Y$-almost every $y$, using the Gaussian-mixture
versions of these conditional moments,
\begin{align}
    v_t(x,y)
    &=\frac{\alpha_t}{\sigma_t^2}
        \bigl(\mu_t(x,y)-\bar\mu_t(x)\bigr),
    \label{eq:app-smi-posterior-mean}\\
    \nabla_xv_t(x,y)
    &=\frac{\alpha_t^2}{\sigma_t^4}
        \bigl(C_t(x,y)-\bar C_t(x)\bigr).
    \label{eq:app-smi-posterior-covariance}
\end{align}
Furthermore,
\begin{equation}
    v_t(X_t,Y)=-\frac1{\sigma_t}
    \bigl(\mathbb{E}[\epsilon\mid X_t,Y]
          -\mathbb{E}[\epsilon\mid X_t]\bigr).
    \label{eq:app-smi-noise-projection}
\end{equation}
\end{lemma}

\begin{proof}
Differentiating the Gaussian-mixture density and dividing by its
positive value gives
\begin{equation}
    \nabla_x\log p_\theta(x\mid y)
    =-\frac{x-\alpha_t\mu_t(x,y)}{\sigma_t^2}.
\end{equation}
The corresponding marginal identity replaces $\mu_t$ by $\bar\mu_t$.
Their difference proves \Cref{eq:app-smi-posterior-mean}.
Differentiating the normalized posterior mean componentwise yields
\begin{equation}
    \partial_{x_j}(\mu_t(x,y))_i
    =\frac{\alpha_t}{\sigma_t^2}(C_t(x,y))_{ij}.
\end{equation}
The analogous marginal identity proves
\Cref{eq:app-smi-posterior-covariance}. These differentiations are
valid for the canonical Gaussian-mixture versions: the denominators
are positive, and on compact sets of $x$ the polynomial factors in
the numerators are dominated by Gaussian tails. In particular,
\begin{equation}
    \nabla_x^2\log p_\theta(x\mid y)
    =-\frac1{\sigma_t^2}I_d
       +\frac{\alpha_t^2}{\sigma_t^4}C_t(x,y),
    \label{eq:app-smi-log-hessian}
\end{equation}
with the corresponding marginal formula.
Finally, $\epsilon=(X_t-\alpha_tX_0)/\sigma_t$.
Taking conditional expectations and subtracting proves
\Cref{eq:app-smi-noise-projection}.
\end{proof}

\begin{theorem}[Regularity of the score difference]
\label{thm:smi-field-regularity}
Suppose additionally that $\|X_0\|\le R$ almost surely.
For every $x_t\in\mathbb{R}^d$ and $p_Y$-almost every $y$,
\begin{align}
    \|v_t(x_t,y)\|\le\frac{2\alpha_tR}{\sigma_t^2},
    \qquad\|\nabla_{x_t}v_t(x_t,y)\|_{\mathrm{op}}
        \le\frac{\alpha_t^2R^2}{\sigma_t^4}.
    \label{eq:theory-field-lipschitz}
\end{align}
If $\sigma_t^2>\alpha_t^2R^2$, the conditional and marginal spread
densities are strongly log-concave, with parameter at least
\begin{equation}
    \frac{\sigma_t^2-\alpha_t^2R^2}{\sigma_t^4}.
    \label{eq:theory-log-concavity}
\end{equation}
\end{theorem}

\begin{proof}
Every posterior distribution is supported in the ball of radius $R$.
Its mean has norm at most $R$, so
\Cref{eq:app-smi-posterior-mean} gives the magnitude bound.
For any unit vector $u$ and any such posterior law,
$\operatorname{Var}(u^\top X_0)\le\mathbb{E}(u^\top X_0)^2\le R^2$.
Thus $0\preceq C_t(x,y),\bar C_t(x)\preceq R^2I_d$ and
\begin{equation}
    -R^2I_d\preceq C_t(x,y)-\bar C_t(x)\preceq R^2I_d.
\end{equation}
The operator-norm bound follows from
\Cref{eq:app-smi-posterior-covariance}.
Finally, \Cref{eq:app-smi-log-hessian} gives
\begin{equation}
    -\nabla_x^2\log p_\theta(x\mid y)
    \succeq\frac{\sigma_t^2-\alpha_t^2R^2}{\sigma_t^4}I_d,
\end{equation}
and the same inequality holds for the marginal density.
This is the claimed strong log-concavity bound.
\end{proof}

Strong log-concavity gives a unique mode for each spread density.
These spatial statements require bounded outputs and do not imply
convexity of the training objective in $\theta$.

\subsection{Proof of the gradient moment bounds}
\label{app:smi-gradient-moments-proof}
\begin{proof}[Proof of \Cref{thm:smi-gradient-moments}]
Set $A_t=\mathbb{E}[\epsilon\mid X_t,Y]$ and
$B_t=\mathbb{E}[\epsilon\mid X_t]$. The tower property gives
$\mathbb{E}[A_t\mid X_t]=B_t$. Orthogonality of conditional
expectations and conditional Jensen's inequality imply
\begin{equation}
    \mathbb{E}\|A_t-B_t\|^2
    =\mathbb{E}\|A_t\|^2-\mathbb{E}\|B_t\|^2
    \le\mathbb{E}\|\epsilon\|^2=d.
\end{equation}
Equation~\eqref{eq:app-smi-noise-projection} therefore yields
$\mathbb{E}\|v_t(X_t,Y)\|^2\le d/\sigma_t^2$.
Since $\|J_\theta\|_{\mathrm{op}}\le M$ almost surely,
\begin{equation}
    \mathbb{E}\|G_t\|^2
    \le\alpha_t^2M^2\mathbb{E}\|v_t(X_t,Y)\|^2
    \le\frac{M^2d\alpha_t^2}{\sigma_t^2}.
\end{equation}
Unbiasedness follows from \Cref{prop:smi-gradient}.

For independent copies, the cross terms of the centered sum vanish.
Consequently, the fixed-noise minibatch bound is
\begin{equation}
\begin{split}
    \mathbb{E}\left\|
       \frac1B\sum_{i=1}^BG_{t,i}-\nabla_\theta I_\theta(X_t;Y)
    \right\|^2
    =\frac1B\mathbb{E}\|G_t-\mathbb{E}G_t\|^2
    \le\frac{M^2d\alpha_t^2}{B\sigma_t^2}.
    \label{eq:theory-gradient-mse}
\end{split}
\end{equation}
For random noise levels, let $H=w_{\mathrm{SMI}}(T)G_T$.
The assumed interchange of the time expectation and differentiation
gives $\mathbb{E}H=\nabla_\theta\mathcal{L}_{\mathrm{SMI}}$, and
\begin{equation}
    \mathbb{E}\|H\|^2
    \le M^2d\,\mathbb{E}_T\!\left[
        w_{\mathrm{SMI}}(T)^2\frac{\alpha_T^2}{\sigma_T^2}
    \right]<\infty.
\end{equation}
Applying the same centered-sum identity to independent copies of
$H$ proves \Cref{eq:theory-weighted-gradient-mse}.
\end{proof}

\subsection{Proof of the binary Gaussian comparison}
\label{app:smi-binary-comparison-proof}
\begin{proof}[Proof of \Cref{prop:smi-binary-comparison}]
Let $s_t^2=\alpha_t^2b^2+\sigma_t^2$. Conditional on $Y=y$,
$X_t\sim\mathcal{N}(\alpha_tay,s_t^2)$.
Bayes' rule with equal class probabilities gives
\begin{equation}
    u_t(x):=\mathbb{E}[Y\mid X_t=x]
    =\tanh\!\left(\frac{\alpha_tax}{s_t^2}\right).
\end{equation}
The conditional and marginal scores are
\begin{equation}
    \nabla_x\log p_\theta(x\mid y)
       =\frac{-x+\alpha_tay}{s_t^2},
    \qquad
    \nabla_x\log p_\theta(x)
       =\frac{-x+\alpha_ta u_t(x)}{s_t^2}.
\end{equation}
Since $\partial g_a(Y,Z)/\partial a=Y$, the pathwise estimator is
\begin{equation}
    G_t=\frac{\alpha_t^2a}{s_t^2}
        \left[1-Y\tanh\!\left(\frac{\alpha_taX_t}{s_t^2}\right)\right].
    \label{eq:app-smi-binary-estimator}
\end{equation}
The Gaussian mixture has smooth parameter-dependent densities for
$b>0$, and differentiation under the expectation is justified locally
in $a$ by Gaussian-tail bounds. Thus the pathwise identity applies
also at the unspread endpoint $\alpha_0=1$, $\sigma_0=0$.

Write $U=u_t(X_t)$ and $A=1-YU$. The tower property implies
$\mathbb{E}[YU]=\mathbb{E}[U^2]$. Since $Y^2=1$,
\begin{align}
    m_t&=\mathbb{E}(Y-U)^2=1-\mathbb{E}U^2,\\
    \mathbb{E}A&=m_t,
    \qquad
    \mathbb{E}A^2=1-2\mathbb{E}[YU]+\mathbb{E}U^2=m_t.
\end{align}
Both class-conditional Gaussian densities are positive, so $m_t>0$.
With $\kappa_t=\alpha_t^2a/s_t^2=a/(b^2+\tau_t)$,
\begin{equation}
    \mathbb{E}G_t=\kappa_tm_t,
    \qquad
    \operatorname{Var}(G_t)=\kappa_t^2m_t(1-m_t).
\end{equation}
These equalities prove \Cref{eq:theory-binary-relative-variance}.

To prove monotonicity, divide the observation by $\alpha_t$.
Its law is that of $aY+\sqrt{b^2+\tau_t}\,N$, with $N$ an
independent standard Gaussian. For $\tau_2>\tau_1$, choose a
coupling $W_2=W_1+\sqrt{\tau_2-\tau_1}\,N'$, where
$W_1=aY+\sqrt{b^2+\tau_1}\,N$ and $N'$ is independent.
Then $Y\to W_1\to W_2$ and
\begin{equation}
    \mathbb{E}\bigl[\mathbb{E}[Y\mid W_2]^2\bigr]
    \le\mathbb{E}\bigl[\mathbb{E}[Y\mid W_1]^2\bigr]
\end{equation}
by the tower property and conditional Jensen's inequality.
Hence $m_t$ is nondecreasing in $\tau_t$, and $1/m_t-1$ is
nonincreasing. Comparing with $\tau=0$ gives the standard-MI
comparison. Independent averaging divides the variance by $B$,
proving the relative MSE formula. Finally,
$0<\mathbb{E}G_t\le a/(b^2+\tau_t)\to0$ as $\tau_t\to\infty$.
\end{proof}

\subsection{Proofs for classifier estimation error}
\label{app:smi-classifier-error-proof}
\begin{proof}[Proof of \Cref{prop:smi-classifier-error}]
Condition on the fitted classifier and current generator, and let
$\Delta_t=\widehat v_t(X_t,Y)-v_t(X_t,Y)$.
The update error is $\widehat G_t-G_t=\alpha_tJ_\theta^\top\Delta_t$.
Consequently,
\begin{align}
    \|\mathbb{E}\widehat G_t-\nabla_\theta I_\theta(X_t;Y)\|
    \le\alpha_tM\mathbb{E}\|\Delta_t\| 
    \le\alpha_tM\delta_t.
\end{align}
Using $\|u+v\|^2\le2\|u\|^2+2\|v\|^2$ and
\Cref{eq:theory-score-moment},
\begin{equation}
    \mathbb{E}\|\widehat G_t\|^2
    \le2M^2\alpha_t^2
        \left(\frac d{\sigma_t^2}+\delta_t^2\right).
\end{equation}
For fresh independent copies, the bias--variance decomposition is
\begin{align}
    \mathbb{E}\left\|
        \frac1B\sum_{i=1}^B\widehat G_{t,i}
        -\nabla_\theta I_\theta(X_t;Y)\right\|^2
    =\frac1B\mathbb{E}\|
        \widehat G_t-\mathbb{E}\widehat G_t\|^2
       +\|\mathbb{E}\widehat G_t-\nabla_\theta I_\theta(X_t;Y)\|^2.
\end{align}
Bounding the variance by the second moment and substituting the
bias bound proves \Cref{eq:theory-estimated-gradient-mse}.
\end{proof}

\begin{corollary}[Classifier error for the weighted objective]
\label{cor:smi-weighted-classifier-error}
Under the assumptions of \Cref{prop:smi-classifier-error} at
$p_T$-almost every noise level, let $w=w_{\mathrm{SMI}}$ and assume
\begin{equation}
    \mathbb{E}_T\!\left[
        w(T)^2\alpha_T^2
        \left(\frac d{\sigma_T^2}+\delta_T^2\right)
    \right]<\infty.
\end{equation}
Suppose the weighted objective is finite and differentiation can
be interchanged with the time expectation. For $B$ fresh,
independent draws of $(T,Y,Z,\epsilon)$, define
\begin{equation}
    \widehat G_B=\frac1B\sum_{i=1}^Bw(T_i)\widehat G_{T_i,i}.
\end{equation}
Conditionally on the fitted state,
\begin{align}
    \|\mathbb{E}\widehat G_B-\nabla_\theta\mathcal{L}_{\mathrm{SMI}}\|
    &\le M\mathbb{E}_T[w(T)\alpha_T\delta_T],
    \label{eq:app-smi-weighted-bias}\\
    \mathbb{E}\|\widehat G_B-\nabla_\theta\mathcal{L}_{\mathrm{SMI}}\|^2
    &\le M^2\bigl(\mathbb{E}_T[w(T)\alpha_T\delta_T]\bigr)^2 +\frac{2M^2}{B}\mathbb{E}_T\!\left[
        w(T)^2\alpha_T^2
        \left(\frac d{\sigma_T^2}+\delta_T^2\right)\right].
    \label{eq:app-smi-weighted-classifier-mse}
\end{align}
\end{corollary}

\begin{proof}
Let $\widehat H=w(T)\widehat G_T$ and $H=w(T)G_T$.
The assumptions imply finite second moments and
$\mathbb{E}H=\nabla_\theta\mathcal{L}_{\mathrm{SMI}}$.
Conditioning on $T$ and applying the fixed-level bias bound gives
\begin{equation}
    \|\mathbb{E}(\widehat H-H)\|
    \le M\mathbb{E}_T[w(T)\alpha_T\delta_T].
\end{equation}
The latter expectation is finite by Cauchy--Schwarz and the stated
second-moment condition. Likewise,
\begin{equation}
    \mathbb{E}\|\widehat H\|^2
    \le2M^2\mathbb{E}_T\!\left[
        w(T)^2\alpha_T^2
        \left(\frac d{\sigma_T^2}+\delta_T^2\right)\right].
\end{equation}
The bias--variance decomposition for the mean of $B$ independent
copies of $\widehat H$ proves the result.
\end{proof}

\begin{proposition}[Classification risk does not control derivative error]
\label{prop:smi-risk-counterexample}
There exist smooth classifiers whose excess population
cross-entropy risk tends to zero while their mean-square error in
the input gradient of the log odds diverges, even when the joint
and product distributions coincide and have smooth positive
densities in the observation variable.
\end{proposition}

\begin{proof}
Let $X_t\sim\mathcal{N}(0,1)$ be independent of $Y$. This is
compatible with Gaussian spreading, for example by taking
$X_0\sim\mathcal{N}(0,1/2)$ independently of $Y$,
$\alpha_t=1$, and $\sigma_t^2=1/2$.
The optimal classifier is $c^*=1/2$ and the exact field is zero.
For positive integers $n$, define
\begin{equation}
    c_n(x,y)=\left(1+
       \exp\!\left[-n^{-1}\sin(n^2x)\right]\right)^{-1}.
\end{equation}
For the sum of the two class expectations used in the method,
the excess risk is
\begin{equation}
    2\mathbb{E}\log\cosh\!\left(
         \frac{\sin(n^2X_t)}{2n}\right)
    \le\frac1{4n^2}\longrightarrow0,
\end{equation}
where $\log\cosh u\le u^2/2$.
In contrast,
\begin{equation}
    \partial_x\log\frac{c_n(x,y)}{1-c_n(x,y)}
    =n\cos(n^2x).
\end{equation}
The Gaussian characteristic function therefore gives
\begin{equation}
    \mathbb{E}\left|
       \partial_x\log\frac{c_n(X_t,Y)}{1-c_n(X_t,Y)}
    \right|^2
    =\frac{n^2}{2}\left(1+e^{-2n^4}\right)\longrightarrow\infty.
\end{equation}
\end{proof}

\subsection{Information dissipation and limiting behavior}
\label{app:smi-information-dissipation}
The following result gives the full statement behind
\Cref{eq:theory-information-dissipation}. The proof applies the
Gaussian-channel I--MMSE identity
\citep{guo_mutual_2005} to both conditional
and marginal channels. Integrated score-difference representations
are also used for MI estimation in
\cite{franzese2024mindemutualinformationneural}.

\begin{theorem}[Information dissipation and limiting behavior]
\label{thm:smi-information-dissipation}
Fix $\theta$ and suppose $\alpha_t,\sigma_t$ are positive and
continuously differentiable on an interval. Then
\Cref{eq:theory-information-dissipation} holds. In particular,
$I_\theta(X_t;Y)$ is nonincreasing whenever
$\tau_t=\sigma_t^2/\alpha_t^2$ is nondecreasing.
If $X_0$ and $Y$ are dependent, the derivative is strictly negative
wherever $\tau_t'>0$.

For any positive-noise levels along which the corresponding ratio
has the stated limit,
\begin{align}
    I_\theta(X_t;Y)&\longrightarrow I_\theta(X_0;Y)
       &&\text{as }\tau_t\longrightarrow0,
    \label{eq:theory-zero-noise-limit}\\
    I_\theta(X_t;Y)&\longrightarrow0
       &&\text{as }\tau_t\longrightarrow\infty.
    \label{eq:theory-high-noise-limit}
\end{align}
The first limit holds in the extended real sense, including when
$I_\theta(X_0;Y)=\infty$. If $I_\theta(X_0;Y)<\infty$, then
\begin{equation}
    I_\theta(X_0;Y)-I_\theta(X_t;Y)
    =I_\theta(X_0;Y\mid X_t).
    \label{eq:theory-information-gap}
\end{equation}
For $t_1<t_2$ in the interval,
\begin{equation}
    I_\theta(X_{t_1};Y)-I_\theta(X_{t_2};Y)
    =\frac12\int_{t_1}^{t_2}\alpha_t^2
        \frac{\mathrm{d}}{\mathrm{d}t}
        \left(\frac{\sigma_t^2}{\alpha_t^2}\right)\cdot\mathbb{E}\|v_t(X_t,Y)\|^2\,\mathrm{d}t.
    \label{eq:theory-integrated-dissipation}
\end{equation}
The two limits and the information-gap identity do not require
differentiability of the schedule.
\end{theorem}

\begin{proof}
We suppress $\theta$, which is fixed throughout.
For $\gamma>0$, set $V_\gamma=\sqrt\gamma X_0+\epsilon$ and define
\begin{align}
    M(\gamma)
       &=\mathbb{E}\|X_0-\mathbb{E}[X_0\mid V_\gamma]\|^2,\\
    M_Y(\gamma)
       &=\mathbb{E}\|X_0-\mathbb{E}[X_0\mid V_\gamma,Y]\|^2.
\end{align}
The Gaussian-channel information quantities are finite by the
finite second moment. The Markov chain $Y\to X_0\to V_\gamma$
and the chain rule give
\begin{equation}
    I(V_\gamma;Y)
    =I(V_\gamma;X_0)-I(V_\gamma;X_0\mid Y).
\end{equation}
The vector I--MMSE identity applies to the first term and, for
$p_Y$-almost every $y$, to the channel conditional on $Y=y$.
The conditional derivative is bounded by
$\mathbb{E}[\|X_0\|^2\mid Y=y]/2$, an integrable function of $y$.
The same bound controls difference quotients, so differentiation
may be passed through the expectation over $Y$. Thus
\begin{equation}
    \frac{\mathrm{d}}{\mathrm{d}\gamma}I(V_\gamma;Y)
    =\frac12\bigl(M(\gamma)-M_Y(\gamma)\bigr).
    \label{eq:app-smi-conditional-immse}
\end{equation}
Orthogonality of nested conditional expectations yields
\begin{equation}
    M(\gamma)-M_Y(\gamma)
    =\mathbb{E}\left\|
       \mathbb{E}[X_0\mid V_\gamma,Y]
       -\mathbb{E}[X_0\mid V_\gamma]\right\|^2.
\end{equation}
At $\gamma_t=\alpha_t^2/\sigma_t^2$, we have
$X_t=\sigma_tV_{\gamma_t}$. Multiplication by the positive constant
$\sigma_t$ preserves the conditioning information and MI. By
\Cref{eq:app-smi-posterior-mean},
\begin{equation}
    M(\gamma_t)-M_Y(\gamma_t)
    =\frac{\sigma_t^4}{\alpha_t^2}
       \mathbb{E}\|v_t(X_t,Y)\|^2.
\end{equation}
Since $\gamma_t=1/\tau_t$,
\begin{equation}
    \gamma_t'\frac{\sigma_t^4}{2\alpha_t^2}
    =-\frac{\alpha_t^2}{2}\tau_t'.
\end{equation}
Combining this equality with \Cref{eq:app-smi-conditional-immse}
proves \Cref{eq:theory-information-dissipation} and its
monotonicity conclusion. The MMSE bound also makes the
information functions locally Lipschitz in $\gamma$. Their
composition with a continuously differentiable schedule is
absolutely continuous on compact subintervals. Integrating the
derivative proves \Cref{eq:theory-integrated-dissipation}.

To prove strictness, suppose $\mathbb{E}\|v_t(X_t,Y)\|^2=0$ at a
positive-noise level. Then for $p_Y$-almost every $y$,
$v_t(x,y)=0$ for $p_\theta(x\mid y)\,\mathrm{d}x$-almost every $x$.
The conditional density is everywhere positive and $v_t$ is
continuous in $x$, so $v_t(x,y)=0$ for every $x$ for those $y$.
The log density ratio therefore has zero gradient on the connected
space $\mathbb{R}^d$ and is constant in $x$.
Normalization implies $p_\theta(x\mid y)=p_\theta(x)$.
Thus $X_t\perp Y$, and Gaussian injectivity gives $X_0\perp Y$.
Taking the contrapositive proves strict negativity when
$X_0,Y$ are dependent and $\tau_t'>0$.

For the limits, let $U_\tau=X_0+\sqrt\tau\epsilon$.
For every finite positive $\tau_t$,
$I(X_t;Y)=I(U_{\tau_t};Y)$.
The bound in \Cref{eq:theory-finite-smi} becomes
\begin{equation}
    I(U_\tau;Y)
    \le\frac d2\log\!\left(
      1+\frac{\mathbb{E}\|X_0-\mathbb{E}X_0\|^2}{d\tau}
    \right),
\end{equation}
which tends to zero as $\tau\to\infty$.

To handle $\tau\to0$, equip the standard Borel space of $Y$ with
a compatible Polish topology. Under the coupling that uses the
same $(X_0,Y,\epsilon)$ for every $\tau$, $(U_\tau,Y)$ converges
almost surely, and hence weakly, to $(X_0,Y)$.
The product laws $P_{U_\tau}\otimes P_Y$ also converge weakly to
$P_{X_0}\otimes P_Y$.
Joint lower semicontinuity of relative entropy gives
\begin{equation}
    I(X_0;Y)\le\liminf_{\tau\to0}I(U_\tau;Y).
\end{equation}
For completeness, this lower semicontinuity follows from the
variational formula
\begin{equation}
    D_{\mathrm{KL}}(P\|Q)
    =\sup_{\psi\in C_b}
       \left\{\int\psi\,\mathrm{d}P
             -\log\int e^\psi\,\mathrm{d}Q\right\}
\end{equation}
on Polish spaces, as each expression inside the supremum is
continuous under weak convergence of both measures.
Data processing gives $I(U_\tau;Y)\le I(X_0;Y)$.
Together these inequalities establish convergence when the limit
is finite; when it is infinite, the lower-limit inequality alone
forces divergence to infinity.

Finally, the Markov property gives
\begin{align}
    I(Y;(X_0,X_t))
       &=I(Y;X_0)+I(Y;X_t\mid X_0)=I(Y;X_0),\\
    I(Y;(X_0,X_t))
       &=I(Y;X_t)+I(Y;X_0\mid X_t).
\end{align}
When $I(Y;X_0)<\infty$, subtraction is legitimate and proves
\Cref{eq:theory-information-gap}.
\end{proof}

\subsection{Independence and the scale dependence of weighted SMI}
\label{app:smi-weighted-independence}
\begin{proposition}[Zero weighted SMI and its quantitative limitation]
\label{prop:smi-weighted-independence}
Let $w_{\mathrm{SMI}}\ge0$, let $\sigma_t>0$ at $p_T$-almost every
level, and suppose
\begin{equation}
    p_T\bigl(\{t:w_{\mathrm{SMI}}(t)>0,\ \alpha_t\ne0\}\bigr)>0.
\end{equation}
For the weighted objective in \Cref{eq:smi-objective},
\begin{equation}
    \mathcal{L}_{\mathrm{SMI}}(\theta)=0
    \quad\Longleftrightarrow\quad X_0\perp Y.
\end{equation}
However, small SMI does not uniformly imply small standard MI.
For $Y$ uniform on $\{-1,+1\}$ and $X_0=aY$ with $a>0$,
$I(X_0;Y)=\log2$, while at every fixed positive-noise level,
$I(X_t;Y)\to0$ as $a\to0$.
If additionally $\mathbb{E}_T[w_{\mathrm{SMI}}(T)]<\infty$,
the weighted objective also tends to zero for this family.
\end{proposition}

\begin{proof}
If $X_0\perp Y$, every spread variable is independent of $Y$, so
the objective is zero. Conversely, nonnegativity implies that
a zero objective has $w_{\mathrm{SMI}}(t)I(X_t;Y)=0$ for
$p_T$-almost every $t$. On the positive-measure set of levels with
positive weight and $\alpha_t\ne0$, this gives $I(X_t;Y)=0$.
The independence equivalence of \Cref{thm:smi-smooth-finite} at any
such level yields $X_0\perp Y$.

For the binary family, $Y$ is determined by $X_0=aY$ whenever
$a>0$, so $I(X_0;Y)=H(Y)=\log2$. The finite-SMI bound gives
\begin{equation}
    I(X_t;Y)\le\frac12\log\!\left(
        1+\frac{\alpha_t^2a^2}{\sigma_t^2}\right)\longrightarrow0.
\end{equation}
For the weighted claim, use the additional uniform bound
$I(X_t;Y)\le H(Y)=\log2$. The integrable function
$w_{\mathrm{SMI}}(t)\log2$ then permits dominated convergence.
\end{proof}

\section{Experimental Details}
\label{app:experiments}

Unless stated otherwise, we report means and standard deviations (SDs) over independent training runs.
Within each task, matched MI methods share the same generator or representation architecture and update budget, differing only in their estimator heads and gradient construction.
Estimator updates use detached model outputs; model updates freeze estimator parameters while retaining gradients through their inputs.
MIGE recomputes spectral scores at each update.

\paragraph{Implementation sources.}
We adapt CLUB from \url{https://github.com/Linear95/CLUB}; NWJ, MINE, InfoNCE and SMILE from \url{https://github.com/ermongroup/smile-mi-estimator}; FLO from \url{https://github.com/qingguo666/FLO}; KNIFE from \url{https://github.com/g-pichler/knife}; $f$-DIME from \url{https://github.com/nicolaNovello/fDIME}; and MIGE from \url{https://github.com/zhouyiji/MIGE}.
GAIN, MIWAE, HyperImpute and MissForest use HyperImpute 0.1.17, \url{https://github.com/vanderschaarlab/hyperimpute}; MIRI and WGF are adapted from \url{https://github.com/yujhml/MIRI-Imputation} and \url{https://github.com/anewgithubname/gradest2}, respectively.

For Waterbirds, the GroupDRO and DFR adaptations follow \url{https://github.com/kohpangwei/group_DRO} and \url{https://github.com/PolinaKirichenko/deep_feature_reweighting}.
The image-translation model and latent $\ell_1$ baseline are adapted from BicycleGAN, \url{https://github.com/junyanz/BicycleGAN}, and the diversity regulariser from DSGAN, \url{https://github.com/maga33/DSGAN}.
DSM, ERM, HSIC and class-conditional adversarial training are implemented locally from their stated objectives.
Task-specific architectures and training schedules are specified below.

\paragraph{Hardware.}
Experiments use NVIDIA GeForce RTX 2080 Ti, 3080 Ti, 3090, 4090 and 5080 GPUs.
All computational-cost measurements use an isolated RTX 3090 exclusively (\cref{app:synthetic-cost}).

\paragraph{Data access and use.}
\label{app:data-assets}
Synthetic data are generated as specified below; external dataset links and licences or usage terms are provided in \cref{tab:imputation-datasets,app:nuisance,app:edges-additional}.
We reuse existing benchmark releases and collect no new participant data.
Blood and Maternal contain health-related attributes, and the housing datasets retain geographic coordinates; we use these data solely for methodological evaluation.
The linked source records do not establish participant-level consent for all reused datasets.

\subsection{Synthetic Experiments}
\label{app:synthetic-experiments}

\subsubsection{Support Geometry}
\label{app:synthetic-support}

\paragraph{Data and objective.}
Let $X_a=(Z,aY)$, where $Z\sim\mathcal N(0,1)$ and
$Y\sim\operatorname{Unif}\{-1,+1\}$ are independent.
The conditional supports are parallel lines separated by $2a$.
Independence occurs at $a=0$, whereas $I(X_a;Y)=\log 2$ for every $a>0$.
We minimise dependence by optimising $a$, comparing SMI with an otherwise matched no-spreading ablation.

\paragraph{Methods and training.}
Both methods fit joint-versus-product log density ratios using classifiers
with three width-128 SiLU hidden layers.
We use Adam with learning rate $10^{-3}$, batch size $4{,}096$, and five
auxiliary updates per parameter update.
SMI samples $\sigma\sim\operatorname{LogUniform}(0.05,3)$ and uses the
stop-gradient surrogate, while the ablation leaves $X_a$ unperturbed.
Across ten runs, we initialise $a=3$ and perform $2{,}000$ projected
gradient updates on $[0,5]$ with learning rate $5\times10^{-2}$.

\paragraph{Evaluation.}
We use $2a<0.1$ as the success criterion, measuring support separation.
\Cref{fig:synthetic-summary}a shows the first $1{,}500$ updates.

\subsubsection{Gradient Accuracy}
\label{app:synthetic-gradient}

\paragraph{Data and objective.}
Following \citet{chen_diffratio_2025}, the unperturbed distribution has
40 support points $\mu_m$ on a jittered $5\times8$ grid in $[-4,4]^2$.
For $Y\sim\operatorname{Unif}\{-1,+1\}$, the conditional probability
assigned to support point $\mu_m$ is proportional to $\exp(2Yb_m)$,
where $b_m$ standardises
$\sin(2\mu_{m,1})+\cos(2\mu_{m,2})$.
After Gaussian spreading, the analytic score difference is
\begin{equation}
\nabla_x\log\frac{p_\sigma(x\mid y)}{p_\sigma(x)}
=
\sigma^{-2}\sum_m
\big[\gamma_{m\mid y}(x)-\gamma_m(x)\big]\mu_m,
\end{equation}
where $\gamma_{m\mid y}$ and $\gamma_m$ denote the conditional and
marginal posterior probabilities associated with support point $\mu_m$.
Parameter gradients are obtained by averaging the corresponding
contributions with respect to the support-point locations.

\paragraph{Methods and training.}
SMI differentiates a single scale-conditioned log density ratio.
The comparator subtracts conditional and marginal scores fitted separately
by denoising score matching (DSM) \citep{vincent_connection_2011}, using
direct score outputs and a $\sigma^2$-weighted squared-error loss.
All networks have four width-256 SiLU hidden layers.
We use Adam with learning rate $10^{-3}$, batch size $2{,}048$, and
$20{,}000$ fitting updates, with
$\sigma\sim\operatorname{LogUniform}(0.05,2)$.
Across ten runs, each estimator is trained on independently sampled batches.

\paragraph{Evaluation.}
Each run evaluates both estimators on the same $16{,}384$ held-out noisy
samples.
Relative score errors use pooled root-mean-square normalisation, while
parameter-gradient errors use Euclidean normalisation against Monte Carlo
averages of the analytic field.
\Cref{tab:synthetic-gradient} reports errors pooled over noise levels.

\subsubsection{MI Optimisation}
\label{app:synthetic-control}

\paragraph{Data and objective.}
For independent $U,V\sim\mathcal N(0,I_d)$, define
\begin{equation}
\begin{aligned}
X=Q_X\!\left(\rho\odot U+\sqrt{1-\rho^2}\odot V\right),
\quad Y=Q_YU,
\quad I(X;Y)=-\frac12\sum_{j=1}^{r}\log(1-\rho_j^2).
\end{aligned}
\end{equation}
The orthogonal matrices $Q_X,Q_Y$ are fixed within each run.
The first $r$ latent coordinate pairs have trainable correlations, while the remainder have zero correlation.
We optimise either a shared correlation parameter or one parameter per coordinate pair.
The coordinatewise parameterisation is initialised with unequal MI contributions and alternating correlation signs.
We use $r=d\in\{20,100,200\}$.

We either leave $Y$ Gaussian or apply the coordinatewise transforms
$y^3$, $\operatorname{asinh}(y)$, or
$\operatorname{sign}(y)|y|^{3/2}$.
For the spiral transformation, the first two coordinates are rotated by angle $\|y\|^2/d$.
All transformations are invertible and therefore preserve MI.
Transformed $Y$ is divided by its root-mean-square magnitude, estimated from $16{,}384$ calibration samples and fixed during training.

\paragraph{MI-based methods.}
SMI spreads only $X$, sampling
$\sigma\sim\operatorname{LogUniform}(0.01,1)$ with
$\alpha_t=w_{\mathrm{SMI}}(t)=1$.
Binary cross-entropy fits the scale-conditioned log density ratio, and the generator uses its stopped input gradient.
Joint and product examples share simulator residuals and noise draws in this experiment.
All nine baselines use unperturbed inputs.
CLUB \citep{cheng_club_2020} uses a sampled diagonal-Gaussian conditional contrast;
NWJ \citep{nguyen_estimating_2010}, MINE \citep{belghazi_mutual_2018},
InfoNCE \citep{oord_representation_2018}, and SMILE
\citep{song_understanding_2020} use all-pairs scores.
MINE uses moving-normaliser momentum $0.9$, while SMILE uses its Jensen--Shannon backward rule without clipping.
FLO \citep{guo_tight_2022} learns a temperature initialised at one.
KNIFE \citep{pichler_differential_2022} uses eight-component full-covariance marginal and conditional mixtures.
The KL variant of $f$-DIME \citep{letizia_mutual_2024} uses $\alpha=1$.
MIGE \citep{wen_mutual_2020} uses joint and marginal spectral scores, median-distance bandwidth, eigen-energy threshold $0.99$, and no Gram regularisation.
\Cref{app:mi-complexity} summarises the estimator costs of these implementations.

\paragraph{Training.}
We use ten runs with fixed hyperparameters across data families and dimensions.
Neural estimators have two width-256 SiLU hidden layers with method-specific heads.
Adam uses default moments, batch size $256$, generator learning rate $5\times10^{-3}$, and estimator learning rate $10^{-3}$, with five auxiliary updates per generator update and no initial estimator fitting.
Estimator gradients are clipped to norm $10$; generator gradients are unclipped.
Training lasts at most $5{,}000$ generator updates.
Correlations use a hyperbolic-tangent parameterisation and are clipped to $[-0.9995,0.9995]$.

\paragraph{Evaluation.}
We evaluate analytic MI before spreading.
Minimisation starts at $10$ nats and uses the full budget; maximisation starts at $0.5$ nats and stops upon first reaching $50$ nats or exhausting the budget.
Both terminate on numerical failure (NaN or infinity).
For minimisation, numerical summaries average the final $1{,}000$ updates, while attainment requires MI to remain below the stated tolerance throughout this window.
For maximisation, summaries report MI at stopping.
Finite runs that miss a target are marked unmet, and failures remain in all outcome denominators.
The main figure equally weights five data families, three dimensions, two parameterisations, and ten runs.

\begin{table*}[!t]
\centering\small
\setlength{\tabcolsep}{1.5pt}
\medmuskip=0mu
\caption{\textbf{Gaussian MI optimisation.} Shared correlation ($r=d$); mean $\pm$ SD in nats over ten runs. n.f.: at least one numerical failure.}
\label{tab:synthetic-control}
\begin{tabular}{@{}lcccccc@{}}
\toprule
& \multicolumn{3}{c}{Minimisation $\downarrow$} & \multicolumn{3}{c}{Maximisation (target: 50)} \\
\cmidrule(lr){2-4}\cmidrule(lr){5-7}
Method & $d=20$ & $d=100$ & $d=200$ & $d=20$ & $d=100$ & $d=200$ \\
\midrule
\rowcolor{smirow}
\textbf{SMI (Ours)} & $\mathbf{(1.36\pm0.366)\times10^{-8}}$ & $\mathbf{(7.83\pm11.6)\times10^{-8}}$ & $\mathbf{(1.00\pm2.74)\times10^{-6}}$ & $50.007\pm0.006$ & $50.205\pm0.111$ & $50.230\pm0.197$ \\
CLUB & $\underline{(1.12\pm0.373)\times10^{-5}}$ & $(1.20\pm0.362)\times10^{-4}$ & $0.973\pm0.025$ & $50.006\pm0.003$ & $50.268\pm0.057$ & $50.095\pm0.058$ \\
NWJ & $(2.44\pm1.56)\times10^{-5}$ & $(6.59\pm4.85)\times10^{-5}$ & $(4.84\pm3.08)\times10^{-4}$ & n.f. & n.f. & $50.430\pm0.316$ \\
MINE & n.f. & $(6.55\pm7.41)\times10^{-5}$ & $(1.00\pm0.827)\times10^{-4}$ & n.f. & n.f. & n.f. \\
InfoNCE & $(1.90\pm0.991)\times10^{-5}$ & $\underline{(6.26\pm5.40)\times10^{-5}}$ & $\underline{(6.90\pm4.34)\times10^{-5}}$ & $21.391\pm0.246$ & $49.526\pm0.556$ & $50.009\pm0.007$ \\
SMILE & $(9.20\pm9.98)\times10^{-3}$ & $(9.76\pm7.50)\times10^{-5}$ & $(2.72\pm2.82)\times10^{-4}$ & $24.540\pm0.303$ & $50.032\pm0.036$ & $50.029\pm0.021$ \\
FLO & $(6.12\pm5.68)\times10^{-5}$ & $(1.79\pm1.07)\times10^{-4}$ & $(1.05\pm1.36)\times10^{-3}$ & n.f. & $50.290\pm0.152$ & $50.316\pm0.335$ \\
KNIFE & $(1.28\pm0.336)\times10^{-5}$ & $0.100\pm0.071$ & $0.027\pm0.018$ & $50.052\pm0.025$ & $50.304\pm0.234$ & $50.850\pm0.312$ \\
$f$-DIME & $(5.82\pm1.74)\times10^{-5}$ & $(2.40\pm1.38)\times10^{-4}$ & $(1.01\pm0.765)\times10^{-3}$ & n.f. & n.f. & $50.298\pm0.143$ \\
MIGE & $(2.66\pm0.385)\times10^{-4}$ & $(1.07\pm0.117)\times10^{-3}$ & $(4.05\pm0.661)\times10^{-3}$ & $50.005\pm0.004$ & $50.242\pm0.143$ & $50.349\pm0.210$ \\
\bottomrule
\end{tabular}
\end{table*}

\paragraph{Additional results.}
\Cref{fig:synthetic-gaussian,fig:synthetic-cubic,fig:synthetic-asinh,fig:synthetic-signed-power,fig:synthetic-spiral} report results by data family and dimension.

\begin{figure*}[!t]
\centering
\begin{subfigure}[t]{0.49\textwidth}
\includegraphics[width=\linewidth]{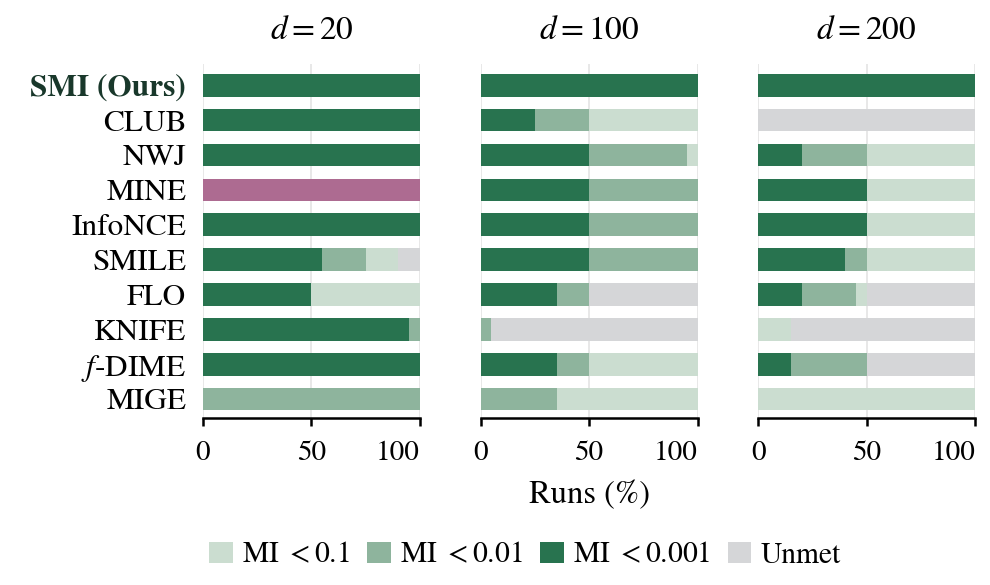}
\caption{MI minimisation.}\label{fig:synthetic-gaussian-min}
\end{subfigure}\hfill
\begin{subfigure}[t]{0.49\textwidth}
\includegraphics[width=\linewidth]{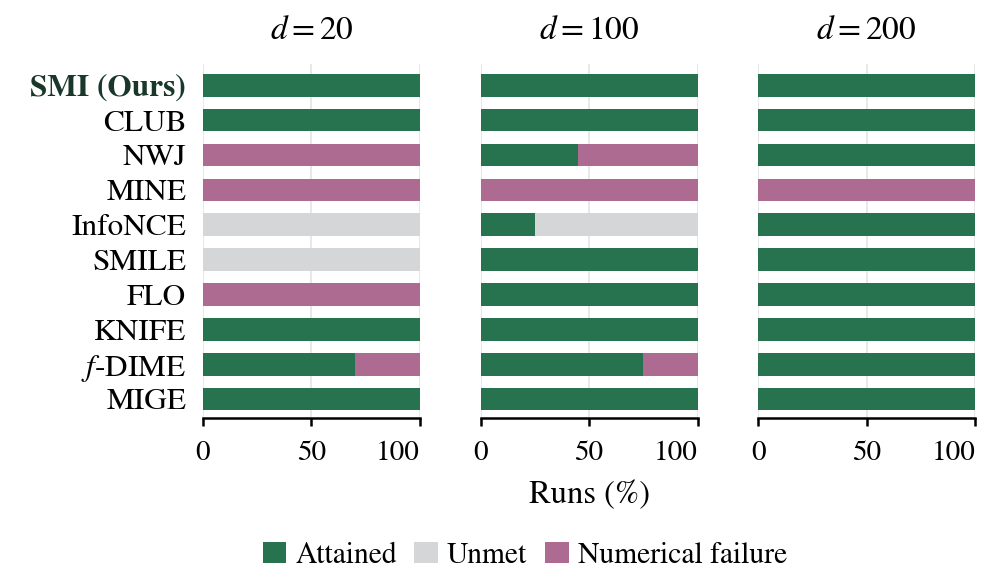}
\caption{MI maximisation.}\label{fig:synthetic-gaussian-max}
\end{subfigure}
\caption{\textbf{Gaussian data.} Outcomes by dimension, pooled over one shared correlation parameter or one per coordinate pair (ten runs each).}
\label{fig:synthetic-gaussian}
\end{figure*}

\begin{figure*}[!t]
\centering
\begin{subfigure}[t]{0.49\textwidth}
\includegraphics[width=\linewidth]{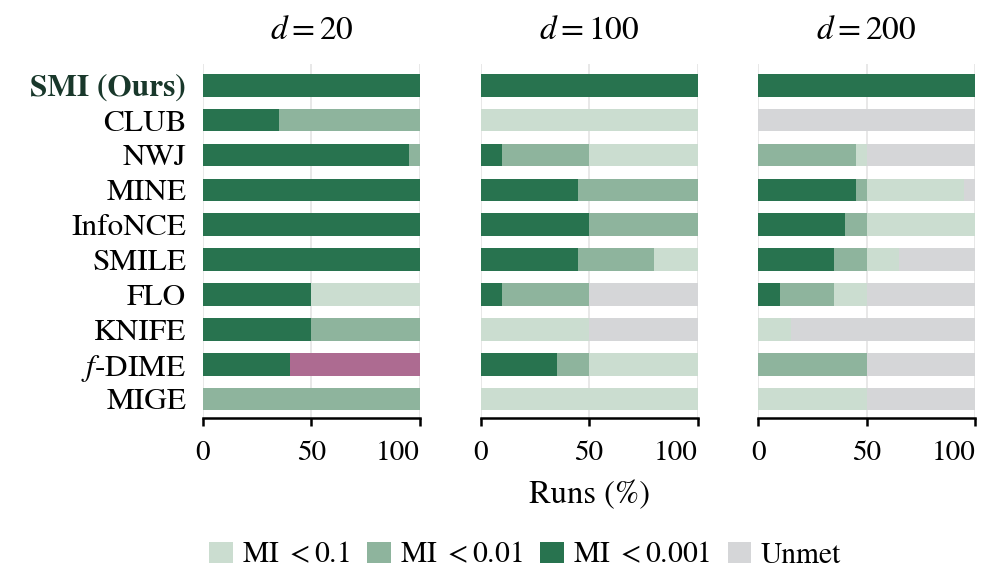}
\caption{MI minimisation.}\label{fig:synthetic-cubic-min}
\end{subfigure}\hfill
\begin{subfigure}[t]{0.49\textwidth}
\includegraphics[width=\linewidth]{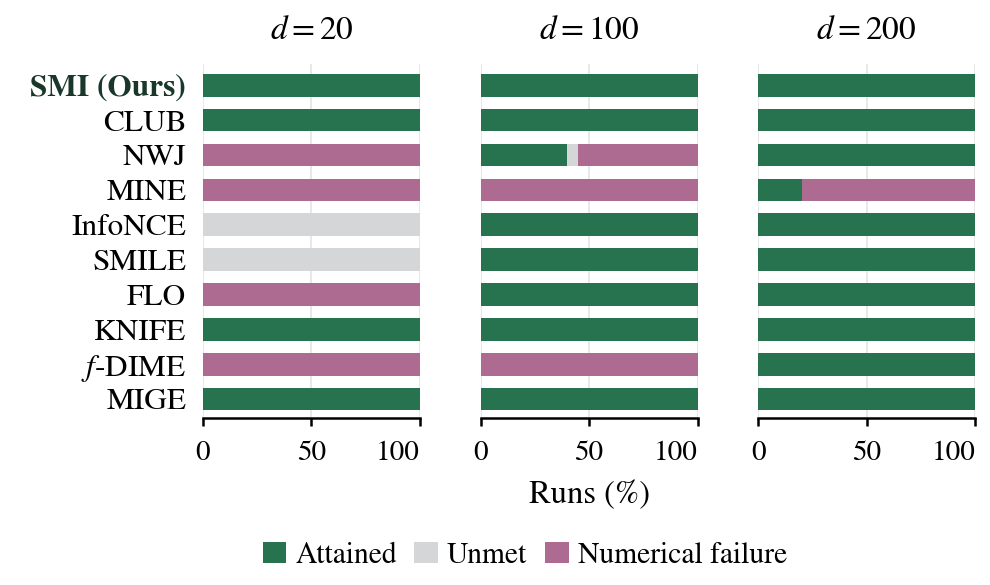}
\caption{MI maximisation.}\label{fig:synthetic-cubic-max}
\end{subfigure}
\caption{\textbf{Cubic data.} Outcomes by dimension, pooled over one shared correlation parameter or one per coordinate pair (ten runs each).}
\label{fig:synthetic-cubic}
\end{figure*}

\begin{figure*}[!t]
\centering
\begin{subfigure}[t]{0.49\textwidth}
\includegraphics[width=\linewidth]{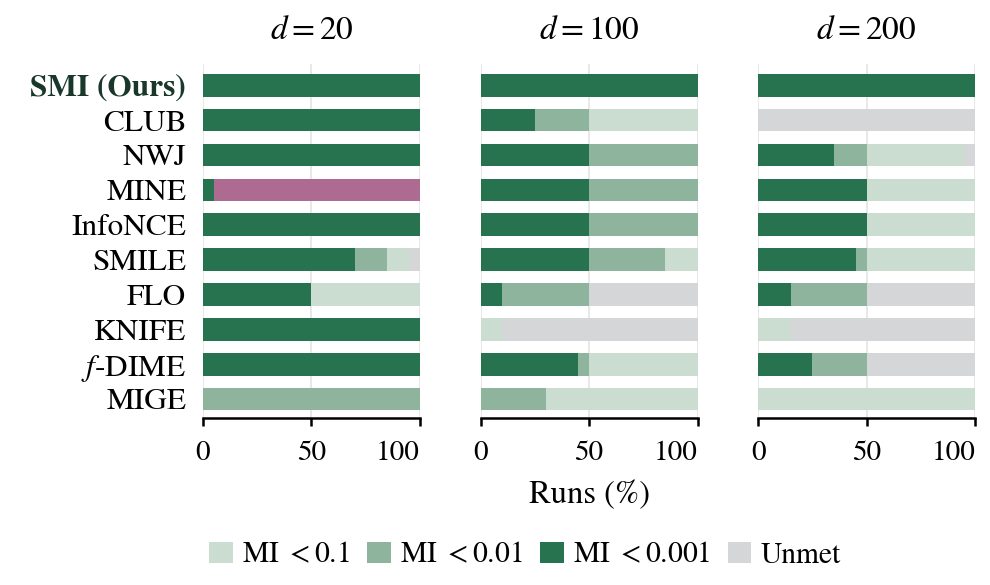}
\caption{MI minimisation.}\label{fig:synthetic-asinh-min}
\end{subfigure}\hfill
\begin{subfigure}[t]{0.49\textwidth}
\includegraphics[width=\linewidth]{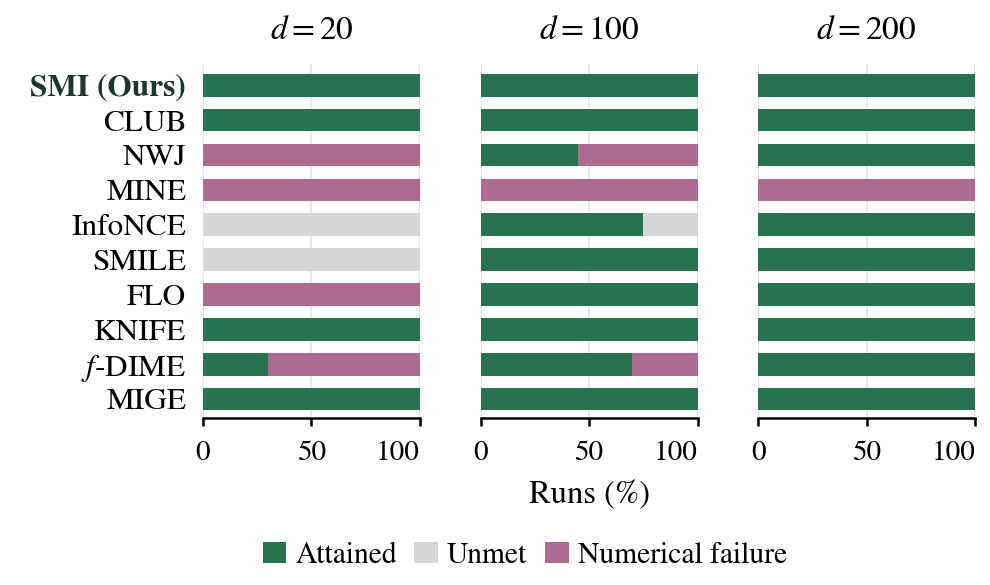}
\caption{MI maximisation.}\label{fig:synthetic-asinh-max}
\end{subfigure}
\caption{\textbf{Asinh data.} Outcomes by dimension, pooled over one shared correlation parameter or one per coordinate pair (ten runs each).}
\label{fig:synthetic-asinh}
\end{figure*}

\begin{figure*}[!t]
\centering
\begin{subfigure}[t]{0.49\textwidth}
\includegraphics[width=\linewidth]{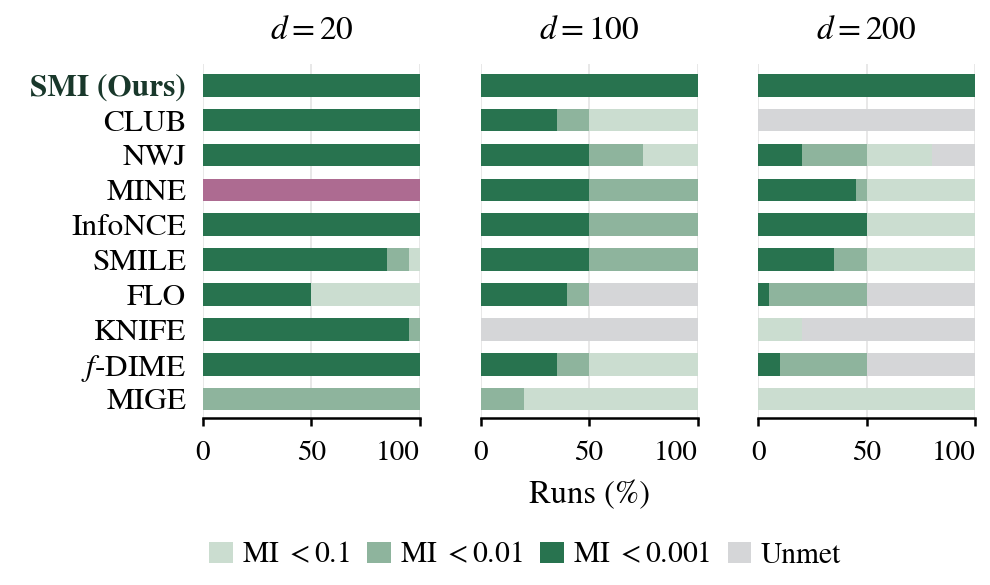}
\caption{MI minimisation.}\label{fig:synthetic-signed-power-min}
\end{subfigure}\hfill
\begin{subfigure}[t]{0.49\textwidth}
\includegraphics[width=\linewidth]{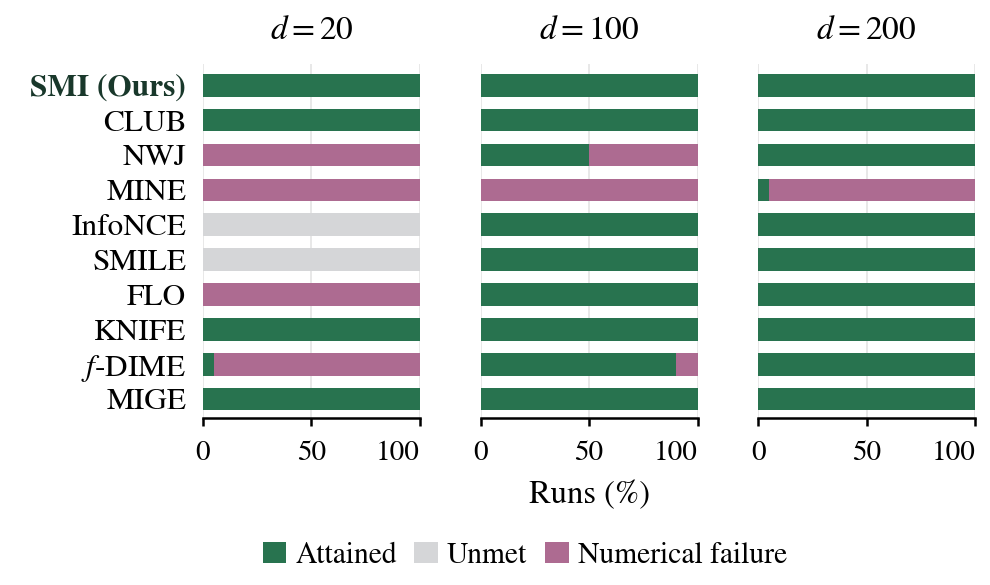}
\caption{MI maximisation.}\label{fig:synthetic-signed-power-max}
\end{subfigure}
\caption{\textbf{Signed-power data.} Outcomes by dimension, pooled over one shared correlation parameter or one per coordinate pair (ten runs each).}
\label{fig:synthetic-signed-power}
\end{figure*}

\begin{figure*}[!t]
\centering
\begin{subfigure}[t]{0.49\textwidth}
\includegraphics[width=\linewidth]{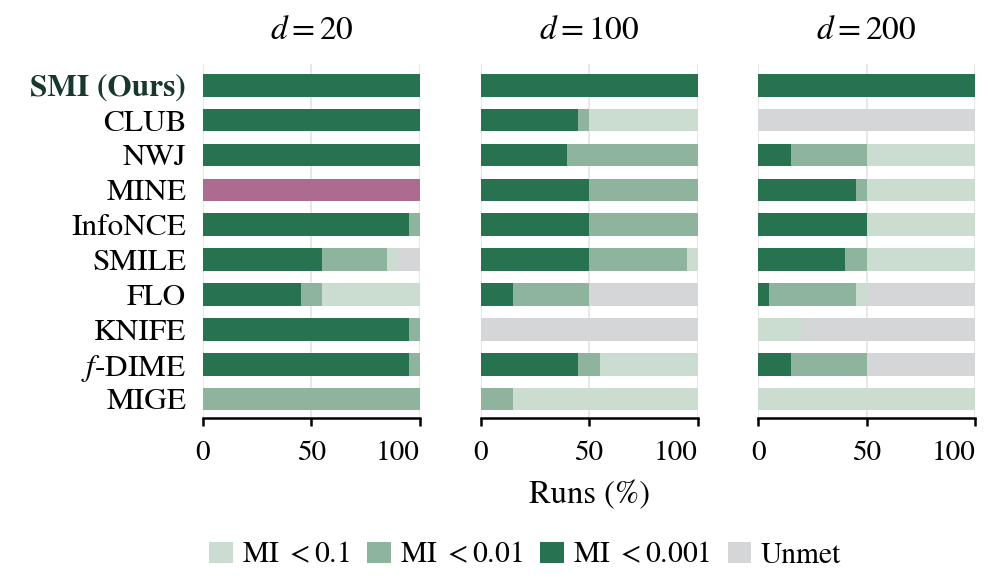}
\caption{MI minimisation.}\label{fig:synthetic-spiral-min}
\end{subfigure}\hfill
\begin{subfigure}[t]{0.49\textwidth}
\includegraphics[width=\linewidth]{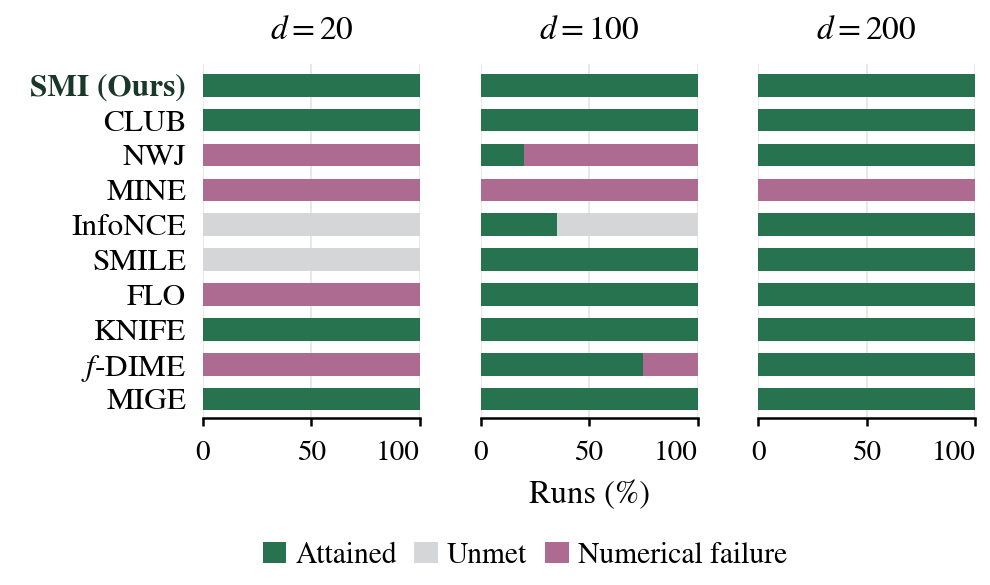}
\caption{MI maximisation.}\label{fig:synthetic-spiral-max}
\end{subfigure}
\caption{\textbf{Spiral data.} Outcomes by dimension, pooled over one shared correlation parameter or one per coordinate pair (ten runs each).}
\label{fig:synthetic-spiral}
\end{figure*}

\subsubsection{Information Across Noise Levels}
\label{app:information-across-noise}

\paragraph{Data and objective.}
We evaluate population quantities underlying \cref{sec:smi-theory-dissipation}.
Let $Z,\epsilon\sim\mathcal N(0,1)$ be mutually independent and independent of $Y$.
Before spreading, we consider
\begin{align}
\text{Noiseless binary:}\qquad
&Y\sim\operatorname{Unif}\{-1,+1\},&
X_0&=Y, \label{eq:noise-noiseless-binary}\\
\text{Binary-input Gaussian:}\qquad
&Y\sim\operatorname{Unif}\{-1,+1\},&
X_0&=Y+0.5Z, \label{eq:noise-binary-gaussian}\\
\text{Correlated Gaussian:}\qquad
&Y\sim\mathcal N(0,1),&
X_0&=0.8Y+0.6Z, \label{eq:noise-correlated-gaussian}\\
\text{Perfectly correlated Gaussian:}\qquad
&Y\sim\mathcal N(0,1),&
X_0&=Y. \label{eq:noise-perfect-gaussian}
\end{align}
The binary-input Gaussian model satisfies
$X_0\mid Y=y\sim\mathcal N(y,0.25)$.
The two Gaussian pairs have correlations $\rho=0.8$ and $\rho=1$, with the latter having a singular joint distribution.
In the order above, the unperturbed MI values are $\log 2$, approximately $0.63272$, $-\log 0.6$, and $+\infty$.
Independent controls use $X_0=Z$ for both choices of $Y$.

\paragraph{Evaluation.}
Write $X_0=aY+bZ$ and spread only $X_0$:
\begin{equation}
U_\tau=X_0+\sqrt{\tau}\epsilon,
\qquad
\tau=\sigma^2>0,
\qquad
q=b^2+\tau.
\end{equation}
Define
$v_\tau(u,y)=\partial_u\log[p_\tau(u\mid y)/p_\tau(u)]$.
For binary $Y$, Gaussian-channel identities \citep{guo_mutual_2005} give
\begin{equation}
I(U_\tau;Y)
=
\log 2-\mathbb E\log\!\left(1+e^{-2aYU_\tau/q}\right),
\qquad
v_\tau(u,y)
=
\frac{a}{q}
\left[
y-\tanh\!\left(\frac{au}{q}\right)
\right].
\end{equation}
For Gaussian $Y$,
\begin{equation}
I(U_\tau;Y)
=
\frac12\log\left(1+\frac{a^2}{q}\right),
\qquad
v_\tau(u,y)
=
-\frac{u-ay}{q}
+\frac{u}{a^2+q}.
\end{equation}
We compute expectations by quadrature and compare five-point finite differences with
$-\partial I/\partial\log\tau
=\tfrac{\tau}{2}\mathbb E[v_\tau^2]$
at 61 log-spaced values of $\sqrt{\tau}\in[10^{-4},10^2]$.
We also verify the integrated identity, the schedule
$\alpha_t=e^{-t/2}$, $\sigma_t^2=1-e^{-t}$, and the generator derivative
$\partial_a I=\mathbb E[Yv_\tau]$.
To examine the zero-noise limit, we fix the generator and use unit weights with
$\sigma\sim\operatorname{LogUniform}(0.01\eta,\eta)$ for
$\eta\in[10^{-6},10^3]$.

\paragraph{Additional results.}
All tested differential and integrated identities agree to absolute error below $10^{-7}$, while the independent controls have zero MI and score energy.
As the noise distribution contracts towards zero, the weighted SMI approaches the finite unperturbed MI when finite and diverges for the perfectly correlated Gaussian case
(\cref{fig:information-across-noise}).

\begin{figure*}[!t]
\centering
\begin{subfigure}[t]{0.49\textwidth}
\includegraphics[width=\linewidth]{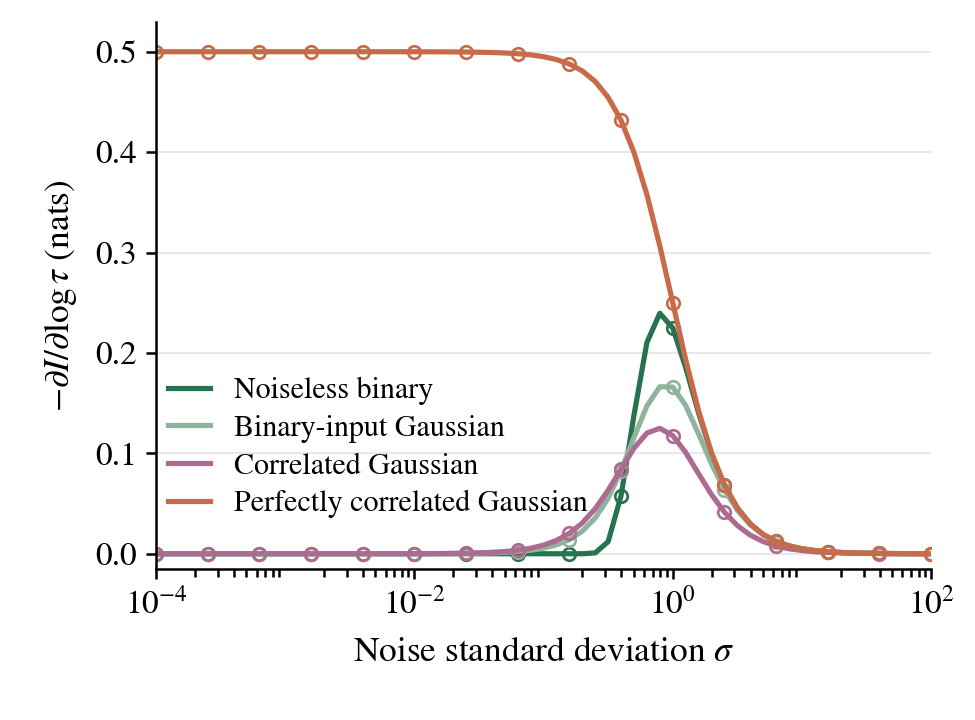}
\caption{Information dissipation.}
\end{subfigure}\hfill
\begin{subfigure}[t]{0.49\textwidth}
\includegraphics[width=\linewidth]{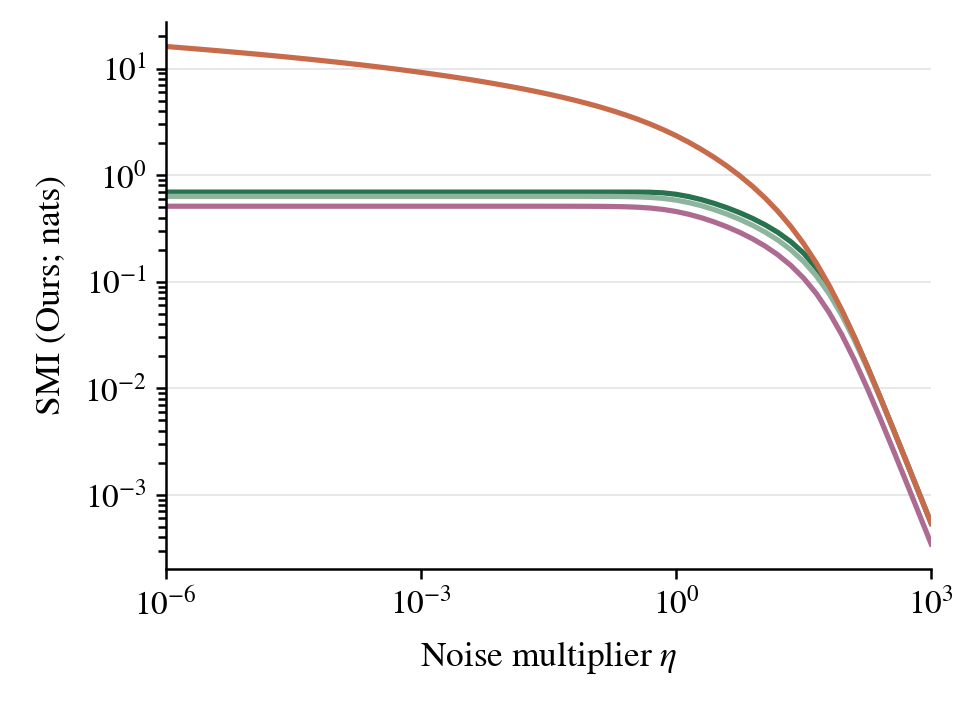}
\caption{Rescaling the noise distribution.}
\end{subfigure}
\caption{\textbf{Information across noise levels.} Exact-density
evaluations. (a) $\tfrac\tau2\mathbb E[v_\tau^2]$ (lines) and
$-\partial I/\partial\log\tau$ (markers), with $\tau=\sigma^2$.
(b) Multiscale SMI for
$\sigma\sim\operatorname{LogUniform}(0.01\eta,\eta)$.}
\label{fig:information-across-noise}
\end{figure*}

\subsection{Missing Data Imputation}
\label{app:tabular-imputation}

\paragraph{Data and objective.}
We use all rows from ten TabArena datasets \citep{erickson_tabarena_2025}, retaining numerical features and excluding targets (\cref{tab:imputation-datasets}).
For Airfoil, angle of attack is represented in degrees; Wine retains both red and white wines but excludes the colour indicator; Miami Housing excludes the categorical feature \texttt{avno60plus}.
We perform transductive imputation without a train/test split, retaining duplicate and entirely masked rows.

Three runs independently repeat masking, initialisation, and fitting.
Within each run, uniform draws for each cell define nested $30\%$, $60\%$, and $80\%$ MCAR masks, which are fixed and shared across methods.
Masking precedes preprocessing: each feature is standardised using its observed mean and population standard deviation, with zero scales replaced by one.
Masked values are used only for evaluation.
The initial seven datasets were selected after reviewing an eight-dataset study; Airfoil, Wine, and Miami Housing were selected before their results were available.
All fits on the ten reported datasets are included.

For observed-entry mask $M$, we minimise dependence between the completed data and the mask:
\begin{equation}
\min_\theta I(X_\theta;M),\qquad
X_\theta
=
M\odot X+(1-M)\odot g_\theta(M\odot X,M,Z).
\end{equation}
This is the MI objective underlying WGF and MIRI, described below, with observed entries held fixed.

\begin{table}[!t]
\centering\small
\setlength{\tabcolsep}{4pt}
\caption{\textbf{TabArena datasets.} All rows; numerical feature counts exclude targets. Licences follow the linked OpenML records.}
\label{tab:imputation-datasets}
\begin{tabular}{@{}llrrl@{}}
\toprule
Dataset & Source & Rows & Features & Licence \\
\midrule
Blood & \url{https://www.openml.org/d/46913} & 748 & 4 & CC BY 4.0 \\
Concrete & \url{https://www.openml.org/d/46917} & 1,030 & 8 & CC BY 4.0 \\
Fish & \url{https://www.openml.org/d/46954} & 907 & 6 & CC BY 4.0 \\
Hazelnut & \url{https://www.openml.org/d/46930} & 2,400 & 30 & CC BY-SA \\
Houses & \url{https://www.openml.org/d/46934} & 20,640 & 8 & Public \\
Maternal & \url{https://www.openml.org/d/46941} & 1,014 & 6 & CC BY 4.0 \\
Protein & \url{https://www.openml.org/d/46949} & 45,730 & 9 & CC BY 4.0 \\
Airfoil & \url{https://www.openml.org/d/46904} & 1,503 & 5 & CC BY 4.0 \\
Wine & \url{https://www.openml.org/d/46964} & 6,497 & 11 & CC BY 4.0 \\
Miami Housing & \url{https://www.openml.org/d/46942} & 13,776 & 14 & CC BY-NC-SA \\
\bottomrule
\end{tabular}
\par\smallskip
\begin{minipage}{\textwidth}\footnotesize
OpenML specifies no licence version for Hazelnut or Miami Housing; ``Public'' for Houses does not identify a standard licence.
\end{minipage}
\end{table}
\paragraph{MI-based methods.}
The ten matched methods share the residual imputer
\begin{align}
g_\theta(M\odot X,M,Z)=Z+h_\theta(M\odot X,M,Z),
\end{align}
where the coordinates of $Z$ are sampled independently from the observed feature marginals.
SMI adds Gaussian noise, and after restoring observed entries, it targets
\begin{equation}
\label{eq:imputation-smi}
\mathbb E_{\sigma\sim\operatorname{LogUniform}(0.01,1)}
\big[I(X_\theta+\sigma\epsilon;M)\big].
\end{equation}
Here $\epsilon\sim\mathcal N(0,I_d)$ is independent across examples and between joint and product draws; masks are left unchanged.
Binary cross-entropy fits the scale-conditioned logit
$f_\phi(T,M,\log\sigma)$, where $T=X_\theta+\sigma\epsilon$.
The imputer minimises the stop-gradient surrogate
\begin{equation}
\mathbb E\!\left[
\operatorname{stopgrad}\!\left(\nabla_T f_\phi(T,M,\log\sigma)\right)^\top T
\right].
\end{equation}

CLUB \citep{cheng_club_2020} uses a sampled diagonal-Gaussian $q(X\mid M)$ with a tanh log-variance head.
KNIFE \citep{pichler_differential_2022} fits eight-component full-covariance marginal and conditional mixtures.
NWJ \citep{nguyen_estimating_2010}, MINE \citep{belghazi_mutual_2018},
InfoNCE \citep{oord_representation_2018}, and SMILE
\citep{song_understanding_2020} retain their released all-pairs objectives,
including MINE's moving-normaliser backward rule and SMILE's
Jensen--Shannon backward gradient with clipping threshold five.
FLO \citep{guo_tight_2022} uses normalised embeddings, fixed inverse temperature one, and an auxiliary variational head.
The KL variant of $f$-DIME \citep{letizia_mutual_2024} uses a Softplus ratio, deranged product pairs, and $\alpha=1$, without a log stabiliser.
MIGE \citep{wen_mutual_2020} uses the joint-minus-marginal spectral score field with Gaussian median-distance bandwidth, eigen-energy threshold $0.99$, and no Gram regularisation; only completion-coordinate gradients are used.

\paragraph{Other methods.}
The six established imputers retain their imputation-specific procedures.

GAIN \citep{yoon_gain_2018} and MIWAE \citep{mattei_miwae_2019} use the HyperImpute 0.1.17 implementations.
GAIN retains sigmoid heads, observed min--max normalisation, hint probability $0.9$, reconstruction weight $10$, uniform $[0,0.01]$ missing-input noise, and separate discriminator and generator minibatches.
This input noise is distinct from SMI's spreading of completed rows.
MIWAE uses a Gaussian proposal, Student-$t$ likelihood, 20 training and ten prediction importance samples, and weighted sampled-value prediction; its latent dimension equals the feature count.

MIRI \citep{yu_missing_2025} uses rectified flow \citep{liu_flow_2023}, missing-coordinate velocity loss, and coupled augmented-state Euler integration.
Its ideal proximal objective satisfies
\begin{equation}
\operatorname{KL}(Q^{X,M}\Vert Q_{t-1}^X\otimes P_M)
=
I_Q(X;M)+\operatorname{KL}(Q^X\Vert Q_{t-1}^X).
\end{equation}
WGF \citep{liu_minimizing_2024} uses the local-linear quadratic dual, a product of Gaussian-data and linear-mask kernels, independent Bernoulli product masks, and unit particle steps on missing entries.
Its bandwidth is the square root of half the lower median of ordered squared distances, with all rows contributing to the kernel calculation.

HyperImpute \citep{jarrett_hyperimpute_2022} and MissForest \citep{stekhoven_missforest_2012} use HyperImpute 0.1.17 plugins.
HyperImpute selects among random forests, linear/logistic models, XGBoost, CatBoost, and neural networks, using simple optimisation, per-column/per-iteration selection, patience five, lazy selection, and at most 40 inner iterations.
MissForest retains its native defaults and stopping rule, with at most 100 iterations.

\paragraph{Training.}
The matched imputer has three width-256 SiLU hidden layers and a zero-initialised output layer.
Adam uses default moments, no weight decay, batch size $256$, and $10{,}000$ imputer updates.
The learning rate is $10^{-5}$ for the first $2{,}000$ updates, then cosine-decayed to $10^{-6}$.
Applicable estimator networks use the same hidden-layer architecture with method-specific heads; they receive $5{,}000$ initial fitting updates and 20 updates per imputer step at learning rate $3\times10^{-4}$.
MIGE has no learned auxiliary optimiser.
Losses retain their native units without normalisation by feature count.
We use no reconstruction or remasking loss, moment penalty, output clipping, external tuning, or metric-based early stopping.

GAIN and MIWAE use three width-256 SiLU hidden layers and $10{,}000$ Adam updates with nominal batch size $256$.
Their generator learning rates are $3\times10^{-4}$ for the first $2{,}000$ updates, then cosine-decayed to $3\times10^{-5}$; GAIN's discriminator remains at $3\times10^{-4}$.
MIRI uses the same hidden-layer architecture, ten fresh-network rounds of $1{,}000$ Adam updates at constant learning rate $3\times10^{-4}$, batch size $256$, and 100 Euler steps.
WGF uses 500 outer iterations, each with 101 RMSprop updates at learning rate $10^{-3}$, $\alpha=0.9$, and $\epsilon=10^{-6}$.

\paragraph{Evaluation.}
For missing cells $\Omega=\{(i,j):M_{ij}=0\}$, pointwise error is
\begin{equation}
\operatorname{RMSE}
=
\sqrt{
|\Omega|^{-1}
\sum_{(i,j)\in\Omega}
(\hat x_{ij}-x_{ij})^2
}.
\end{equation}
Distributional error is the biased empirical squared maximum mean discrepancy
\citep{gretton_kernel_2012} between complete rows:
\begin{equation}
\widehat{\operatorname{MMD}}^2
=
\frac{1}{n^2}\sum_{i,j=1}^{n}
\left[
k_h(x_i,x_j)
+k_h(\hat x_i,\hat x_j)
-2k_h(x_i,\hat x_j)
\right],
\qquad
k_h(u,v)=e^{-\|u-v\|^2/(2h^2)}.
\end{equation}
The bandwidth is
$h=\operatorname{median}_{i,j}\|x_i-\hat x_j\|$,
including diagonal pairs, and is computed separately for each fit.
Both metrics use one unperturbed completion in observed-only standardised coordinates, with observed entries restored exactly.
MMD evaluation uses all rows.

Within each dataset--rate--run block, all 16 methods are ranked jointly.
Finite ties receive average ranks, while failed fits receive rank 16.
Ranks are averaged equally across datasets for each rate and run, and the main table reports the mean and SD across the three run-level averages.
No block is omitted.
There are 43 numerical failures and one reproducible HyperImpute implementation error on Wine at $80\%$ missingness, caused by a singleton batch in upstream BatchNorm.

\paragraph{Additional results.}
\Cref{fig:tabular-imputation-raw} reports every method, dataset, and missingness rate.
Circles denote individual runs; diamonds with SD whiskers are shown only when all three runs succeed.
Incomplete groups retain finite runs without a mean and mark numerical failures as n.f.\ ($k/3$) or implementation errors as err.\ ($k/3$).
Logarithmic axis limits are shared across missingness rates within each dataset and metric.

\begin{figure*}[!t]
\centering
\captionsetup[subfigure]{skip=0pt,font=small}
\begin{subfigure}{\textwidth}
\includegraphics[width=\linewidth]{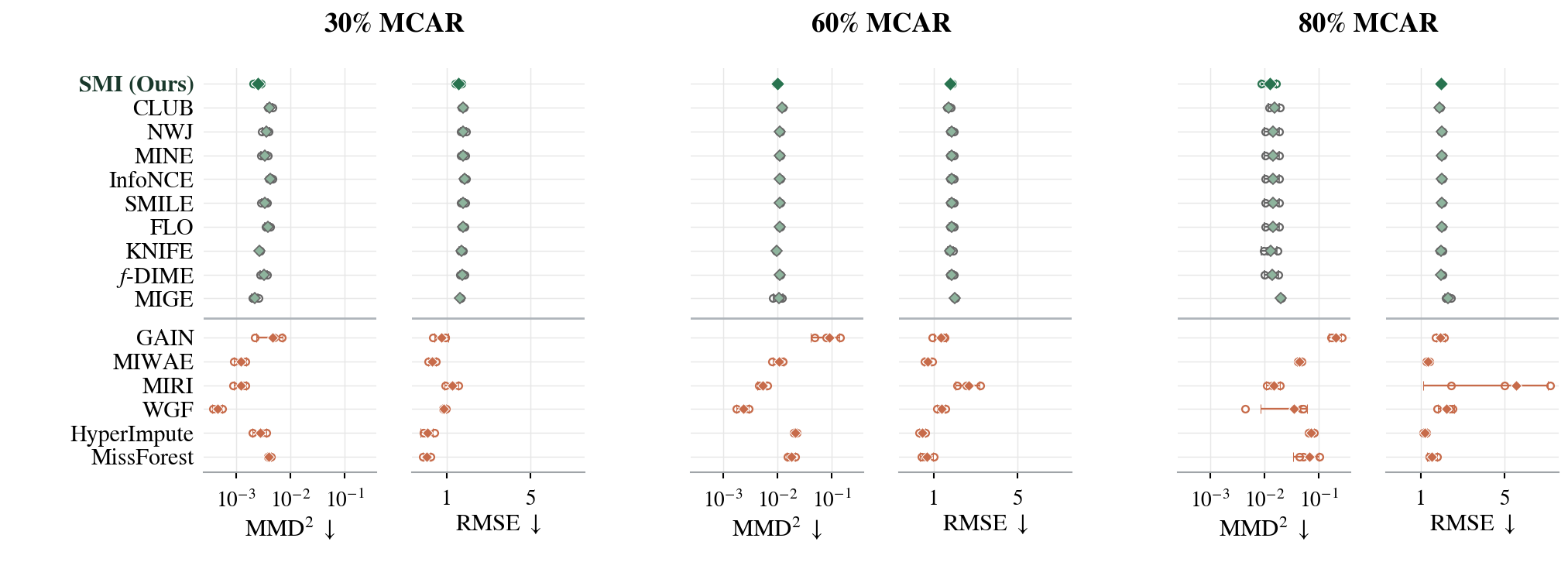}
\caption{Blood.}\label{fig:imputation-raw-blood}
\end{subfigure}\par\vspace{3pt}
\begin{subfigure}{\textwidth}
\includegraphics[width=\linewidth]{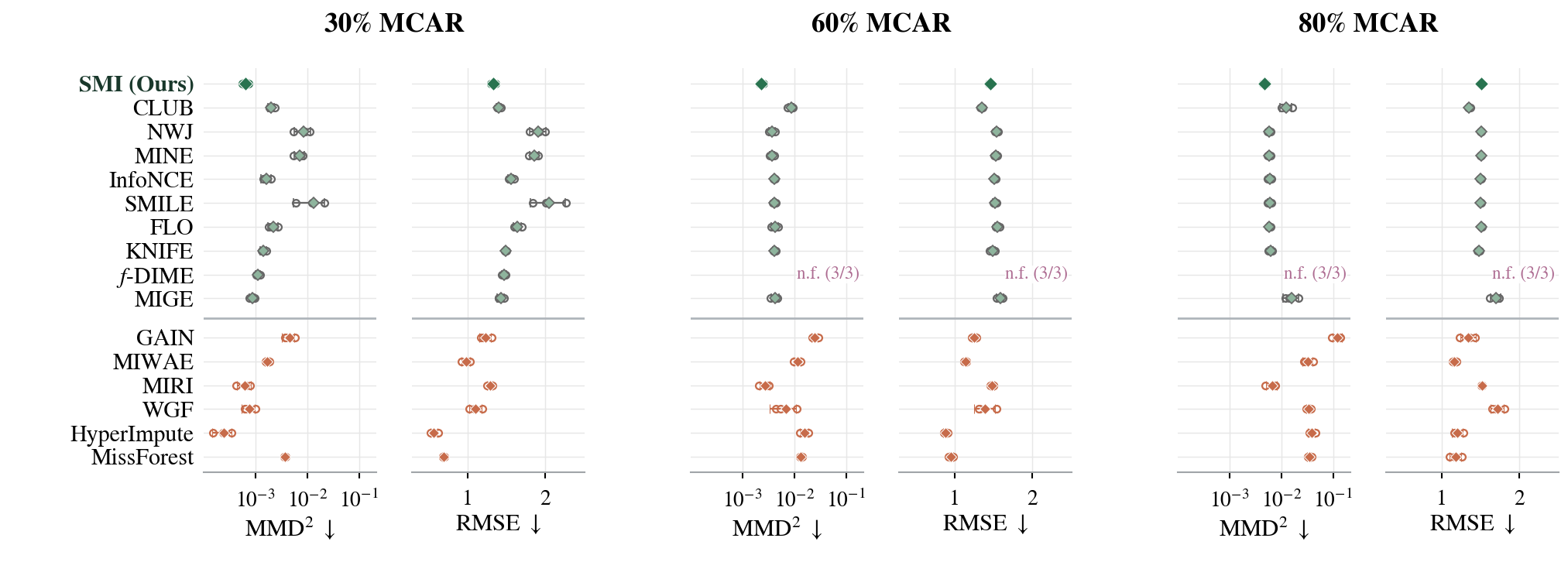}
\caption{Concrete.}\label{fig:imputation-raw-concrete}
\end{subfigure}
\caption{\textbf{Missing data imputation.} MMD$^2$ and RMSE at 30\%, 60\% and 80\% MCAR; individual runs and mean $\pm$ SD for groups with three successful fits.}
\label{fig:tabular-imputation-raw}
\end{figure*}

\begin{figure*}[!t]
\ContinuedFloat
\centering
\captionsetup[subfigure]{skip=0pt,font=small}
\begin{subfigure}{\textwidth}
\includegraphics[width=\linewidth]{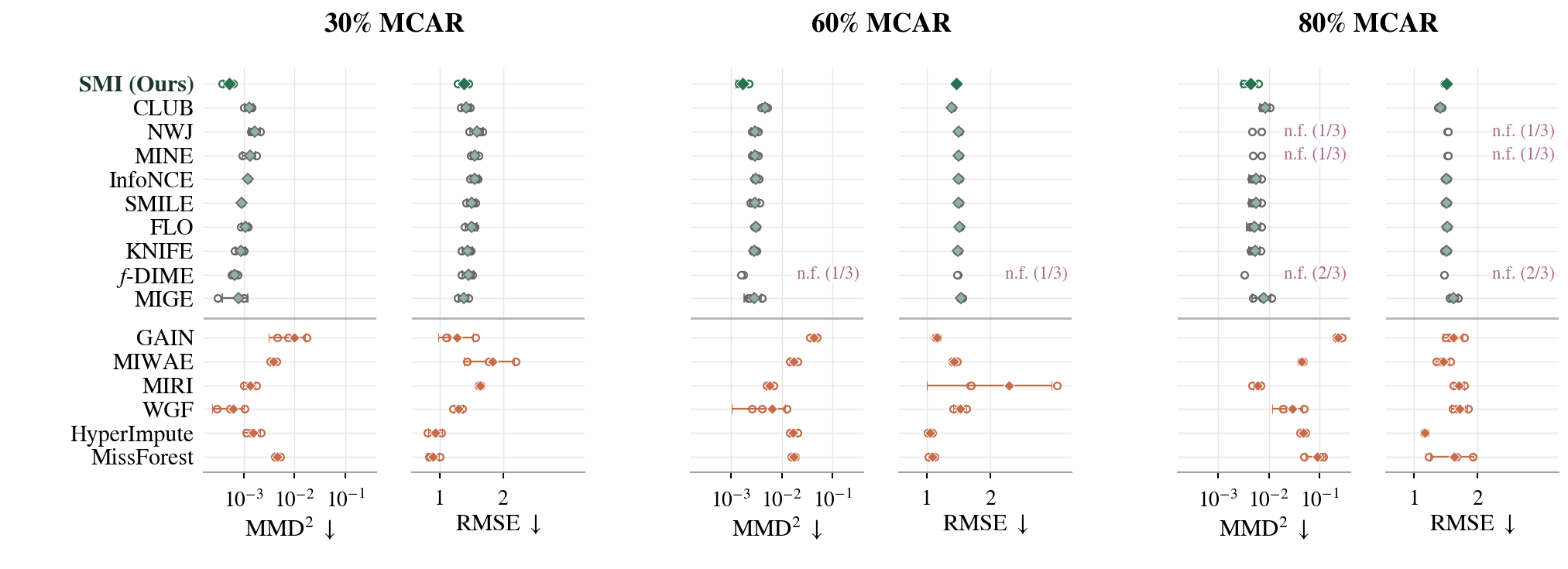}
\caption{Fish.}\label{fig:imputation-raw-fish}
\end{subfigure}\par\vspace{3pt}
\begin{subfigure}{\textwidth}
\includegraphics[width=\linewidth]{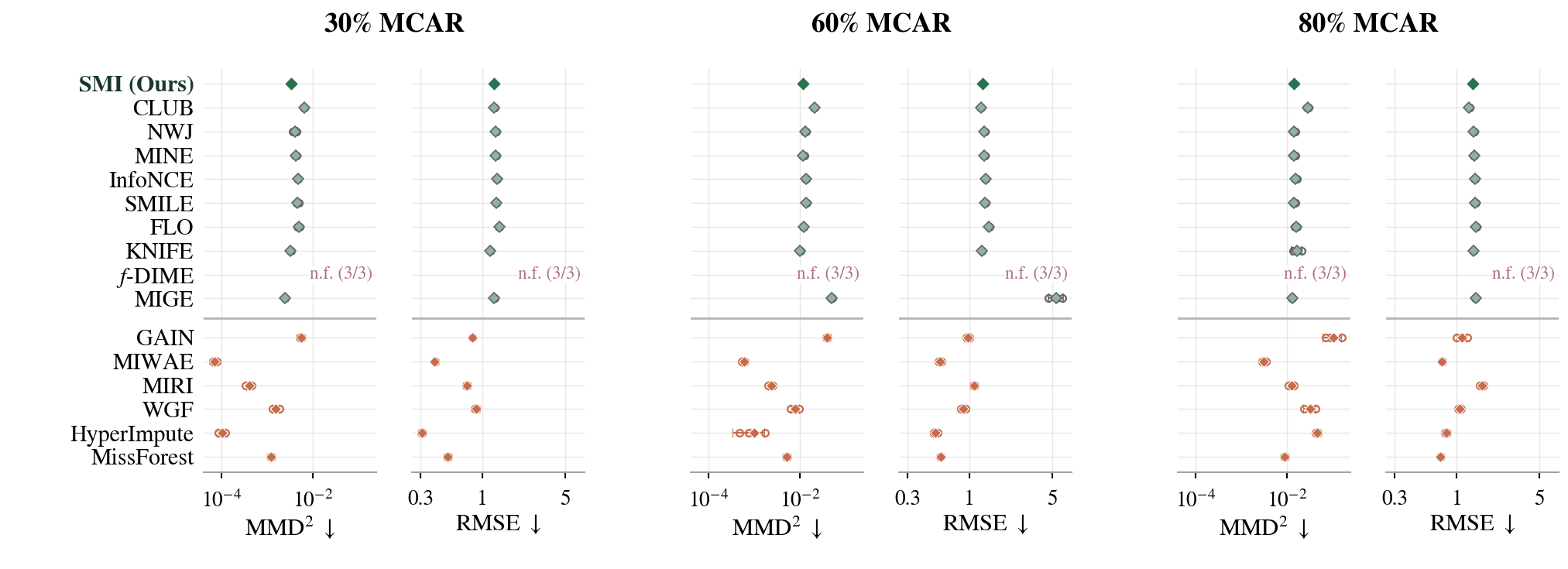}
\caption{Hazelnut.}\label{fig:imputation-raw-hazelnut}
\end{subfigure}\par\vspace{3pt}
\begin{subfigure}{\textwidth}
\includegraphics[width=\linewidth]{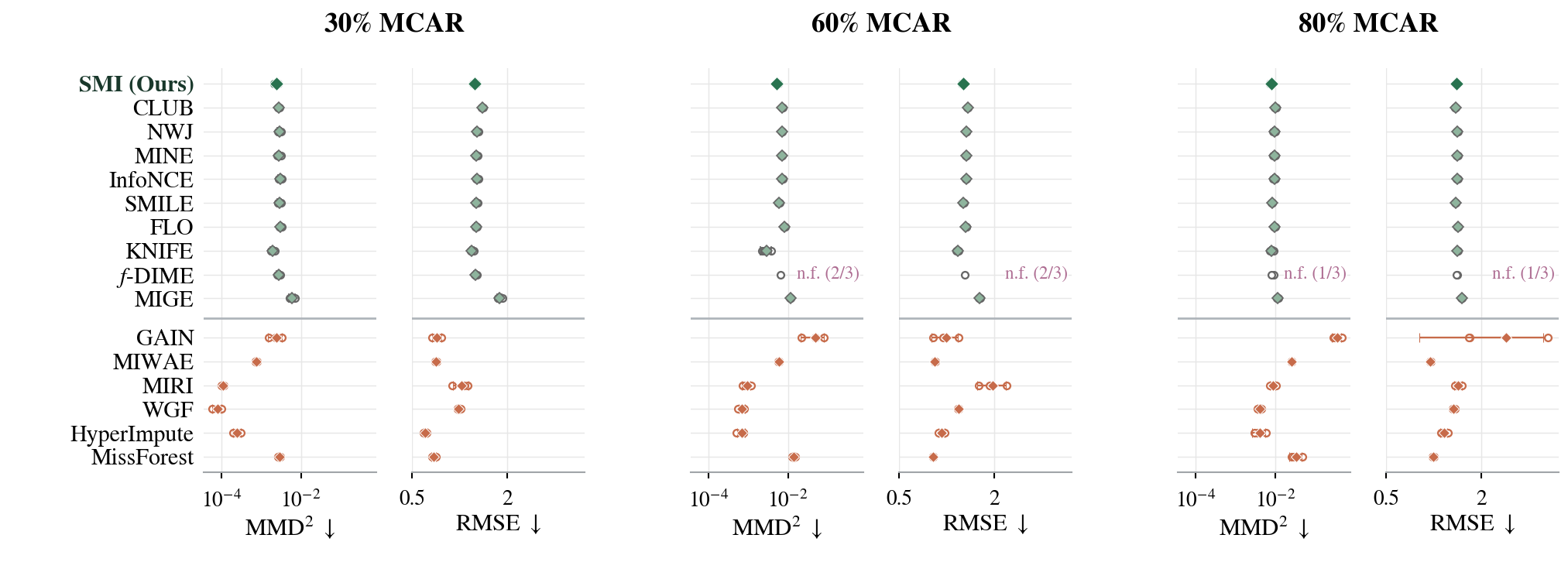}
\caption{Houses.}\label{fig:imputation-raw-houses}
\end{subfigure}
\caption[]{\textbf{Missing data imputation (continued).}}
\end{figure*}

\begin{figure*}[!t]
\ContinuedFloat
\centering
\captionsetup[subfigure]{skip=0pt,font=small}
\begin{subfigure}{\textwidth}
\includegraphics[width=\linewidth]{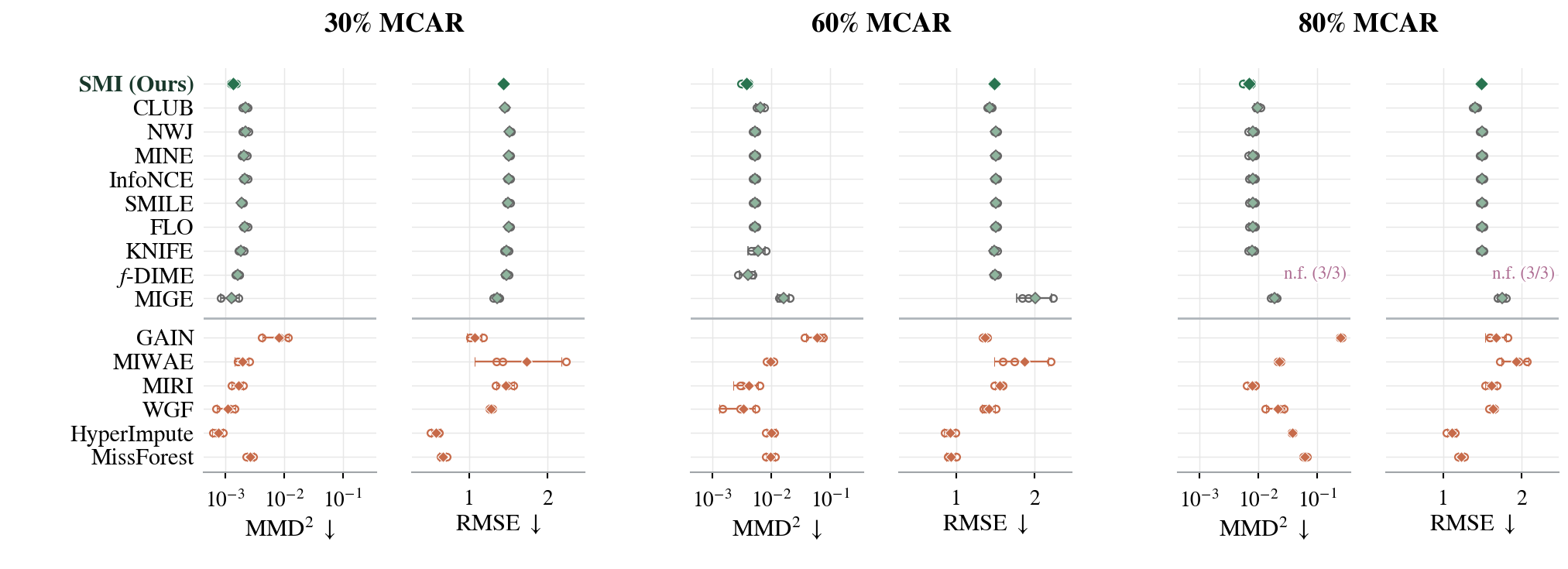}
\caption{Maternal.}\label{fig:imputation-raw-maternal}
\end{subfigure}\par\vspace{3pt}
\begin{subfigure}{\textwidth}
\includegraphics[width=\linewidth]{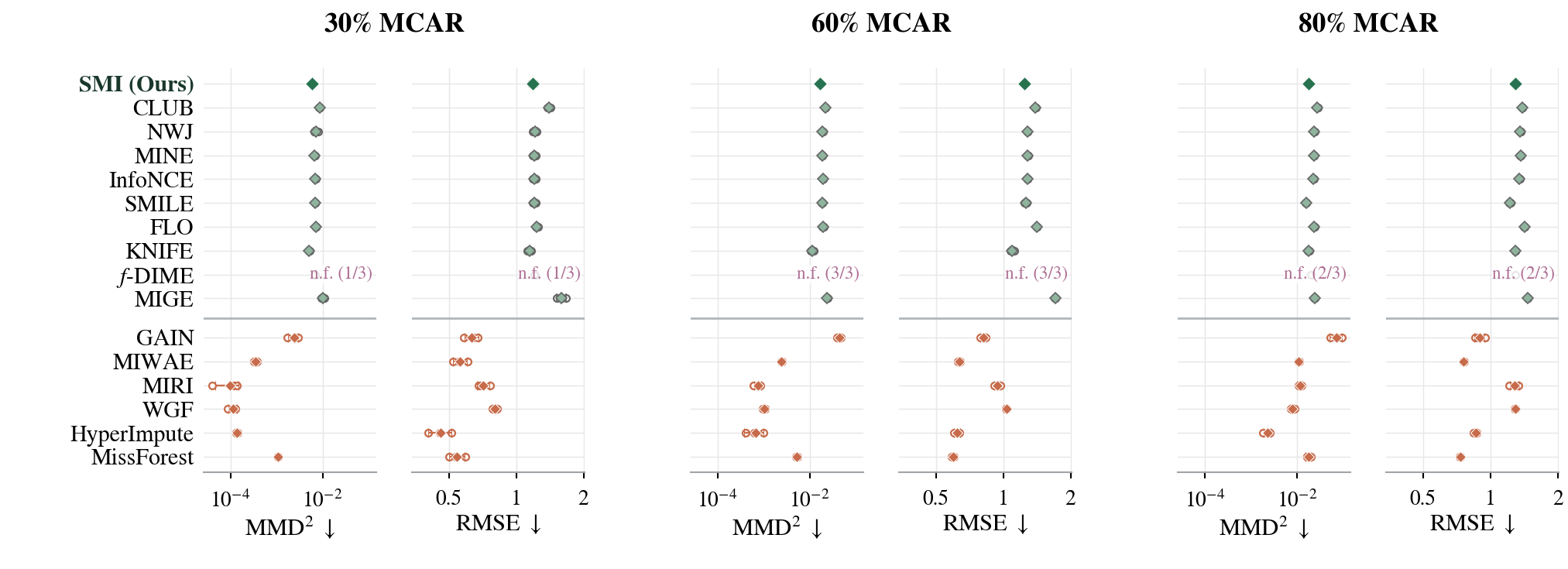}
\caption{Protein.}\label{fig:imputation-raw-protein}
\end{subfigure}\par\vspace{3pt}
\begin{subfigure}{\textwidth}
\includegraphics[width=\linewidth]{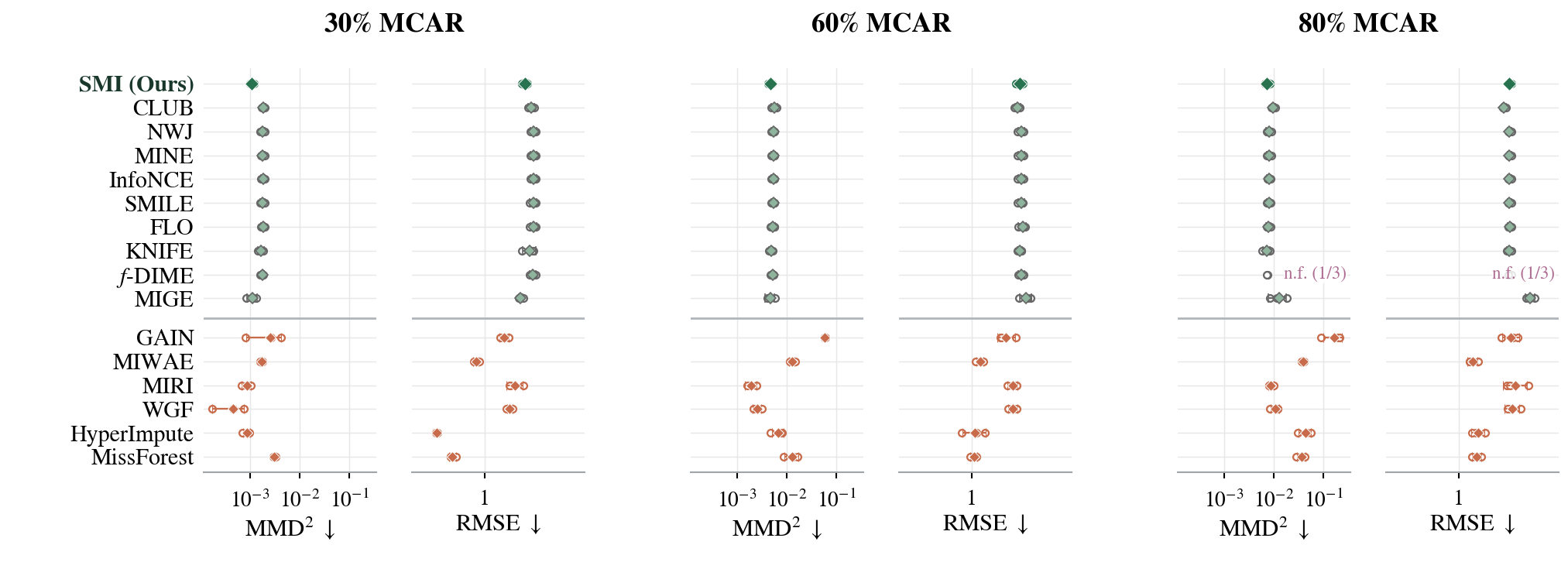}
\caption{Airfoil.}\label{fig:imputation-raw-airfoil}
\end{subfigure}
\caption[]{\textbf{Missing data imputation (continued).}}
\end{figure*}

\begin{figure*}[!t]
\ContinuedFloat
\centering
\captionsetup[subfigure]{skip=0pt,font=small}
\begin{subfigure}{\textwidth}
\includegraphics[width=\linewidth]{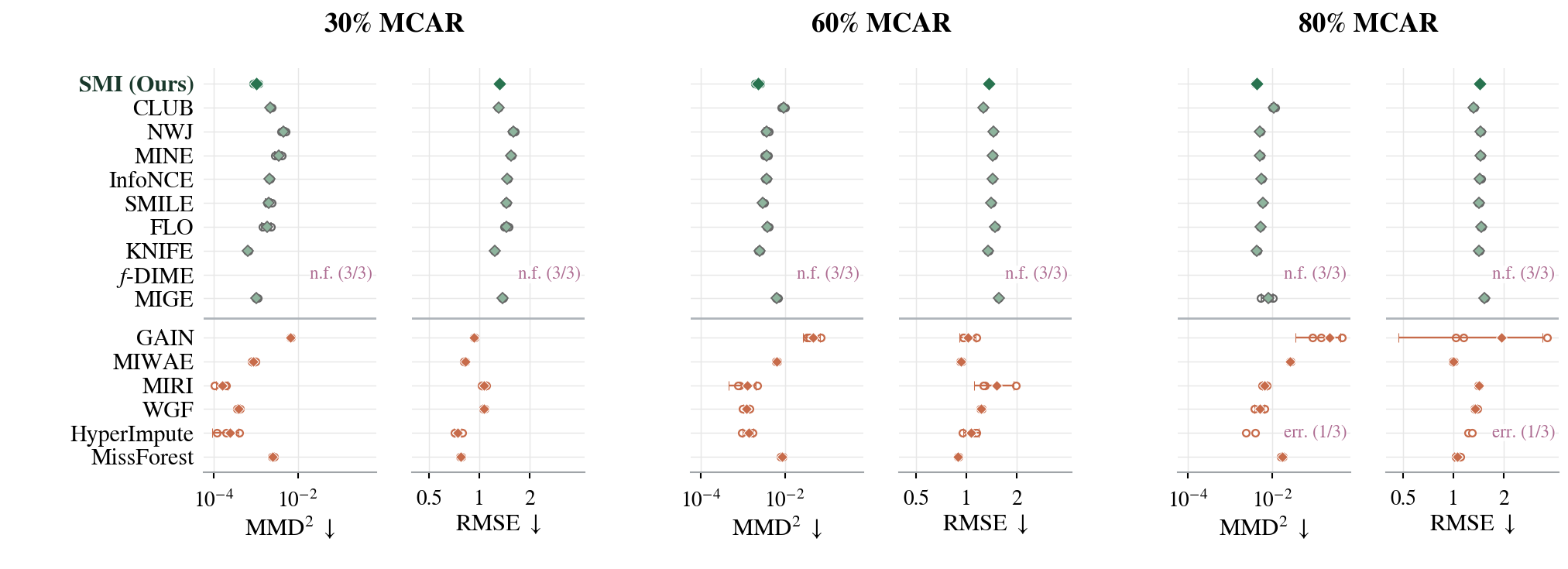}
\caption{Wine.}\label{fig:imputation-raw-wine}
\end{subfigure}\par\vspace{3pt}
\begin{subfigure}{\textwidth}
\includegraphics[width=\linewidth]{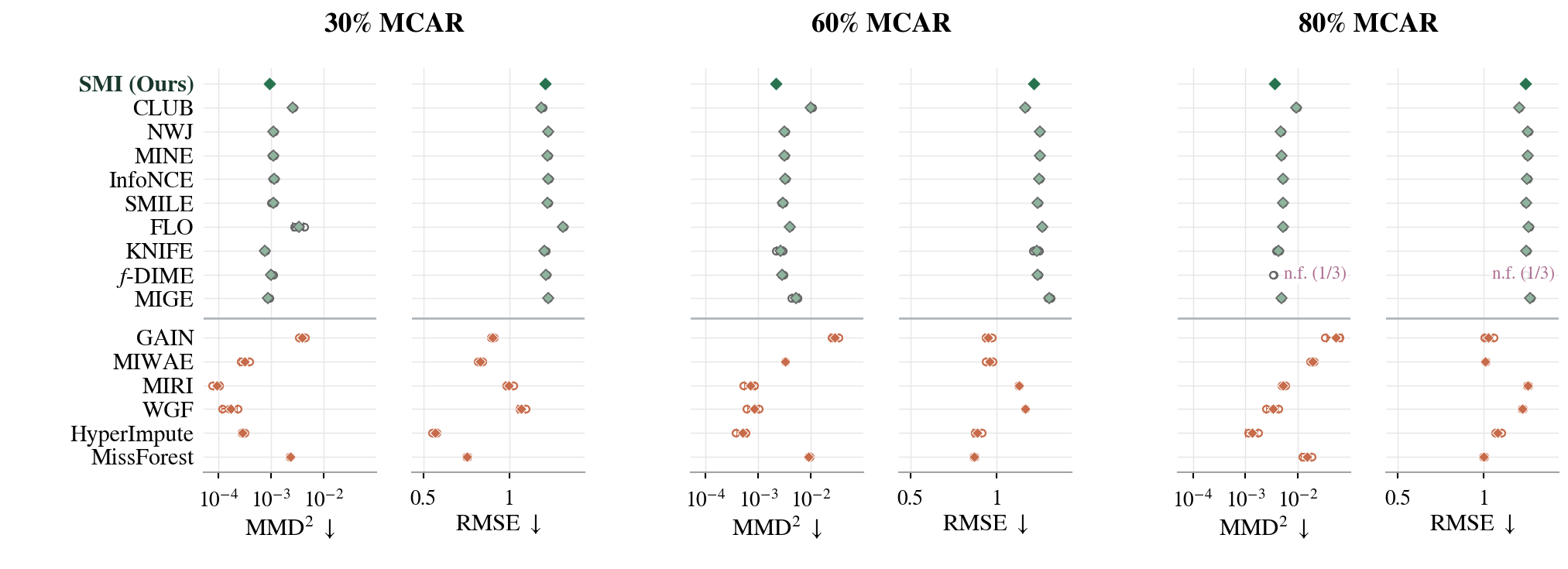}
\caption{Miami Housing.}\label{fig:imputation-raw-miami}
\end{subfigure}
\caption[]{\textbf{Missing data imputation (continued).}}
\end{figure*}

\subsection{Nuisance Suppression on Waterbirds}
\label{app:nuisance}

\paragraph{Data and objective.}
Waterbirds \citep{sagawa_distributionally_2020} contains $4{,}795$ training, $1{,}199$ validation, and $5{,}794$ test images.
The release is linked at \url{https://github.com/kohpangwei/group_DRO#waterbirds}.
Its CUB-200-2011 bird images (\url{https://www.vision.caltech.edu/datasets/cub_200_2011/}) are restricted to non-commercial research and education; Places backgrounds (\url{http://places2.csail.mit.edu/}) are provided for academic research and education.
Image copyrights remain with their owners.
We split the validation set into 598 tuning and 601 development-reporting examples; both are used during development.
From frozen ResNet-50 features, an adapter produces a 64-dimensional representation $R_\theta$ for bird classification.
MI-based methods minimise
\begin{equation}
\mathcal L_{\mathrm{CE}}(\theta)+\beta I(R_\theta;A\mid L),
\end{equation}
where $A$ denotes background and $L$ bird class; SMI replaces the MI term with its multiscale spread counterpart.

Test data are not used to fit models, probes, or standardisers.
Test results informed baseline and probe-protocol development.

\paragraph{MI-based methods.}
All ten methods use regularisation coefficients $\beta\in\{0.2,1,5\}$.
SMI uses
$\widetilde R=R_\theta+\sigma\epsilon$,
with independent Gaussian $\epsilon$ and
$\sigma\sim\operatorname{LogUniform}(0.5,2)$; labels are unchanged.
The critic receives $\widetilde R/\sqrt{1+\sigma^2}$, retaining gradients through this rescaling.
Class-conditional softmax posteriors parameterise the log density ratio:
\begin{equation}
\label{eq:waterbirds-ratio}
T_\psi(\widetilde r,a,l,\sigma)
=
\log q_\psi(a\mid\widetilde r,l,\sigma)
-\log\pi_{\rm train}(a\mid l),
\qquad
\sum_a
\pi_{\rm train}(a\mid l)
e^{T_\psi(\widetilde r,a,l,\sigma)}
=1.
\end{equation}
At the exact posterior, this recovers the conditional log density ratio.
SMI fits joint-versus-product binary cross-entropy and uses the stopped input gradient to update the representation.
Joint and product samples use independent noise draws from the same distribution.
Product representations are sampled with replacement within bird class, and group-balanced estimator samples are reweighted to the empirical training distribution.

The other MI methods use unperturbed, class-conditioned features.
MINE \citep{belghazi_mutual_2018} retains its moving-normaliser gradient with decay $0.9$ and floor $10^{-4}$; SMILE \citep{song_understanding_2020} uses clipping threshold five and its Jensen--Shannon backward rule.
InfoNCE \citep{oord_representation_2018} draws nuisance candidates independently within class, allowing repeated labels, with 256 candidates for estimator updates and 64 for representation updates.
CLUB \citep{cheng_club_2020} fits a binary posterior by likelihood and minimises the joint-minus-product log-posterior contrast.
NWJ \citep{nguyen_estimating_2010} and KL $f$-DIME \citep{letizia_mutual_2024} use independent within-class product draws.
FLO \citep{guo_tight_2022} uses within-class IID batches, normalised embeddings, and inverse temperature one.
KNIFE \citep{pichler_differential_2022} fits eight-component full-covariance mixtures $p(R\mid A,L)$ and forms
\begin{equation}
p(R\mid L)
=
\sum_a
\pi_{\rm train}(a\mid L)\,
p(R\mid a,L).
\end{equation}
MIGE \citep{wen_mutual_2020} uses conditional-minus-marginal spectral scores, independent reference batches of 512, median-distance bandwidth, and eigen-energy threshold $0.99$.
For KNIFE and MIGE, orthonormal coordinates remove the exact zero-sum constraint induced by LayerNorm without information loss; prediction and probes still use the original 64-dimensional representation.

\paragraph{Other methods.}
Class-conditional adversarial training \citep{edwards_censoring_2016,louppe_learning_2017} reverses the representation gradient of nuisance cross-entropy, with $\lambda\in\{0.2,1,5\}$.
HSIC \citep{gretton_measuring_2005} uses class-weighted within-class estimates, an RBF representation kernel with squared bandwidth 64, a nuisance equality kernel, and $\lambda\in\{3,10,30\}$.
ERM and GroupDRO \citep{sagawa_distributionally_2020} use weight decay in $\{0,0.001,0.01\}$.
GroupDRO uses group-balanced batches and exponentiated group-weight updates with step size $0.01$.
DFR \citep{kirichenko_last_2023} fits $\ell_1$-regularised logistic heads on a frozen adapter trained without dependence or weight-decay penalties.
It uses $C\in\{0.01,0.1,1\}$ and averages coefficients from ten independently sampled, equally sized four-group training subsets, limited by the smallest group.
DFR uses LIBLINEAR, while logistic probes use LBFGS.
\paragraph{Training.}
We use five independent training runs.
Frozen ResNet-50 \citep{he_deep_2016} uses ImageNet-1K V2 weights, a shorter-side resize to 232 pixels and a $224\times224$ centre crop.
Training means and sample SDs standardise its 2,048 pooled features.
The adapter has one width-256 SiLU hidden layer, a linear 64-dimensional projection and non-affine LayerNorm, followed by a linear two-class head.

Dependence-penalised methods receive 600 task-only updates at $5\times10^{-4}$, then 800 updates at $10^{-4}$ with a fresh Adam optimiser.
Task batches contain 128 examples; Adam uses default moments.
ERM and GroupDRO share adapter initialisations and receive 1,400 updates, with their objectives and weight decay active throughout.
Applicable estimator networks have two width-128 SiLU hidden layers, learning rate $3\times10^{-4}$ and batch size 512, with 2,000 initial fitting updates and 20 auxiliary updates per representation update.

\paragraph{Evaluation.}
Worst-group accuracy is the minimum test accuracy over the four bird--background groups.
Separate probes fit training examples using within-class training standardisation.
Linear leakage uses one balanced $\ell_2$ logistic probe per bird class with $C=1$.
The MLP probe takes the representation and one-hot bird class, with one width-128 SiLU hidden layer.
Three probe fits use Adam at $10^{-3}$, no weight decay and four-group-balanced batches of 128, with checkpoints at 800 and 1,600 updates.

Leakage is the mean classwise departure of nuisance-prediction ROC AUC from chance:
\begin{equation}
\label{eq:nuisance-leakage}
D_{\rm probe}
=\frac{100}{2}\sum_{l\in\{0,1\}}
\left|\operatorname{AUC}(s(R),A\mid L=l)-\frac12\right|.
\end{equation}
This probe-discrimination metric ranges from 0 to 50 percentage points.
MLP leakage averages the three probe fits before summarising across five training runs.

\paragraph{Additional results.}
\Cref{tab:nuisance-wb-r50-64} adds overall accuracy to the main comparison.

\begin{table}[!t]
\centering\small
\setlength{\tabcolsep}{4pt}
\caption{\textbf{Waterbirds.} Test accuracy and background leakage; mean $\pm$ SD over five training runs.}
\label{tab:nuisance-wb-r50-64}
\begin{tabularx}{\textwidth}{l*{4}{>{\centering\arraybackslash}X}}
\toprule
 & \multicolumn{2}{c}{Accuracy (\%, $\uparrow$)} & \multicolumn{2}{c}{Leakage (points, $\downarrow$)} \\
\cmidrule(lr){2-3}\cmidrule(lr){4-5}
Method & Overall & Worst group & Linear & MLP \\
\midrule
\multicolumn{5}{l}{\textit{MI-based methods}} \\
\rowcolor{smirow}
\textbf{SMI (Ours)} & $89.96\pm0.32$ & $\mathbf{71.68\pm1.33}$ & $\mathbf{18.19\pm5.35}$ & $\mathbf{9.20\pm5.49}$ \\
MINE & $90.68\pm0.32$ & $70.87\pm1.16$ & $26.21\pm2.42$ & $22.23\pm4.00$ \\
SMILE & $90.62\pm0.36$ & $71.43\pm1.83$ & $25.34\pm3.15$ & $17.69\pm1.76$ \\
InfoNCE & $\underline{90.75\pm0.22}$ & $71.43\pm1.40$ & $25.59\pm3.14$ & $18.12\pm4.09$ \\
CLUB & $90.22\pm0.24$ & $70.62\pm1.32$ & $24.82\pm1.78$ & $20.95\pm2.62$ \\
NWJ & $90.46\pm0.23$ & $70.34\pm1.09$ & $27.18\pm4.59$ & $18.82\pm2.32$ \\
FLO & $\mathbf{90.79\pm0.15}$ & $\underline{71.59\pm0.95}$ & $\underline{20.94\pm3.08}$ & $\underline{14.39\pm2.36}$ \\
KNIFE & $87.76\pm1.32$ & $65.79\pm1.91$ & $25.70\pm5.60$ & $21.67\pm4.78$ \\
$f$-DIME & $89.84\pm0.42$ & $71.37\pm1.01$ & $33.69\pm4.24$ & $34.50\pm3.65$ \\
MIGE & $83.38\pm1.17$ & $67.16\pm1.50$ & $30.55\pm2.21$ & $30.34\pm2.55$ \\
\midrule
\multicolumn{5}{l}{\textit{Other methods}} \\
Adversarial & $\mathbf{90.56\pm0.36}$ & $\underline{72.31\pm0.80}$ & $\underline{20.50\pm6.49}$ & $\mathbf{12.22\pm3.65}$ \\
HSIC & $\underline{89.67\pm0.27}$ & $\mathbf{73.05\pm0.87}$ & $\mathbf{16.12\pm1.77}$ & $15.41\pm2.15$ \\
ERM & $85.57\pm0.92$ & $68.57\pm1.41$ & $36.40\pm8.88$ & $37.40\pm7.63$ \\
GroupDRO & $84.72\pm0.51$ & $65.95\pm1.55$ & $21.15\pm6.29$ & $\underline{13.12\pm8.63}$ \\
DFR & $83.07\pm1.59$ & $68.27\pm4.21$ & $32.91\pm6.04$ & $30.23\pm3.37$ \\
\bottomrule
\end{tabularx}
\end{table}

\subsection{Edges-to-Photos Translation}
\label{app:edges-additional}

\paragraph{Data and objective.}
Edges2Shoes is distributed through \url{https://github.com/junyanz/pytorch-CycleGAN-and-pix2pix/blob/master/docs/datasets.md}.
It derives from UT Zappos50K (\url{https://vision.cs.utexas.edu/projects/finegrained/utzap50k/}), which restricts the dataset to academic, non-commercial use.
On Edges2Shoes \citep{yu_finegrained_2014,isola_imagetoimage_2017}, we adapt a common pretrained BicycleGAN to retain information about an independent latent code:
\begin{equation}
\min_\theta\ \mathcal L_{\mathrm{task}}(\theta)-\lambda I(H_\theta;Z),
\qquad H_\theta=\Phi(G_\theta(C,Z)),\qquad Z\sim\mathcal N(0,I_8),
\end{equation}
where $C$ is an edge map and $\Phi$ is a frozen image-feature map.
The MI objective is marginal over $C$; when finite,
$I(Z;H_\theta\mid C)=I(Z;H_\theta)+I(Z;C\mid H_\theta)$.
With deterministic generation and continuous codes, unperturbed MI need not be finite.
We therefore evaluate latent predictability and image-distribution matching separately.

Of 49,825 official training pairs, a fixed random permutation reserves 512 for development and 49,313 for adaptation.
Disjoint adaptation subsets contain 4,096 pairs for regressor fitting, 1,024 for regressor selection, 43,681 for estimator fitting and 512 reserved for diagnostics.
All adaptation pairs remain eligible for generator updates.
The regressor-fitting photos also fit the training feature projection; their first 512 calibrate the edge detector.
Development pairs are excluded from adaptation but may have appeared in public pretraining.
The 200 official validation pairs informed earlier exploration and are treated as a reporting set.
Three training runs use the same data splits and evaluation latent codes.

Each pair is split into edge/photo halves, resized bicubically to $256\times256$, converted to a grayscale edge map and RGB photo, and scaled to $[-1,1]$.
No random flips or other training augmentation are used.

\paragraph{MI-based methods.}
SMI spreads only the generated feature vector:
\begin{equation}
\label{eq:edges-smi-full}
\mathcal I_{\mathrm{SMI}}(\theta)
=\int_{0.01}^{1}I(H_\theta+\sigma\epsilon;Z)\,
\frac{\mathrm d\sigma}{\sigma\log100},
\qquad \epsilon\sim\mathcal N(0,I_{256}).
\end{equation}
Each example independently samples a scale at each update.
Joint pairs are $(H_\theta(C,Z)+\sigma\epsilon,Z)$; product pairs are $(H_\theta(C',Z')+\sigma\epsilon',Z)$ with independent $(C',Z')$ and $\epsilon'$.
Paired joint/product examples share the conditioning scale and use independent Gaussian noise.
Balanced binary cross-entropy fits the scale-conditioned logit $T_\eta$:
\begin{equation}
\mathcal L_{\mathrm{aux}}
=\tfrac12\mathbb E_{J,\sigma}\operatorname{softplus}(-T_\eta)
+\tfrac12\mathbb E_{P,\sigma}\operatorname{softplus}(T_\eta).
\end{equation}
Writing $\widetilde H=H_\theta+\sigma\epsilon$, generator minimisation uses
\begin{equation}
\mathcal L_{\mathrm{SMI},G}
=-\mathbb E\!\left[
\operatorname{stopgrad}\!\left(\nabla_{\widetilde H}
T_\eta(\widetilde H,Z,\log\sigma)\right)^{\mathsf T}H_\theta
\right].
\end{equation}
The minus sign gives MI maximisation; the surrogate value is not an MI estimate.

Other MI methods use unperturbed features and codes.
InfoNCE \citep{oord_representation_2018} and NWJ \citep{nguyen_estimating_2010} use concatenated-input MLP critics and the released all-pairs objectives.
InfoNCE uses 256 in-batch candidates, including the positive, without feature normalisation.
NWJ uses all off-diagonal product pairs and $\exp(T-1)$ without clipping; negative pairs retain generator gradients.
FLO \citep{guo_tight_2022} uses normalised width-256 bilinear embeddings, an auxiliary scalar head, 255 negatives and a trainable temperature initialised at one.
KNIFE \citep{pichler_differential_2022} fits 128-component full-covariance Gaussian mixtures for $p(Z)$ and $p(Z\mid H)$, with Tanh log-variance heads.
It differentiates the fitted marginal-minus-conditional entropy through $H$.
MIGE \citep{wen_mutual_2020} uses joint-minus-marginal spectral scores, support size 256, Gaussian median-distance bandwidth and eigen-energy threshold 0.99.
Only generated-feature coordinates enter its gradient.

\paragraph{Other methods.}
Task only minimises $\mathcal L_{\mathrm{task}}$.
Latent $\ell_1$ \citep{zhu_multimodal_2017} uses BicycleGAN's latent reconstruction with weight 0.5 on four prior-branch outputs; the encoder posterior mean predicts the code.
Encoder parameters are frozen for this loss while retaining input gradients and encoder gradients from the common task.
DSGAN \citep{yang_diversitysensitive_2019} uses eight independently sampled edge maps and two independent codes per map:
\begin{equation}
\mathcal L_{\mathrm{DSGAN}}
=-8\,\mathbb E\!\left[
\frac{d_1(G_\theta(C,Z),\operatorname{stopgrad}(G_\theta(C,Z')))}
{d_1(Z,Z')+10^{-5}}
\right],
\end{equation}
where $d_1$ is mean absolute coordinate difference.
The second image is detached, as in the release; no margin is enabled.
These regularisers act on unperturbed RGB outputs.
Pretrained is the single common generator before adaptation and has no training-run SD.

\paragraph{Training.}
Generator and encoder weights come from the public BicycleGAN Shoes checkpoint.
The generator is a width-64 U-Net with instance normalisation, ReLU and eight latent coordinates injected at every encoder level.
Dropout is disabled during adaptation.
The encoder is the released width-64 residual architecture for $256\times256$ images.
Each run freshly initialises two two-scale RGB PatchGAN discriminators, paired across methods.

A task minibatch contains eight real pairs; its first four edge maps are used for both encoded-posterior and independent-prior branches.
The common generator and encoder loss is
\begin{equation}
\mathcal L_{\mathrm{task}}
=\mathcal L_{\mathrm{LSGAN}}^{\mathrm{encoded}}
+\mathcal L_{\mathrm{LSGAN}}^{\mathrm{prior}}
+10\mathcal L_{\mathrm{image}\text{-}\ell_1}
+0.01\mathcal L_{\mathrm{KL}}.
\end{equation}
Image reconstruction uses the encoded branch; KL sums latent coordinates and averages examples; adversarial losses sum discriminator scales.
Public pretraining includes latent reconstruction, but adaptation adds it only for the latent $\ell_1$ comparator.

The frozen training representation uses ImageNet AlexNet \citep{krizhevsky_imagenet_2012}, $224\times224$ antialiased bilinear resizing and ImageNet channel normalisation.
Adaptive $2\times2$ pooling yields 1,024 features; PCA retains and whitens 256 components using adaptation photos.
The feature extractor and PCA projection remain fixed.

Discriminators receive 100 initial updates with generator/encoder fixed, using Adam at $2\times10^{-4}$ and moments $(0.5,0.999)$.
Fresh optimisers then take 20,000 generator updates, with Adam moments $(0.5,0.999)$ and no weight decay.
The learning rate is $10^{-5}$ for 4,000 updates, then cosine-decayed to $10^{-6}$.
No gradient clipping is used; every run completes the fixed budget.

Neural estimators have two width-256 SiLU hidden layers and method-specific output layers.
Adam uses learning rate $3\times10^{-4}$, batch size 256 and no weight decay, with moments $(0.9,0.999)$ except KNIFE's $(0.5,0.999)$.
Each estimator receives 10,000 initial fitting updates on 32,768 stored feature--code pairs from the common starting generator, with equal joint and independent-feature halves sampled with replacement.
Thereafter, eight estimator updates use fresh generator samples per generator update; MIGE has no fitting stage.
MI generator batches contain 256 examples, implemented by deterministic forward caching and backward replay in microbatches of eight, retaining full-batch negative-sample gradients.
Task-loss and estimator minibatches are sampled independently.
These shared architectures and schedules are benchmark adaptations.

After initial estimator fitting, $\lambda$ is fixed per method and run to match 30\% of the task-gradient norm, using median norms over 16 calibration batches.
Reporting-set metrics select neither $\lambda$ nor generator checkpoints.

\paragraph{Evaluation.}
We report results only for the final generators after 20,000 updates, fitting a separate MLP regressor to predict latent codes for each method and run.
The primary representation concatenates the 1,024-dimensional CLS and mean patch tokens of DINOv2 \citep{oquab_dinov2_2024,darcet_vision_2024}, using the ViT-L/14 model with four register tokens.
ImageNet ResNet-50 \citep{he_deep_2016} supplies a second 2,048-dimensional representation to assess sensitivity to the evaluation features.
Both use $224\times224$ antialiased bilinear resizing and ImageNet normalisation, with no evaluation PCA.
Feature means and standard deviations are fitted only on the regressor-training set.

Each edge map uses eight independent Gaussian codes; the codes for each data split and regressor initialisations are shared across methods.
The MLP has two width-256 SiLU hidden layers and eight outputs, using AdamW with default moments, batch size 512 and at most 200 epochs.
Selection-set MSE, evaluated every five epochs, chooses the regressor checkpoint and one of four learning-rate/weight-decay pairs in $\{3\times10^{-4},10^{-3}\}\times\{10^{-4},10^{-3}\}$.
Development and reporting pairs enter neither regressor fitting nor selection.
For $n$ reporting edge maps and $K=8$ codes per map, $N=nK$ generated images give
\begin{equation}
\label{eq:edges-r2}
R^2=\frac18\sum_{j=1}^{8}\left(
1-\frac{\sum_{a=1}^{N}(z_{aj}-\widehat z_{aj})^2}
{\sum_{a=1}^{N}(z_{aj}-\bar z_j)^2}
\right),
\end{equation}
where $\bar z_j$ is the reporting-set coordinate mean.
The reporting set has 200 edge maps and 1,600 outputs.
This coefficient of determination measures latent predictability.

The unperturbed condition applies no additional image processing.
JPEG (75) encodes and decodes 8-bit RGB at quality 75.
Resize downsamples to $128\times128$ and upsamples to $256\times256$ using bilinear interpolation, with antialiasing on downsampling.
Both use the fixed DINOv2 regressor and standardiser fitted on unperturbed images.

Distribution matching uses KID \citep{binkowski_demystifying_2018} on unperturbed outputs, with 2,048-dimensional FID-compatible Inception features, bilinear resizing to $299\times299$ and kernel $k(u,v)=(u^{\mathsf T}v/2048+1)^3$.
For paired real features $r_i$ and generated features $f_{ik}$, we use
\begin{align}
A_{ij}
&=k(r_i,r_j)+\frac1K\sum_{k=1}^{K}
\big[k(f_{ik},f_{jk})-k(r_i,f_{jk})-k(f_{ik},r_j)\big],\\
\widehat{\mathrm{KID}}
&=\frac{1}{n(n-1)}\sum_{i\ne j}A_{ij}.
\end{align}
This U-statistic averages over code indices, excluding pairs sharing an edge map from every term.
Values are reported in raw units.

Edge F1 measures agreement between HED \citep{xie_holisticallynested_2015} detections on unperturbed outputs and the input edge map.
A threshold calibrated on 512 real adaptation pairs is fixed at 0.6.
Matching uses a two-pixel circular tolerance, excludes a two-pixel border and permits multiple nearby matches.
F1 is computed per image and then averaged.
Tables report mean and sample SD across three training runs; the pretrained reference is evaluated once.

\paragraph{Additional results.}
\Cref{tab:edges-additional} reports KID, edge F1 and latent-code prediction using ResNet-50 features.

\begin{table}[!t]
\centering\small
\setlength{\tabcolsep}{4pt}
\caption{\textbf{Edges2Shoes additional results.} Reporting-set metrics at 20,000 generator updates; mean $\pm$ SD over three training runs.}
\label{tab:edges-additional}
\begin{tabularx}{\textwidth}{lc*{2}{>{\centering\arraybackslash}X}}
\toprule
Method & KID ($\downarrow$) & Edge F1 ($\uparrow$) & ResNet-50 MLP $R^2$ ($\uparrow$) \\
\midrule
\multicolumn{4}{l}{\textit{MI-based methods}} \\
\rowcolor{smirow}
\textbf{SMI (Ours)} & $(1.587\pm0.145)\times10^{-2}$ & $0.679\pm0.004$ & $\mathbf{0.584\pm0.008}$ \\
InfoNCE & $\underline{(1.293\pm0.084)\times10^{-2}}$ & $0.685\pm0.003$ & $0.523\pm0.007$ \\
NWJ & $(1.481\pm0.064)\times10^{-2}$ & $0.679\pm0.002$ & $\underline{0.573\pm0.016}$ \\
FLO & $\mathbf{(1.177\pm0.014)\times10^{-2}}$ & $\mathbf{0.690\pm0.002}$ & $0.470\pm0.005$ \\
KNIFE & $(1.451\pm0.154)\times10^{-2}$ & $0.684\pm0.002$ & $0.533\pm0.011$ \\
MIGE & $(1.420\pm0.148)\times10^{-2}$ & $\underline{0.689\pm0.004}$ & $0.520\pm0.009$ \\
\midrule
\multicolumn{4}{l}{\textit{Other methods and reference}} \\
Task only & $\underline{(1.187\pm0.045)\times10^{-2}}$ & $\mathbf{0.690\pm0.002}$ & $0.470\pm0.004$ \\
Latent $\ell_1$ & $(1.499\pm0.054)\times10^{-2}$ & $\underline{0.675\pm0.002}$ & $\underline{0.570\pm0.005}$ \\
DSGAN & $(7.666\pm3.324)\times10^{-2}$ & $0.539\pm0.040$ & $\mathbf{0.626\pm0.018}$ \\
Pretrained & $\mathbf{1.019\times10^{-2}}$ & $0.672$ & $0.463$ \\
\bottomrule
\end{tabularx}
\end{table}

\subsection{Computational Costs}
\label{app:synthetic-cost}

\paragraph{Measurement protocol.}
We time Gaussian MI minimisation with one shared correlation parameter on an isolated RTX 3090.
An iteration includes fresh sampling, $N_{\mathrm{aux}}=5$ estimator updates and one generator update; MIGE uses $N_{\mathrm{aux}}=0$.
Each trial uses one pinned CPU thread, $N_{\mathrm{warm}}=100$ warm-up iterations and $N_{\mathrm{iter}}=100$ timed iterations, with CUDA synchronisation around measurements.
Initialisation and input-scale calibration are excluded.
We report the median and range of mean iteration times over $N_{\mathrm{rep}}=3$ repetitions.
Memory is peak allocated GPU tensor memory during timing, including model and optimiser state.

The sweeps vary dimension $d$, minibatch size $B$ and number of correlated coordinates $r$, fixing $(r,B)=(2,256)$, $(d,r)=(100,2)$ and $(d,B)=(100,256)$, respectively.
A dagger marks any setting where a repetition's mean over the final 20 iterations differs by more than 10\% from its first 20; no measurements are excluded.

\paragraph{Results.}
\Cref{fig:synthetic-runtime} reports iteration time and peak memory across the three sweeps.
These iteration costs are hardware- and implementation-specific.

\begin{figure*}[!t]
\centering
\includegraphics[width=\textwidth]{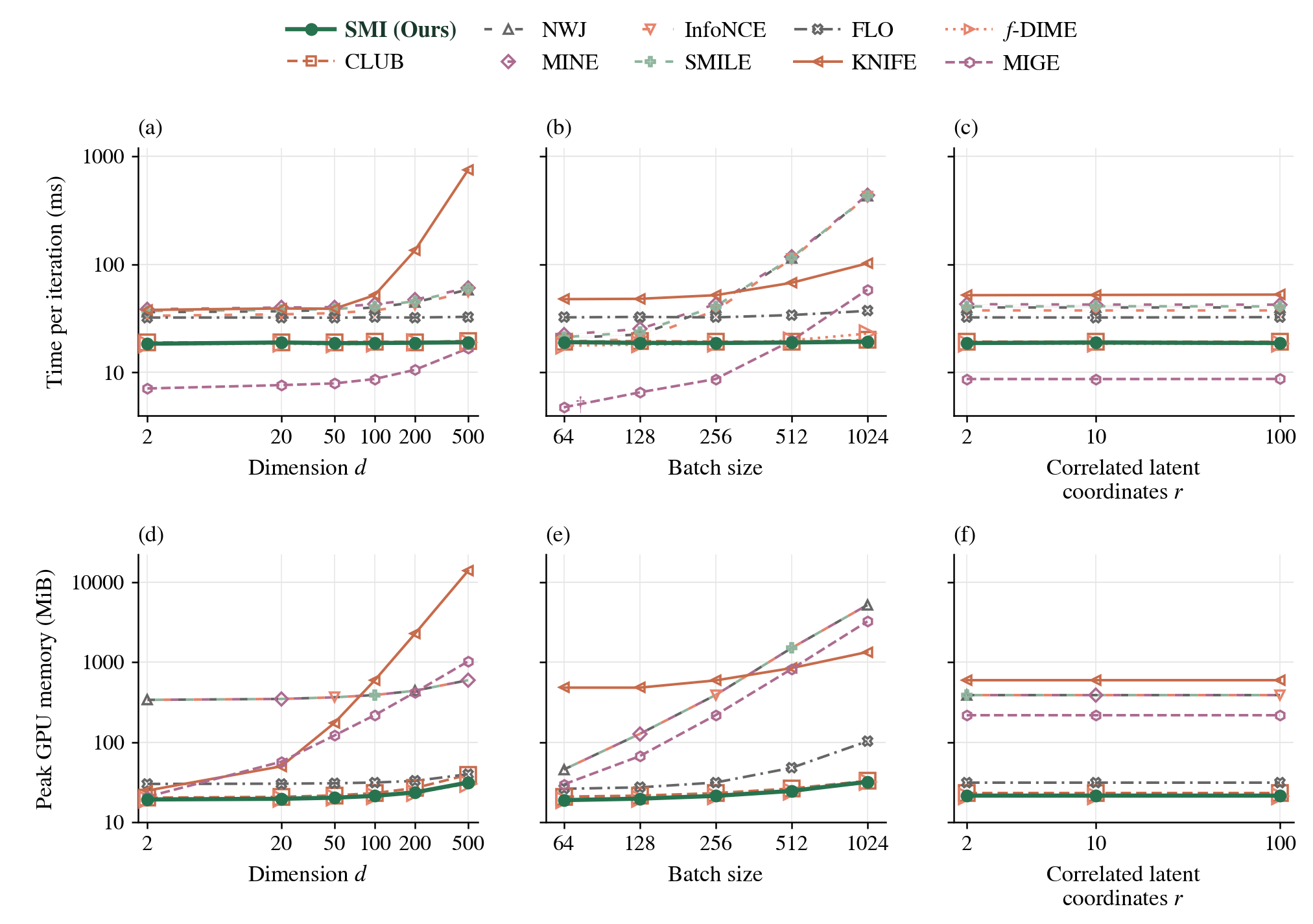}
\caption{\textbf{Computational costs.} Iteration time and peak allocated GPU memory across dimension, minibatch size and correlated coordinates; median and range over three repetitions.}
\label{fig:synthetic-runtime}
\end{figure*}

\subsubsection{Analytical Costs}
\label{app:mi-complexity}

\paragraph{Scope and notation.}
The bounds in \cref{tab:mi-computational-costs} concern the evaluated unconditional synthetic implementations, with $X,Y\in\mathbb R^d$.
As in \cref{alg:smi}, $B$ denotes minibatch size.
Let $N_{\mathrm{aux}}$ be the number of estimator updates per generator update; for SMI, $N_{\mathrm{aux}}=N_{\mathrm{clf}}$.
Write $N_\ell\ge2$ for the number of hidden layers, $N_{\mathrm{mix}}$ for mixture components and $h$ for hidden width.
Timing uses $h=256$, $N_\ell=2$, $N_{\mathrm{aux}}=5$ and $N_{\mathrm{mix}}=8$, except that MIGE has no estimator optimiser.
A dense MLP with $O(d)$ inputs and scalar or $O(d+h)$ outputs has $O(dh+N_\ell h^2)$ parameters and work per example, retaining $O(d+N_\ell h)$ activations.

We count dense arithmetic and reverse-mode differentiation without checkpointing.
Peak memory includes parameters, gradients, optimiser state, activations and temporary tensors in scalar words.
Sequential estimator updates increase runtime while reusing memory.
Analytical bounds exclude generators, sampling and dataset storage; measured iterations include sampling and generator backpropagation.

\begin{table*}[!t]
\centering
\caption{\textbf{MI control methods and computational costs.} Estimator costs per generator update for the evaluated synthetic implementations.}
\label{tab:mi-computational-costs}
\small
\setlength{\tabcolsep}{2pt}
\renewcommand{\arraystretch}{1.13}
\begin{tabularx}{\textwidth}{lYcc}
\toprule
Method & Objective or gradient & Time $O(\cdot)$ & Peak memory $O(\cdot)$ \\
\midrule
\rowcolor{smirow}
\textbf{SMI (Ours)} & Spread log-ratio gradient & $(N_{\mathrm{aux}}+1)B(dh+N_\ell h^2)$ & $dh+N_\ell h^2+B(d+N_\ell h)$ \\
CLUB & Sampled Gaussian contrast & $(N_{\mathrm{aux}}+1)B(dh+N_\ell h^2)$ & $dh+N_\ell h^2+B(d+N_\ell h)$ \\
NWJ & Variational KL bound & $(N_{\mathrm{aux}}+1)B^2(dh+N_\ell h^2)$ & $dh+N_\ell h^2+B^2(d+N_\ell h)$ \\
MINE & DV; moving normaliser & $(N_{\mathrm{aux}}+1)B^2(dh+N_\ell h^2)$ & $dh+N_\ell h^2+B^2(d+N_\ell h)$ \\
InfoNCE & Contrastive bound & $(N_{\mathrm{aux}}+1)B^2(dh+N_\ell h^2)$ & $dh+N_\ell h^2+B^2(d+N_\ell h)$ \\
SMILE & JS gradient; DV readout & $(N_{\mathrm{aux}}+1)B^2(dh+N_\ell h^2)$ & $dh+N_\ell h^2+B^2(d+N_\ell h)$ \\
FLO & Bilinear Fenchel--Legendre bound & $\begin{aligned}(N_{\mathrm{aux}}+1)\bigl[&B(dh+N_\ell h^2)\\&+B^2h\bigr]\end{aligned}$ & $\begin{aligned}&dh+N_\ell h^2\\&\quad+B(d+N_\ell h)+B^2\end{aligned}$ \\
KNIFE & Full-covariance mixture entropies & $\begin{aligned}&(N_{\mathrm{aux}}+1)B\\&\quad\times(dh+N_\ell h^2+N_{\mathrm{mix}}hd^2)\end{aligned}$ & $\begin{aligned}&dh+N_\ell h^2+N_{\mathrm{mix}}hd^2\\&\quad+B(d+N_\ell h+N_{\mathrm{mix}}d^2)\end{aligned}$ \\
$f$-DIME & KL density ratio; deranged pairs & $(N_{\mathrm{aux}}+1)B(dh+N_\ell h^2)$\rlap{$^{\dagger}$} & $dh+N_\ell h^2+B(d+N_\ell h)$ \\
MIGE & Spectral score difference & $B^3+B^2d$ & $B^2d$ \\
\bottomrule
\end{tabularx}
\vspace{3pt}
\begin{minipage}{\textwidth}
\footnotesize
$^{\dagger}$Expected time for rejection-sampled derangements. Bounds refer to the evaluated implementations and exclude generator and data-sampling costs.
DV: Donsker--Varadhan; JS: Jensen--Shannon.
\end{minipage}
\end{table*}

\paragraph{SMI.}
One ratio network processes $O(B)$ spread joint/product pairs.
Noise generation costs $O(Bd)$; each fitting update costs $O(B(dh+N_\ell h^2))$.
The classifier logit is $T_\phi(x_t,y,t)=\log[c_\phi(x_t,y,t)/(1-c_\phi(x_t,y,t))]$, giving the estimated score difference $\widehat v_t=\nabla_{x_t}T_\phi$.
For the evaluated $\alpha_t=1$ kernel, the surrogate is
\begin{equation}
\widehat{\mathcal L}_{\mathrm{sur}}
=\frac1B\sum_{i=1}^{B} w_{\mathrm{SMI}}(t_i)\,
\operatorname{stopgrad}\!\left[\widehat v_{t_i}(x_{t_i,i},y_i)\right]^{\top}x_{t_i,i}.
\end{equation}
One reverse pass obtains all input gradients at the same order of work.
Stopping them avoids higher-order differentiation; the generator uses a vector--Jacobian product.
Sampling one scale per example introduces no factor for a discretised noise schedule.
Thus $N_{\mathrm{aux}}$ fitting updates and one generator update require $O((N_{\mathrm{aux}}+1)B(dh+N_\ell h^2))$ estimator work.
Network and optimiser state occupy $O(dh+N_\ell h^2)$ memory, and a batch retains $O(B(d+N_\ell h))$ activations and score vectors.
Their sum gives SMI's peak-memory bound in \cref{tab:mi-computational-costs}.

\end{document}